\documentclass{article}

   \usepackage[final,nonatbib]{neurips_2022}

\usepackage[utf8]{inputenc} 
\usepackage[T1]{fontenc}    
\usepackage[colorlinks=true, linkcolor=blue, citecolor=blue,urlcolor=black]{hyperref}       
\usepackage{url}            
\usepackage{booktabs}       
\usepackage{amsfonts}       
\usepackage{nicefrac}       
\usepackage{microtype}      
\usepackage{xcolor}         

\usepackage{amsmath}
\usepackage{amsthm}
\usepackage{amssymb}
\usepackage[english]{babel}
\usepackage{algorithm}
\usepackage{algorithmic}
\usepackage[capitalize,noabbrev]{cleveref}
\usepackage{booktabs}  
\usepackage{makecell}
\usepackage{pifont}
\usepackage{colortbl}
\usepackage[toc,page,header]{appendix}
\usepackage[]{minitoc}
\renewcommand{\partname}{}
\renewcommand{\thepart}{}

\newtheorem{theorem}{Theorem}[section]

\newtheorem{lemma}[theorem]{Lemma}
\newtheorem{corollary}[theorem]{Corollary}
\newtheorem{remark}[theorem]{Remark}

\crefname{equation}{Eq.}{Equations}

\usepackage{nicefrac}

\DeclareMathOperator*{\argmin}{arg\,min}
\DeclareMathOperator*{\argmax}{arg\,max}

\newcommand{\bbE}{\mathbb{E}}
\newcommand{\bbI}{\mathbb{I}}
\newcommand{\bbR}{\mathbb{R}}
\newcommand{\bbG}{\mathbb{G}}

\newcommand{\calF}{\mathcal{F}}
\newcommand{\calH}{\mathcal{H}}
\newcommand{\calM}{\mathcal{M}}

\newcommand{\calS}{\mathcal{S}}
\newcommand{\calA}{\mathcal{A}}
\newcommand{\calP}{\mathcal{P}}
\newcommand{\wt}{\widetilde}
\newcommand{\wh}{\widehat}
\newcommand{\ind}{\mathbb{I}}
\newcommand{\indevent}[1]{\ind \{ #1 \}}
\newcommand{\regret}{R_K}
\newcommand{\ocset}{\Delta(\calM)}
\newcommand{\ocsetk}[1]{\Delta(\calM,#1)}
\newcommand{\KL}[2]{\text{KL}(#1 \;\|\; #2)}
\newcommand{\sinit}{s_{\text{init}}}
\newcommand{\filt}[1]{\wt \calH^{ #1 }}
\newcommand{\logterm}{\iota}

\newcommand{\cmark}{\ding{51}}%
\newcommand{\xmark}{\ding{55}}%

\newcommand{\norm}[1]{\left\|{#1}\right\|}
\newcommand{\rbr}[1]{\left(#1\right)}
\newcommand{\sbr}[1]{\left[#1\right]}
\newcommand{\cbr}[1]{\left\{#1\right\}}

\newcommand{\abr}[1]{\left|#1\right|}

\newcommand{\tcjconfpara}{\iota}

\newcommand{\hatl}{\widehat{c}}
\newcommand{\hatell}{\widehat{\ell}}
\newcommand{\hatd}{\widehat{\Delta}}
\newcommand{\hatL}{\widehat{L}}

\newcommand{\whatq}{\widehat{q}}
\newcommand{\estq}{q}
\newcommand{\westq}{\widehat{q}}
\newcommand{\wtilq}{\widetilde{q}}

\newcommand{\regtwo}{\textsc{Bias}_1}
\newcommand{\regthree}{\textsc{Reg}}
\newcommand{\regfour}{\textsc{Bias}_2}

\newcommand{\inner}[1]{ \left\langle {#1} \right\rangle }

\newcommand{\probdist}{\omega}

\title{Near-Optimal Regret for Adversarial MDP with Delayed Bandit Feedback}

\author{%
    Tiancheng Jin \\
    University of Southern California \\
    tiancheng.jin@usc.edu \\
    \And
    Tal Lancewicki \\
    Tel Aviv University \\
    lancewicki@mail.tau.ac.il \\
    \And
    Haipeng Luo \\
    University of Southern California \\
    haipengl@usc.edu \\
    \And
    Yishay Mansour \\ 
    Tel Aviv University and Google Research \\
    mansour.yishay@gmail.com \\
    \And
    Aviv Rosenberg\thanks{Research conducted while the author was a student at Tel Aviv University.} \\ 
    Amazon \\
    avivros@amazon.com \\
}

\begin{document}
\doparttoc[n]
\faketableofcontents 
\maketitle

\begin{abstract}
    The standard assumption in reinforcement learning (RL) is that agents observe feedback for their actions immediately.
    However, in practice feedback is often observed in delay.
    This paper studies online learning in episodic Markov decision process (MDP) with unknown transitions, adversarially changing costs, and unrestricted delayed bandit feedback.
    More precisely, the feedback for the agent in episode $k$ is revealed only in the end of episode $k + d^k$, where the delay $d^k$ can be changing over episodes and chosen by an oblivious adversary.
    We present the first algorithms that achieve near-optimal $\sqrt{K + D}$ regret, where $K$ is the number of episodes and $D = \sum_{k=1}^K d^k$ is the total delay, significantly improving upon the best known regret bound of $(K + D)^{2/3}$.
\end{abstract}


\section{Introduction}

Delayed feedback has become a fundamental challenge that sequential decision making algorithms must face in almost every real-world application.
Notable examples include communication between agents \cite{chen2020delay}, video streaming \cite{changuel2012online} and robotics \cite{mahmood2018setting}.
Broadly, delays occur either for computational reasons, e.g., in autonomous vehicles and wearable technology, or when they are an inherent part of the environment like healthcare, finance and recommendation systems.

Although a prominent challenge in practice, there is only limited theoretical literature on delays in reinforcement learning (RL).
Recently, \cite{howson2021delayed} studied regret minimization in episodic Markov decision processes (MDPs) but assume that the delays (and costs) are \emph{stochastic}, i.e., sampled i.i.d from a fixed (unknown) distribution, which is a limiting assumption since it does not allow dependencies between costs and delays that are very common in practice.
The case of \emph{adversarial} delays and costs was also studied recently \cite{lancewicki2020learning}.
However, they focus on \emph{full-information} feedback where the learner observes the entire cost function, which is not realistic in many applications, and obtain only sub-optimal regret bounds for \emph{bandit} feedback (where the learner observes only the costs on the traversed trajectory).

In this paper we significantly advance our understanding of delayed feedback in adversarial MDPs with bandit feedback.
More precisely, we consider episodic MDPs with unknown transition function, adversarially changing costs (bounded in $[0,1]$) and unrestricted delayed bandit feedback, i.e., the learner observes the costs suffered in episode $k$ only in the end of episode $k + d^k$ where the sequence of delays $\{ d^k \}_{k=1}^K$ are chosen by an oblivious adversary.
We develop the first algorithms for this setting that achieve near-optimal regret and provide a major improvement over the currently best known regret bound \cite{lancewicki2020learning} - see \cref{table:regret bounds} for more details.

\begin{table}
    \caption{Regret bounds for Adversarial MDPs with unknown transition and unrestricted delayed bandit feedback. $K$ is the number of episodes, $D$ is the total delay, $H$ is the horizon,  $S$ is the number of states and  $A$ is the number of actions. Algorithms presented in this paper appear in grey.}
    \begin{center}
        \begin{tabular}[b]{|c|c|c|c|}
            \hline
            Algorithm & Regret & Efficient & Regret w.h.p
            \\
            \hline \hline
            D-OPPO \cite{lancewicki2020learning} & $\wt O (HS \sqrt{A} K^{2/3} + H^2 D^{2/3})$ & \cmark & \cmark
            \\
            \hline
            \rowcolor{lightgray}
            Delayed Hedge & $\wt O (H^2 S \sqrt{  A K  }+  H^{3/2}\sqrt{SD})$ & \xmark & \cmark
            \\
            \hline
            \rowcolor{lightgray}
            Delayed UOB-FTRL & $\wt O (H^2 S\sqrt{AK} + H^{3/2}SA \sqrt{D})$ & \cmark & \xmark
            \\
            \hline
            \rowcolor{lightgray}
            Delayed UOB-REPS & $\wt O (H^2 S\sqrt{AK } + (H S A)^{1/4} \cdot H \sqrt{D})^*$
            & \cmark & \cmark
            \\
            \hline
            Lower bound \cite{lancewicki2020learning} & $\Omega(H^{3/2}\sqrt{SAK} + H\sqrt{D})$ &  & 
            \\
            \hline
        \end{tabular}
    \end{center}
    $^*$Under unknown dynamics Delayed UOB-REPS has an additional additive term in the regret that scales linearly with $d_{max}$. 
    One can avoid the dependency in $d_{max}$ but with a slightly weaker bound than the one that appears in this table - for more details see \cref{remark: delayed trajectory} in the supplementary material.

    \label{table:regret bounds}
\end{table}

In the following paragraph we provide an overview of our contributions and the structure of the paper.
In \cref{sec-paper:hedge} we devise an inefficient Hedge~\cite{freund1997decision} based algorithm that treats every deterministic policy as an arm. This can be seen as a warm-up -- a relatively simple and elegant solution that shows that order $\sqrt{K+D}$ regret is attainable with delayed bandit feedback. 
Moreover, our adaptation of Hedge to the setting of adversarial MDP with unknown transition and bandit feedback presents highly non-trivial algorithmic and technical features that may be of independent interest.
Then, we focus on the pressing question: \emph{Can delayed bandit feedback be handled both optimally
and efficiently?} We answer this affirmatively by presenting two efficient algorithms with near-optimal regret. Through our unique analysis and algorithmic design, we shed light on the great challenges of handling efficiently delayed bandit feedback.
In \cref{sec-paper:ftrl} we consider a relatively standard  algorithm we call Delayed UOB-FTRL, based on the Follow the Regularized Leader (FTRL) framework, and focus on a unique novel analysis that may be of independent interest. 
As seen in \cref{table:regret bounds}, our analysis of Delayed UOB-FTRL shows regret similar to the inefficient Delayed Hedge. However, it has worse dependence on $S$ and $A$, and has regret guarantee on expectation rather than with high probability (w.h.p). 
In \cref{sec-paper:UOB-REPS} we propose our final solution which is mainly algorithmic: we introduce the algorithm Delayed UOB-REPS that has a novel importance-sampling estimator which generalizes the standard estimator and accommodates it to the delays. This approach allows us to follow the path of more standard analysis, but most importantly, ensures w.h.p the best regret so far (see \cref{table:regret bounds}).
The first term of the regret bound matches the best known regret for adversarial MDP with non-delayed bandit feedback \cite{jin2019learning}, while the second term matches the lower bound of \cite{lancewicki2020learning} up to a factor of $(HSA)^{1/4}$.

\subsection{Additional Related Work}

\textbf{Delays in RL.}
While delays are popular in the practical RL literature \cite{schuitema2010control,liu2014impact,changuel2012online,mahmood2018setting,derman2021acting}, there is limited theoretical literature on the subject.
Most previous work \cite{katsikopoulos2003markov,walsh2009learning} considered constant delays in observing the current state.
However, the challenges in that setting are different than the ones considered in this paper (see \cite{lancewicki2020learning} for more details).
As discussed in the introduction, most related to this paper are the recent works of \cite{lancewicki2020learning} and \cite{howson2021delayed}.

\textbf{Delays in multi-arm bandit (MAB).}
Delays were extensively studied in MAB and optimization both in the stochastic setting \cite{agarwal2012distributed,vernade2017stochastic,vernade2020linear,pike2018bandits,cesa2018nonstochastic,zhou2019learning,manegueu2020stochastic,lancewicki2021stochastic,cohen2021asynchronous}, and the adversarial setting \cite{quanrud2015online,cesa2016delay,thune2019nonstochastic,bistritz2019online,zimmert2020optimal,ito2020delay,gyorgy2020adapting,van2021nonstochastic}.
However, as discussed in \cite{lancewicki2020learning}, delays introduce new challenges in MDPs that do not appear in MAB.

\textbf{Regret minimization in RL.}
There is a rich literature on regret minimization in both stochastic \cite{jaksch2010near,azar2017minimax,jin2018q,jin2020provably,yang2019sample,zanette2020frequentist,zanette2020learning} and adversarial \cite{zimin2013online,rosenberg2019online,rosenberg2019onlineb,rosenberg2021stochastic,jin2019learning,jin2020simultaneously,cai2020provably,shani2020optimistic,luo2021policy,jin2021best,pmlr-v151-he22a} MDPs.
Note that regret minimization in standard episodic MDPs is a special case of the model considered in this paper where $d^k = 0$ for every episode $k$.

\section{Preliminaries}

We consider the problem of learning adversarial  MDPs under delayed feedback. 
A finite-horizon episodic MDP is defined by a tuple $\calM = (\calS , \calA , H , p, \{ c^{k} \}_{k=1}^K)$, where
$\calS$ and $\calA$ are finite state and action spaces of sizes $|\calS| = S$ and $|\calA| = A$, respectively, $H$ is the horizon (i.e., episode length) and $K$ is the number of episodes. 
$p: \calS \times \calA \times [H] \to \Delta_{\calS}$ is the \textit{transition function} which defines the transition probabilities.
That is, $p_h(s' | s,a)$ is the probability to move to state $s'$ when taking action $a$ in state $s$ at time $h$. 
$\{ c^{k} : \calS \times \calA \times [H] \to [0,1] \}_{k=1}^K$ are  \textit{cost functions} which are chosen by an \textit{oblivious adversary}, such that $c_h^k(s,a)$ is the cost of taking action $a$ in state $s$ at time $h$ of episode $k$.

A \textit{policy} $\pi: \calS \times [H] \to \Delta_{\calA}$ is a function such that $\pi_h(a | s)$ is the probability to take action $a$ when visiting state $s$ at time $h$.
The value $V^{\pi,p'}_h(s ; c)$ is the expected cost of $\pi$ with respect to cost function $c$ and transition function $p'$ starting from state $s$ in time $h$, i.e.,
$V_{h}^{\pi,p'}(s ; c) = \bbE^{\pi,p'} \Bigl[ \sum_{h'=h}^{H} c_{h'}(s_{h'},a_{h'}) \mid s_{h}=s \Bigr]$,
where $\bbE^{\pi,p'} [\cdot]$ denotes the expectation with respect to policy $\pi$ and transition function $p'$,
that is, $a_{h'} \sim \pi_{h'}(\cdot \mid s_{h'})$ and $s_{h'+1} \sim p'_{h'}(\cdot \mid s_{h'},a_{h'})$.

\textbf{Learner-environment interaction.} 
At the beginning of episode $k$, the learner picks a policy $\pi^k$, and starts in an initial state $s^k_1 = \sinit$.
In each time $h\in [H]$, it observes the current state $s^k_h$, draws an action from the policy $a^k_h \sim \pi^k_h(\cdot | s_h^k)$ and transitions to the next state $s^k_{h+1} \sim p_h(\cdot | s^k_h,a^k_h)$. 
The feedback of episode $k$ contains the cost function over the agent's trajectory $\{ c^k_h(s^k_h,a^k_h) \}_{h=1}^H$, i.e., bandit feedback (as opposed to full-information feedback which contains the whole cost function). 
This feedback is observed only at the end of episode $k+d^k$, where the \textit{delays} $\{d^k\}_{k=1}^K$ are unknown and chosen by the oblivious adversary together with the costs. 
If $d^k = 0$ for all $k$, this model scales down to standard online learning in adversarial MDP.

\textbf{Occupancy measure.} 
Given a policy $\pi$ and a transition function $p'$, the \textit{occupancy measure} $q^{\pi,p'} \in [0,1]^{HS^2A}$ is a vector, where $q^{\pi,p'}_h(s,a,s')$ is the probability to visit state $s$ at time $h$, take action $a$ and transition to state $s'$. 
We also denote $q^{\pi,p'}_h(s,a) = \sum_{s'}q^{\pi,p'}_h(s,a,s')$ and $q^{\pi,p'}_h(s) = \sum_{a}q^{\pi,p'}_h(s,a)$. 
By \cite{rosenberg2019online}, the occupancy measure encodes the policy and the transition function through the relations
$
    \pi_{h}(a \mid s) 
    = \nicefrac{q^{\pi,p'}_{h}(s,a)}{q^{\pi,p'}_h(s)}
    ; \,\, 
    p'_h(s'\mid s,a) 
    = \nicefrac{q^{\pi,p'}_{h}(s,a,s')}{q^{\pi,p'}_{h}(s,a)}.
$
The set of all occupancy measures with respect to an MDP $\calM$ is denoted by $\Delta(\calM)$.
Importantly, the value of a policy from the initial state can be written as the dot product between its occupancy measure and the cost function, i.e., $V^{\pi,p'}_1(\sinit ; c) = \langle q^{\pi,p'} , c \rangle$.
Whenever $p'$ is omitted from the notations $q^{\pi,p'}$ and $V^{\pi,p'}$, this means that they are with respect to the true transition function $p$.

\textbf{Regret.} 
The learner's performance is measured by the \textit{regret} which is the difference between the cumulative expected cost of the learner and the best fixed policy in hindsight:
\begin{align*}
    \regret
    & =
    \sum_{k=1}^K V^{k,\pi^k}_1(\sinit) - \min_{\pi} \sum_{k=1}^K V^{k,\pi}_1(\sinit)
    =
    \sum_{k=1}^K \langle q^{\pi^{k}} , c^k \rangle - \min_{q \in \Delta(\calM)} \sum_{k=1}^K \langle q , c^k \rangle,
\end{align*}
where $V^{k,\pi}_h(s) = V^{\pi,p}_h(s ; c^k)$.

\textbf{Confidence set.} Since the transition function is unknown, we maintain standard Bernstein-based confidence sets $\calP^k$ for each episode $k$ that contain $p$ with high-probability. For the exact definition of $\calP^k$ see \cref{alg:update-cofidence-set,alg:update-cofidence-set delay}, and the fact that $p\in \calP^k$ for every $k$ w.h.p is proved for example in \cite{jin2019learning} (for more details see the appendix).
Using $\calP^k$ we can define a confidence set of occupancy measures by
\[
    \ocsetk{k} = \{q^{\pi,p'} \mid \pi\in (\Delta_\calA)^{\calS \times [H]}, p'\in \calP^{k} \},
\]
which is a polytope with polynomial constraints as shown in~\cite{rosenberg2019online}.
Note that as long as $p\in \calP^k$, $\ocset \subseteq \ocsetk{k}$.

\textbf{Additional notations.}
In general, episode indices always appear as superscripts and in-episode steps as subscripts.  
$\bar p^k_h(s'|s,a)$ is the empirical mean estimation of $p_h(s'|s,a)$ based on the trajectories available to the algorithm at the beginning of the episode $k$. $n_{h}^{k}(s,a,s')$ denotes the total number of visits at state $s$ in which the agent took action $a$ at time $h$ and transitioned to $s'$ by the end of episode $k - 1$, and  $n_{h}^{k}(s,a) = \sum_{s'} n_{h}^{k}(s,a,s')$.
Similarly,
$m_{h}^{k}(s,a,s')$ denotes the total number of visits from rounds  $j$ such that $j+d^j \leq k - 1$ at state $s$ in which the agent took action $a$ at time $h$ and transitioned to $s'$, and  $m_{h}^{k}(s,a) = \sum_{s'} m_{h}^{k}(s,a,s')$.
$\calF^k = \{ j: j+d^j = k\}$ denotes the set of episodes such that their feedback arrives in the end of episode $k$. 
The notations $\wt O(\cdot)$ and $\lesssim$ hide constant and poly-logarithmic factors including $\log (K/\delta)$ for some confidence parameter $\delta$, the indicator of event $E$ is denoted by $\indevent{E}$, and $x \vee y = \max \{x,y\}$.

\textbf{Simplifying assumptions.}
Throughout this paper we assume that $K$ and $D=\sum_{k=1}^K d^k$ are known and that the maximal delay $d_{max} = \max_k d^k \leq \sqrt{D}$. Both of these assumptions are made only for simplicity of presentation and can be easily relaxed using standard \textit{doubling} and \textit{skipping} procedures as shown for example by \cite{thune2019nonstochastic,lancewicki2020learning,bistritz2021no}.
In addition, we focus on the case of non-delayed trajectory feedback, where the learner observes the trajectory immediately at the end of the episode and only the feedback regarding the cost is delayed.
Delayed trajectory feedback mainly affects approximation errors and the ideas presented in 
\cite{lancewicki2020learning} for handling such delay apply to our case as well. Finally, the regret bounds in the main text hide low-order terms that depends polynomially in $H,S$ and $A$ but only poly-logarithmically in $K$ - the full bounds appear in the appendix.

\section{Delayed Hedge}
\label{sec-paper:hedge}

In this section, we consider running a Hedge-based algorithm over all $\Omega = \calA^{\calS \times [H]}$ deterministic policies.
\cref{paper-alg:hedge}, which we call Delayed Hedge, is inefficient but gives the first order-optimal regret bounds for adversarial MDP with delayed bandit feedback.
Although the main issue that Delayed Hedge tackles is delayed feedback, we note that there are many additional challenges introduced by the unknown transitions and the bandit feedback when we maintain a distribution over policies instead of a single stochastic policy. 

\begin{algorithm}[t]
	\caption{Delayed Hedge} 
	\label{paper-alg:hedge}
	\begin{algorithmic}[1]
		
		\STATE \textbf{Initialization:} Set $\probdist^1$ to be the uniform distribution over all deterministic policies, and $\calP^1$ to be set of all transitions functions.
		
		\FOR{$k=1,2,...,K$}
		
		\STATE Execute policy $\pi^k$ sampled from $\probdist^k$, observe trajectory $\{ s^k_h,a^k_h \}_{h=1}^H$.
		
		\STATE Update confidence set $\calP^k$, compute upper occupancy bound $u^k$ and exploration bonus $b^k$ by:
		\begin{align*}
		    u^k_h(s,a) = \max_{p' \in \calP^k} \sum_{\pi \in \Omega} \probdist^k(\pi) q_h^{\pi, p'}(s,a)
		    \quad ; \quad
		    b^k(\pi) = \max_{p' \in \calP^{k}} \lVert q^{\pi,\bar{p}^k } - q^{\pi,p'} \rVert_1.
		\end{align*}
		
		\FOR{$j : j + d^j = k$}
		\STATE Observe costs $\{ c^j_h(s^j_h,a^j_h) \}_{h=1}^H$, compute loss estimator $\hat c^j$ defined in \cref{eq:optimistic-imprtance-sampling}, and estimated loss by $\hatell^j(\pi) = \big\langle q^{\pi, \bar{p}^{j}}, \hatl^j\big\rangle.$
		
		\ENDFOR
		
		\STATE Update policy distribution $\probdist^{k+1}$ by:
		$
		\probdist^{k+1}(\pi) \propto \probdist^{k}(\pi) \cdot \exp\big({ \eta b^k(\pi)  - \eta \sum_{j:j+d^j = k} \hatell^j(\pi) }\big).
		$
		
		\ENDFOR
	\end{algorithmic}
\end{algorithm}

Delayed Hedge maintains a distribution $\probdist^k$ over deterministic policies (starting from a uniform distribution), and in the beginning of episode $k$ samples a policy $\pi^k$ to execute. 
Thus, the expected loss incurred in episode $k$ is $\sum_{\pi\in \Omega} \probdist^k(\pi) \inner{ q^{\pi,p},  c^k }$.
The algorithm updates the distribution $\probdist^k$ based on the exponential weights update, for which we need to compute an estimated loss for every policy $\pi \in \Omega$.

To do so, first we estimate the cost in each state-action pair. Due to unknown dynamics, following~\cite{jin2019learning} we use the confidence sets to compute optimistic importance weighted estimator that will induce exploration:
\begin{align}
    \label{eq:optimistic-imprtance-sampling}
    \hat c_h^k(s,a) 
    = 
    \frac{ c_h^k(s,a) \indevent{s^k_h = s , a^k_h = a}  }{u^k_h(s,a) + \gamma},
\end{align}
where $u_h^k(s,a) = \max_{p' \in \calP^k} \sum_{\pi\in \Omega} \probdist^k(\pi) q_h^{\pi,p'}(s,a)$ is an upper occupancy bound on the probability to visit $(s,a)$ in step $h$ of episode $k$, and $\gamma$ is a small bias added for high probability regret \cite{neu2015explore}.

Then, we use the empirical transition function $\bar{p}^k$ to compute the estimated loss $\hatell^k(\pi) = \bigl\langle q^{\pi, \bar{p}^{k}}, \hatl^k\bigr\rangle$ for each policy $\pi$.
To ensure optimism, we introduce the exploration bonus $b^{k}(\pi) = \max_{p' \in \calP^{k}} \lVert q^{\pi,\bar{p}^k } - q^{\pi,p'} \rVert_1$.
As long as the real transition function $p$ is in  confidence set $\calP^k$, optimism is indeed ensured in the sense that $\bigl\langle q^{\pi,\bar{p}^k }, c \bigr\rangle - b^k(\pi)$ is always no more than the true cost $\bigl\langle q^{\pi,p}, c \bigr\rangle$ for any policy $\pi$ and $[0,1]$-valued cost function $c$.

With the estimated loss and the exploration bonus for each $\pi$, the distribution $\probdist^{k+1}$ is now updated in a manner similar to that of \cite{gyorgy2020adapting}: $\probdist^{k+1}(\pi) \propto \probdist^{k}(\pi) \cdot \exp\big( \eta b^k(\pi) - \eta \sum_{j:j+d^j = k} \hatell^j(\pi)\big)$.
Note that all information required for this update has been received by the learner at the end of episode $k$. 
With the help of all these definitions, we prove the following regret bound for Delayed Hedge, and defer the details to \cref{appendix:delay-hedge} including the complete algorithm and regret analysis. 

\begin{theorem}\label{paper-thm:regret-hedge} 
With appropriate choices of parameters,
Delayed Hedge ensures 
	$
	\regret
	=
	\wt O\rbr{ H^2 S \sqrt{  A K  }+  H^{\nicefrac{3}{2}}\sqrt{SD} }
	$
	with high probability (w.h.p.).
\end{theorem}

\section{Delayed UOB-FTRL}
\label{sec-paper:ftrl}

In this section, we adjust the UOB-REPS algorithm \cite{jin2019learning} to delayed feedback and present the Delayed UOB-FTRL algorithm (\cref{paper-alg:ftrl-normal-estimator}) - the first efficient algorithm to attain order-optimal regret for adversarial MDP with delayed bandit feedback.
The proof is based on a novel analysis without additional changes to the algorithm.
Namely, we use standard loss estimators (defined in \cref{paper-eq:ftrl-loss-estimator}). 
Our algorithm is based on the Follow-the-Regularized-Leader (FTRL) framework, which is widely used for deriving online learning algorithm in adversarial environments. 
Notable examples are \cite{zimin2013online} that applies FTRL over occupancy measure space to solve the adversarial MDP problem with known transition, and \cite{zimmert2020optimal} that uses FTRL to achieve optimal regret for MAB with delayed feedback. 

\begin{algorithm}[t]
	\caption{Delayed UOB-FTRL} 
	\label{paper-alg:ftrl-normal-estimator}
	\begin{algorithmic}[1]
		
		\STATE \textbf{Initialization:} Set $\pi^1$ to be uniform policy, and $\calP^1$ to be set of all transitions functions..
		
		\FOR{$k=1,2,...,K$}
		
		\STATE Execute policy $\pi^k$, observe trajectory $\{ s^k_h,a^k_h \}_{h=1}^H$, update confidence set $\calP^k$ and compute upper occupancy bound $u^k_h(s,a) = \max_{p' \in \calP^k} q^{\pi^k,p'}_h(s,a)$.
		
		\FOR{$j : j + d^j = k$}
		\STATE Observe costs $\{ c^j_h(s^j_h,a^j_h) \}_{h=1}^H$ and compute the standard loss estimator $\hat c^j$ by 
		\begin{align}\label{paper-eq:ftrl-loss-estimator} \hat c_h^j(s,a) = \frac{ c_h^j(s,a) \indevent{s^j_h = s , a^j_h = a}  }{u^j_h(s,a) + \gamma}. \end{align}
		
		\ENDFOR
		
		\STATE Compute occupancy measure by:
		$
    		{q^{k+1} = \argmin_{q\in \cap_{j=1}^{k+1} \ocsetk{j}} \bigl\langle q , \sum_{j+d^j\leq k} \hat c^j \bigr\rangle + \phi(q),}
		$
		where $\phi(q) = \frac{1}{\eta} \sum_{h,s,a,s'} q_{h}(s,a,s')\log q_{h}(s,a,s')$.
		
		\STATE Update policy: $\pi_{h}^{k+1}(a\mid s)
		=\nicefrac{q_{h}^{k+1}(s,a)}{q_{h}^{k+1}(s)}$.
		\ENDFOR
	\end{algorithmic}
\end{algorithm}

In our context, in the beginning of episode $k$, FTRL computes,
\begin{align}
    \label{eq:FTRL update}
    q^k = \argmin_{q\in \cap_{j=1}^k \ocsetk{j} } \bigl\langle q, \hatL^{obs}_k\bigr\rangle + \phi(q),
\end{align}
where $\hatL^{obs}_k = \sum_{j+d^j <k} \hatl^j$ is the cumulative losses observed prior to episode $k$, and $\phi(q) = \frac{1}{\eta} \sum_{h,s,a,s'} q_{h}(s,a,s')\log q_{h}(s,a,s')$ is the Shannon entropy regularizer. 
Note that \cref{eq:FTRL update} is a convex optimization problem with 
linear constraints and thus can be solved efficiently \cite{zimin2013online,rosenberg2019online}.
The policy $\pi^k$ to be played in the episode is then extracted from $q^k$. 
Thus, our algorithm can be regarded as a direct extension to MDP of FTRL for delayed feedback.
However, unlike the successes in MAB, it is highly unclear whether optimal regret could be obtained in adversarial MDPs with FTRL even if the transition function is known. 

In \cref{paper-thm:regret-ftrl}, we show that Delayed UOB-FTRL enjoys order-optimal regret.
Through the key steps of the analysis, we shall take a closer look at the key reason why traditional analysis fails: in occupancy measure space, the interplay between different entries of loss functions is significantly harder to analyze. Thus, many critical properties used in \cite{zimmert2020optimal} do not hold anymore. 
The complete algorithm and proof are deferred to \cref{appendix:delay-ftrl}. 

\begin{theorem}\label{paper-thm:regret-ftrl} With appropriate choices of parameters, Delayed UOB-FTRL (\cref{paper-alg:ftrl-normal-estimator}) ensures 
$
\mathbb{E} \sbr{\regret}
=
\wt O \bigl( H^2 S\sqrt{AK} + HSA \sqrt{HD} \bigr).
$
\end{theorem}
\begin{proof}[Proof sketch of \cref{paper-thm:regret-ftrl}]
	Let $q^\star = q^{\pi^\star,p}$ be the occupancy measure associated with the optimal policy $\pi^\star$. 
	We adopt the regret decomposition of \cite{jin2019learning}:
	\begin{equation}
	\begin{aligned}
	\regret  & = \underbrace{\sum_{k=1}^K \left\langle q^{\pi^k} - \estq^k, c^k \right\rangle}_{\textsc{Est}}
	+ \underbrace{\sum_{k=1}^{K} \left\langle \estq^k , c^k  - \hatl^k \right\rangle}_{\regtwo} \notag 
	+ \underbrace{\sum_{k=1}^K \left\langle \estq^k - q^\star, \hatl^k \right\rangle}_{\regthree}
	+ \underbrace{\sum_{k=1}^K \left\langle q^\star, \hatl^k - c^k \right\rangle}_{\regfour}. \notag 
	\end{aligned}
	\label{paper-eq:ftrl-regret-decomposition}
	\end{equation}
	$\textsc{Est}$, $\regtwo$ and $\regfour$ are standard and bounded in \cite{jin2019learning} w.h.p by $\wt O ( \gamma HSAK+ H^2S\sqrt{AK} + H/\gamma)$. 
	
	Now, we focus on bounding $\regthree$. To this end, we denote by $\hatL_k = \sum_{j=1}^{k-1} \hatl^k$ the non-delayed cumulative loss, and introduce the convex conjugate functions $F_k^\star$ with respect to the regularizer $\phi(\cdot)$: 
	\begin{align*}
	F_k^\star(x) = -\min_{q\in \Delta\rbr{\calM,k}} \cbr{ \phi(q)  - \inner{x,q} }.
	\end{align*}
	We now use $F_k^\star$ to decompose $\regthree$ into the following three terms as
	\begin{align}
        \nonumber
        & \sum_{k=1}^{K} - F_k^\star(-\hatL^{\text{obs}}_k)  + \inner{ \estq^k, \hatl^k } +  F_k^\star(-\hatL^{\text{obs}}_k- \hatl^k )+ \sum_{k=1}^{K}  - F_k^\star( - \hatL_k - \hatl^k ) +  F_k^\star( - \hatL_k )  - \inner{q^\star, \hatl^k} 
        \\
    	&
    	+ \sum_{k=1}^{K} \Big\{- F_k^\star( - \hatL_k^{\text{obs}} - \hatl^k  ) + F_k^\star( - \hatL_k^{\text{obs}}  ) - \rbr{  - F_k^\star( - \hatL_k - \hatl^k ) +  F_k^\star( - \hatL_k )  } \Big\}. 
    	\label{paper-eq:ftrl-regret-estreg-decomposition}
	\end{align}
	The first term is associated with the unseen loss $\hatl^k$.
	It is relatively standard and bounded by $\wt O( \eta HSAK )$ w.h.p. 
	The second term can be regarded as the regret of a ``cheating'' algorithm which does not suffer delay and sees one step into the future. This term can be bounded by $\wt O(H/\eta)$ similarly to \cite{gyorgy2020adapting}.
	The third term which only relates to delayed feedback, is the most critical object in the analysis. 
	
	In the previous work of \cite{zimmert2020optimal} for multi-arm bandit, the authors managed to rewrite and then upper bound the delay-caused term for every episode $k$ by 
	\begin{align*} 
	& \int_{0}^{1} \inner{ \hatl^k,  \nabla F_k^\star(-\hatL_k^{\text{obs}} -x\hatl^k )  - \nabla F_k^\star(-\hatL_k-x\hatl^k)  } dx  
	\leq \eta \sum_{i \in [N]} p^k(i) \cdot \hatl^k(i) \cdot \rbr{\hatL_k(i) - \hatL^{\text{obs}}_k(i)},
	\end{align*}
	where $[N]$ is the set of arms and $p^k(i)$ is the probability that the algorithm chooses arm $i$ in episode $k$.
	Here, the first step uses Newton-Leibniz theorem and the differentiability of convex conjugates, and the second step follows directly from \cite[Lemma 3]{zimmert2020optimal}. 
	Importantly, the second step is largely based on the specific structure of the simplex (over which MAB algorithms operate), which yields the  simple behavior of FTRL-based algorithms (e.g., EXP3).
	Specifically, it is based on the following observation.
	Suppose that we increase the cumulative loss of arm $i$. Now consider the behavior of $p(i')$, the probability of taking arm $i'$ where $p$ is computed from the FTRL framework. 
	One can verify that $p(i')$ will increase for $i' \neq i$ and decrease for $i' = i$. 
	In other words, the relationship between any pair of arms is competitive, and this property is critical to achieve the optimal regret with delayed feedback in \cite{zimmert2020optimal}.
	
	However, this property does not hold for MDPs because the constraints of the transition function can dictate positive correlation between entries of the occupancy measure. 
	Similarly, consider two state-action pairs $(s,a,h)$ and $(s',a',h')$ from different states. 
	It is highly unclear whether increasing the cumulative loss of $(s,a,h)$ will increase or decrease the probability $q_{h'}(s',a')$ of reaching $s'$ in time $h'$ and taking action $a'$.
	In fact, the relation is related to the specific transition function of the MDP.
	For example, the FTRL algorithm will decrease the probability in the cases where taking action $a$ at state $s$ in step $h$ is necessary to reach $(s',a',h')$, and will increase in other cases where not taking action $a$ at state $s$ of step $h$ is necessary. 
	
	Therefore, an alternative analysis is required in our case.
	Specifically, we are able to bound the delayed-caused term by
	\begin{align*} 
	\int_{0}^{1} & \inner{ \hatl^k,  \nabla F_k^\star(-\hatL_k^{\text{obs}} -x\hatl^k )  - \nabla F_k^\star(-\hatL_k-x\hatl^k)  } dx  
	\leq 2 \norm{\hatl^k}_{\nabla^{-2} \phi(\xi)} \norm{\hatL_k - \hatL_k^{\text{obs}} }_{\nabla^{-2} \phi(\xi)} 
	\\
	& \qquad \qquad \qquad \qquad
	\leq 2 \eta \sum_{j=1, j+d^j\geq k}^{k-1} \rbr{ \sum_{h,s,a} \hatl^k_h(s,a) } \cdot \rbr{ \sum_{h,s,a} \hatl^j_h(s,a) } 
	\end{align*} 
	where the first step uses the properties of convex conjugates for some valid occupancy measure $\xi$ (See \cref{lemma:delay-drift-ftrl-normal-loss-estimator} for more details) with $\norm{x}_M = \sqrt{x^\top M x}$ being the matrix norm for any vector $x$ and positive definite matrix $M$, and the second step follows from the facts that $\nabla^{-2} \phi(\xi)$ is a diagonal matrix with values $\cbr{\eta \cdot \xi_h(s,a): \forall (h,s,a)}$ on its diagonal and $\xi_h(s,a) \leq 1$. 
	
	While we managed to overcome the complex dependencies between different states in the MDP, it comes at the price of a looser regret bound.
	The final bound does not have $q^k_h(s,a)$ in the summations which leads to an extra factor of $SA$. 
	This follows from the application of H\"older's inequality and also the relaxation of intermediate occupancy measure $\xi$. 
	
	Taking the summation over all episodes, we have that the third term in \cref{paper-eq:ftrl-regret-estreg-decomposition} is bounded by $\wt O( \eta H^2S^2A^2 D )$ in expectation. 
    Finally, with proper choice of the parameters $\eta$, $\gamma$ and $\delta$, combining the bounds for $\textsc{Est}$, $\regtwo$, $\regfour$ and the three terms in \cref{paper-eq:ftrl-regret-estreg-decomposition} finishes the proof. 
\end{proof}

\section{Delayed UOB-REPS with Delay-adapted Estimator}
\label{sec-paper:UOB-REPS}
Finally, we present our last algorithm, Delayed UOB-REPS equipped with our novel importance sampling estimator which we call \emph{delay-adapted importance sampling estimator}. The algorithm appears as \cref{paper-alg:o-reps-new-estimator} and in its full version together with the analysis for known and unknown dynamics in \cref{appendix:delay-adapted O-REPS known,appendix:delay-adapted O-REPS unknown}. 

Much like Delayed UOB-FTRL, the algorithm is efficient; but it outperforms Delayed UOB-FTRL in two important aspects: (i) it guarantees high-probability regret bound (and not only expected regret), and (ii) the delay term in its regret bound is tighter. 
In fact, as long as $A \leq S$ (which happens in most cases), it obtains an improvement even on the regret of the inefficient Delayed Hedge algorithm.

To maintain the occupancy measures $q^k$ from which the executed policies $\pi^k$ are extracted, Delayed UOB-REPS uses the Online Mirror Decent (OMD) update rule:
\[
    q^{k+1}
    =
    \argmin_{q \in \ocsetk{k+1}} \eta \Bigl\langle q , \sum_{j : j+d^j = k} \hat c^j \Bigr\rangle + \KL{q}{q^k},
\]
where $\eta$ is a learning rate and $\KL{q}{q'}$ is the unnormalized KL-divergence 
(see the full algorithm in \cref{appendix:delay-adapted O-REPS unknown} for the definition of KL-divergence).
We note that OMD is standard in the O-REPS literature, and has similar guarantees to FTRL.
In this case, OMD will be much more useful than FTRL because we can utilize its update rule to prove certain properties for the relation between consecutive occupancy measures (see \cref{lemma:adjacent-o-mesure-KL-bound-unknown-p}).

We do not use the standard importance sampling estimator, 
but the following delay-adapted estimator:
\begin{align}
    \label{paper-eq:delay-adapted estimator}
    \hat c^k_h(s,a) = \frac{c^k_h(s,a) \indevent{s^k_h = s , a^k_h = a}}{\max \{ u^k_h(s,a) , u^{k+d^k}_h(s,a) \} + \gamma}.
\end{align}
The delay-adapted estimator specifically tackles one of the main technical challenges in analyzing algorithms under delayed feedback (especially in MDPs) -- bound their stability.
It is a biased estimator, and in fact has larger bias than the standard importance sampling estimator, but allows us to directly control the stability of the algorithm. 

To describe the intuition behind the delay-adapted estimator, let us first consider a fixed delay $d^k=d$. The policy $\pi^{k+d}$ is updated based on the episodes $1,...,k-1$.
Thus, playing $\pi^{k+d}$ at episode $k$ is equivalent to running OMD on the same loss estimators but in a non-delayed environment. 
Standard analysis for delayed feedback (e.g., \cite{thune2019nonstochastic,bistritz2019online} for MAB or \cite{lancewicki2020learning} for MDPs) utilizes this fact to bound the regret with respect to the estimated cost by the sum of: (i) the regret of playing $\pi^{k+d}$; (ii) the ``drift'' between the playing $\pi^{k+d}$ and $\pi^{k}$:
\begin{align}
    \label{eq:delay-adapted-estimator-intuition}
    \sum_{k=1}^{K}\langle q^{k} - q^\star,\hat{c}^{k}\rangle
    & \lesssim
    \underbrace{
    \sum_{k=1}^{K}\langle q^{k} - q^{k+d},\hat c^{k} \rangle
    }_{\textsc{Drift}}
    + \frac{H}{\eta} 
    + 
    \underbrace{ 
    \eta\sum_{h,s,a,k}q^{k+d}_h(s,a) \hat c^{k}_h(s,a)^2
    }_{\textsc{Stability}}.
\end{align}

\begin{algorithm}[t]
    \caption{Delayed UOB-REPS with Delay-adapted Estimator} 
    \label{paper-alg:o-reps-new-estimator}
    \begin{algorithmic}[1]
        
        \STATE \textbf{Initialization:} Set $\pi^1$ to be uniform policy.
        
        \FOR{$k=1,2,...,K$}
        
            \STATE Execute policy $\pi^k$, observe trajectory $\{ s^k_h,a^k_h \}_{h=1}^H$, update confidence set $\calP^k$ and compute upper occupancy bound $u^k_h(s,a) = \max_{p' \in \calP^k} q^{\pi^k,p'}_h(s,a)$.
            
            \FOR{$j : j + d^j = k$}
                \STATE Observe costs $\{ c^j_h(s^j_h,a^j_h) \}_{h=1}^H$ and compute the delay-adapted cost estimator $\hat c^j$ by \cref{paper-eq:delay-adapted estimator}.
                
            \ENDFOR
            
            \STATE Update occupancy measure by:
            $
                {q^{k+1} = \argmin_{q \in \ocsetk{k+1}} \eta \left\langle q , \sum_{j \in \calF^k} \hat c^j \right\rangle + \KL{q}{q^k}.}
            $
            
            \STATE Update policy: $\pi_{h}^{k+1}(a\mid s)
            =\nicefrac{q_{h}^{k+1}(s,a)}{q_{h}^{k+1}(s)}$.
        \ENDFOR
    \end{algorithmic}
\end{algorithm}

The term $\frac{H}{\eta}$ is usually referred to as the \textsc{Penalty}, and the bound $\text{(i)} \leq \textsc{Penalty} + \textsc{Stability}$ is by standard OMD guarantees. 
The \textit{standard} importance sampling estimator defined in \cref{paper-eq:ftrl-loss-estimator} is approximately unbiased (ignoring $\gamma$ and transition approximation errors), so the left-hand-side of \cref{eq:delay-adapted-estimator-intuition} is approximately the regret in expectation. 
On the other hand, to bound the \textsc{Stability} term, one needs to control the ratio $\nicefrac{q^{k+d}_h(s,a)}{q^{k}_h(s,a)}$ since $\hat c^k_h(s,a)$ has $q^k_h(s,a)$ in the denominator and not $q^{k+d}_h(s,a)$ (for simplicity we ignore the bias between $q^k$ and $u^k$).

In MAB, this ratio is essentially bounded by a constant, but the proof heavily relies on the simple update form of OMD on the simplex (i.e., EXP3), as explained in \cref{sec-paper:ftrl}. 
However, it still remains unclear whether this ratio is bounded by a constant when running OMD or FTRL on a more general convex set such as $\ocset$. 
While in the proof of \cref{paper-thm:regret-hedge,paper-thm:regret-ftrl} we are able to avoid bounding the ratio in the stability term itself by using a ``cheating" regret approach, a similar issue re-appears in the drift term. In \cref{paper-thm:regret-hedge} we bound the ratio between distributions by utilizing the simple update form (for the specific argument see \cref{eq:hedge-bounding ratio} in \cref{appendix:delay-hedge}), and in \cref{paper-thm:regret-ftrl} we solve this issue with the help of convex conjugates (specifically, H\"older's inequality with respect to the Hessian of the regularizer at an intermediate occupancy measure $\xi$), but this comes at the cost of expected regret guarantees and looser bound on the delay term of the regret.

The main idea of the delay-adapted estimator is to re-weight the cost of episode $k$ using both $q^{k+d}$ and $q^{k}$.
The first allows us to control the stability and avoids the need to bound the ratio $\nicefrac{q^{k+d}_h(s,a)}{q^{k}_h(s,a)}$, while the second keeps the bias sufficiently small.
More precisely, we re-weight using their maximum, which 
remarkably, causes the estimator's bias to scale similarly to the \textsc{Drift} term. 

Finally, there are a few important points to notice with respect to our new estimator before we analyze the regret of \cref{paper-alg:o-reps-new-estimator} in \cref{paper-thm:regret-o-reps-new-estimator}.
First, since the estimator $\hat c^k$ is computed only in the end of episode $k+d^k$ (when the feedback from episode $k$ arrives), we have already computed both $u^k$ and $u^{k+d^k}$ at that point and the estimator is well-defined.
Second, it generalizes the standard importance sampling estimator and adapts it to the delays.
That is, whenever there is no delay, our estimator is identical to the standard importance sampling estimator.
Third, there is no additional computational cost in computing the new estimator since we compute $u^k$ for every $k$ anyway. 
Moreover, there is no additional space complexity because every algorithm for adversarial environments with delayed feedback keeps the probabilities to play actions in episode $k$ until its feedback is received in the end of episode $k+d^k$.


\begin{theorem}
    \label{paper-thm:regret-o-reps-new-estimator}
    With appropriate choices of parameters,
    Delayed UOB-REPS with the delay-adapted estimator (\cref{paper-alg:o-reps-new-estimator}) ensures with high probability that
    $
        \regret
        =
        \wt O \bigl( H^2 S\sqrt{AK } + (H S A)^{1/4} \cdot H \sqrt{D} \bigr)
    $.
\end{theorem}

The second term in the regret improves the guarantee of Delayed UOB-FTRL with the standard estimator by a factor of $H^{1/4}(SA)^{3/4}$.
It also improves Delayed Hedge by $(HS)^{1/4}$, but on the other hand has an extra factor $A^{1/4}$. 
Generally, this term is tight up to the $(HSA)^{1/4}$ factor \cite{lancewicki2020learning}.
The first term in the regret matches the state-of-the-art regret bound for non-delayed adversarial MDPs \cite{jin2019learning}.
In \cref{appendix:delay-adapted O-REPS known} we consider the case of known transitions, and present Delayed O-REPS with the delay-adapted estimator that achieves the following regret bound.
It has similar delay term but its first term is optimal up to poly-log factors \cite{zimin2013online}. 

\begin{theorem}
    \label{paper-thm:regret-o-reps-new-estimator-known}
    Assume that the transition function is known to the learner.
    With high probability, Delayed O-REPS with the delay-adapted estimator (\cref{alg:o-reps-new-estimator}) ensures that
    $
        \regret
        =
        \wt O \bigl( H\sqrt{SAK } + (H S A)^{1/4} \cdot H \sqrt{D} \bigr).
    $
\end{theorem}

We conclude the section with a proof sketch of our main theorem (for the unknown transition case).
\begin{proof}[Proof sketch of \cref{paper-thm:regret-o-reps-new-estimator}]
    We first break the regret as follows:
    \begin{align*} 
        \regret &= \underbrace{\sum_{k=1}^{K}\langle q^{\pi^k} - q^k,c^{k}\rangle }_{\textsc{Est}}
        +
        \underbrace{\sum_{k=1}^{K}\langle q^{k},c^{k}-\hat{c}^{k}\rangle }_{\textsc{Bias}_1}
        + 
        \underbrace{\sum_{k=1}^{K}\langle q^\star,\hat{c}^{k}-c^{k}\rangle}_{\textsc{Bias}_2}
        \\ 
        &\quad
        +
        \underbrace{\sum_{k=1}^{K}\langle q^{k} - q^{k+d^k},\hat{c}^{k}\rangle}_{\textsc{Drift}}
        +
        \underbrace{\sum_{k=1}^{K}\langle q^{k+d^k} - q^\star,\hat{c}^{k}\rangle}_{\textsc{\textsc{Reg}}}.
    \end{align*}
    \textsc{Est} is the standard transition approximation error term which is bounded w.h.p by $\wt O(H^2 S \sqrt{AK})$ \cite{jin2019learning}.
    For $\textsc{Bias}_2$ we use the fact that the delay-adapted estimator is always smaller than the standard estimator and bound it by $\wt O(H/\gamma)$ similarly to \cite{jin2019learning}.
    
    The real advantage of the estimator appears in the $\textsc{Reg}$ term.
    Similar to the fixed delay case, we can bound $\textsc{Reg}$ by, 
    \begin{align*}
        \frac{H}{\eta}
        +
        \underbrace{
        \eta \sum_{k,h,s,a} q_{h}^{k+d^{k}}(s,a) \hat c^k_h(s,a) \Bigg(
        \sum_{j \in \calF^{ k+d^{k}}}\hat c_{h}^{j}(s,a)
        \Bigg)
        }_{\textsc{Stability}}
        \leq
        \frac{H}{\eta}
        + \eta \sum_{k,h,s,a} 
        \sum_{j \in \calF^{ k+d^{k}}}\hat c_{h}^{j}(s,a)
        ,
    \end{align*}
    where the inequality above is exactly where we utilize the delay-adapted estimator, as by its definition $\hat c^k_h(s,a) \leq 1/u_{h}^{k+d^{k}}(s,a) \leq 1/q_{h}^{k+d^{k}}(s,a),$
    where the last inequality holds w.h.p.
    Then, using a standard concentration of $\hat c^k_h(s,a)$ around $c^k_h(s,a) \leq 1$ we get that ${\textsc{Stability} \lesssim \eta (HSAK + d_{max}/\gamma)}$.
    Importantly, the concentration arguments hold only because the maximum of $u^k$ and $u^{k+d^k}$ appears in the estimator's denominator.
    If it were only $u^{k+d^k}$, we could not have bounded the distance between the estimator $\hat c^k$ and the real cost $c^k$. 
    
    For the $\textsc{Drift}$ term, let $\wt{H}^{k}$ be the realization of all episodes $j$ such that $j + d^j < k$. 
    Note that $u^k$ and $u^{k+d^k}$ are completely determined by the history $\wt{H}^{k + d^k}$, and on the other hand, the $k$-th episode is not part of this history. 
    Next, we take the absolute value on each element of $q^{k} - q^{k+d^{k}}$ and apply a concentration bound 
    to obtain:
    $
        \textsc{Drift}
        \lesssim 
        \sum_{k=1}^K \bbE\big[\langle |q^{k} - q^{k+d^{k}}| , \hat{c}^{k}\rangle \mid \wt{H}^{k+d^{k}} \big] + \frac{H}{\gamma}.
    $
    \\
    The specific definition of the history $\wt{H}^{k+d^k}$ is crucial because now we have:
    \begin{align*}
        & \textsc{Drift}
        \lesssim 
        \sum_{k=1}^K \bbE\left[\langle |q^{k} - q^{k+d^{k}}| , \hat{c}^{k}\rangle \mid \wt{H}^{k+d^{k}} \right] + \frac{H}{\gamma} 
        =
        \sum_{k=1}^K \langle |q^{k} - q^{k+d^{k}}| , \bbE\left[ \hat{c}^{k} \mid \wt{H}^{k+d^{k}} \right] \rangle + \frac{H}{\gamma} 
        \\
        & \qquad \leq 
        \sum_{k=1}^K \lVert q^{k}-q^{k+d^{k}}\rVert_1  + \frac{H}{\gamma} 
        \leq \sum_{k=1}^K \sum_{j=1}^{d^k} \lVert q^j - q^{j+1} \rVert_1 + \frac{H}{\gamma}
        \lesssim 
        \sum_{k=1}^K \sum_{j=1}^{d^k} \sqrt{\KL{q^j}{q^{j+1}}}+ \frac{H}{\gamma}, 
    \end{align*}
    where the third step follows since w.h.p
    $
        \bbE\big[ \hat{c}^{k}_h(s,a) \mid \wt{H}^{k+d^{k}} \big] = \frac{q^{\pi^k}_h(s,a) c^k_h(s,a)}{\max \{ u^k_h(s,a) , u^{k+d^k}_h(s,a) \}} \leq 1
    $, the fourth step uses the triangle inequality, and the last is by Pinsker inequality.
    Finally, we utilize the OMD update (which uses KL as regularization) to obtain a bound on $\KL{q^j}{q^{j+1}}$ and finally a bound $\wt O (\eta \sqrt{H^3 S A}(D+K)+H/\gamma)$ on the \textsc{Drift} term.
    For $\textsc{Bias}_1$, we apply a similar concentration on the cost estimators around ${\bbE\big[ \hat{c}^{k} \mid \wt{H}^{k+d^{k}} \big]}$ and show that $\textsc{Bias}_1$ is mainly bounded by,
    \begin{align*}
        \sum_k \lVert \max\{u^{k+d^k},u^{k}\} - q^{k} \rVert_1 + \gamma H S A K
         \leq  
        2 \sum_k \lVert u^{k} - q^{k}\rVert_1 + \sum_k \lVert q^{k+d^k} - q^{k}\rVert_1 + \gamma H S A K,
    \end{align*}
    where the maximum is taken element-wise. For last, the first sum is bounded similarly to the \textsc{Est} term while the second sum is bounded similarly to the \textsc{Drift} term. Summing the regret from the different terms and optimizing over $\eta$ and $\gamma$ completes the proof. 
\end{proof}

\section{Conclusions and Future Work}

In this paper we made a substantial contribution to the literature on delayed feedback in RL.
We presented the first algorithms that achieve near-optimal regret bounds for the challenging setting of adversarial MDP with delayed bandit feedback.
Our key algorithmic contribution is a novel delay-adapted importance sampling estimator, and we develop various new techniques to analyze delayed bandit feedback in adversarial MDPs.

We leave a few interesting questions open for future work.
First, there is still a gap of $(HSA)^{1/4}$ in the delay term between our upper bounds and the lower bound of \cite{lancewicki2020learning}.
Second, it remains an open question whether our new estimator is necessary to obtain optimal regret in the presence of delays, or is it possible to achieve optimal regret with standard algorithms.
Finally, our algorithms are based on the O-REPS framework but it remains an important open  problem to achieve $\wt O\rbr{\sqrt{K+D}}$ regret with policy optimization (PO) methods that are widely used in practice, and were recently shown to achieve near-optimal regret in adversarial MDP with non-delayed bandit feedback \cite{luo2021policy}.

\section*{Acknowledgements}
TL, YM and AR have received funding from the European Research Council (ERC) under the European Union’s Horizon 2020 research and innovation program (grant agreement No. 882396), by the Israel Science Foundation (grant number 993/17), Tel Aviv University Center for AI and Data Science (TAD), and the Yandex Initiative for Machine Learning at Tel Aviv University.
HL is supported by NSF Award IIS-1943607 and a Google Faculty Research Award.


\bibliography{main}
\bibliographystyle{abbrv}

\onecolumn
\newgeometry{
    textheight=9in,
    textwidth=6.5in,
    top=1in,
    headheight=12pt,
    headsep=25pt,
    footskip=30pt
}
\newpage
\clearpage
\appendix
\renewcommand{\partname}{}
\renewcommand{\thepart}{}
\part{Appendix} 

\parttoc
\newcommand{\theset}[2]{ \left\{ {#1} \,:\, {#2} \right\} }
\newcommand{\Ind}[1]{ \field{I}{\{{#1}\}} }
\newcommand{\eye}[1]{ \boldsymbol{I}_{#1} }
\newcommand{\trace}[1]{\textsc{tr}({#1})}
\newcommand{\diag}[1]{\mathrm{diag}\!\left\{{#1}\right\}}

\newcommand{\defeq}{\stackrel{\rm def}{=}}
\newcommand{\sgn}{\mbox{\sc sgn}}
\newcommand{\scI}{\mathcal{I}}
\newcommand{\scO}{\mathcal{O}}
\newcommand{\scN}{\mathcal{N}}

\newcommand{\dt}{\displaystyle}
\renewcommand{\ss}{\subseteq}
\newcommand{\ve}{\varepsilon}
\newcommand{\hlambda}{\wh{\lambda}}
\newcommand{\yhat}{\wh{y}}
\newcommand{\wb}{\overline}


\newcommand{\agmin}[1]{\underset{#1}{\argmin\,} }
\newcommand{\agmax}[1]{\underset{#1}{\argmax\,} }
\newcommand{\dotprod}[2]{\langle{#1},{#2}\rangle}

\newcommand{\push}{\hspace{0pt plus 1 filll} }	

\newcommand{\calL}{{\mathcal{L}}}
\newcommand{\calX}{{\mathcal{X}}}
\newcommand{\calY}{{\mathcal{Y}}}
\newcommand{\calI}{{\mathcal{I}}}
\newcommand{\calD}{{\mathcal{D}}}
\newcommand{\calK}{{\mathcal{K}}}
\newcommand{\calR}{{\mathcal{R}}}
\newcommand{\calT}{{\mathcal{T}}}
\newcommand{\calZ}{{\mathcal{Z}}}
\newcommand{\calN}{{\mathcal{N}}}
\newcommand{\calV}{{\mathcal{V}}}
\newcommand{\calQ}{{\mathcal{Q}}}
\newcommand{\ips}{\wh{r}}
\newcommand{\whpi}{\wh{\pi}}
\newcommand{\whE}{\wh{\E}}
\newcommand{\whV}{\wh{V}}
\newcommand{\reg}{{\mathcal{R}}}
\newcommand{\breg}{{\mathcal{\bar{R}}}}
\newcommand{\hmu}{\wh{\mu}}
\newcommand{\tmu}{\wt{\mu}}
\newcommand{\one}{\boldsymbol{1}}
\newcommand{\loss}{\ell}
\newcommand{\hloss}{\wh{\ell}}
\newcommand{\bloss}{\bar{\ell}}
\newcommand{\tloss}{\wt{\ell}}
\newcommand{\htheta}{\wh{\theta}}

\newcommand{\paralog}{\beta}
\newcommand{\cnt}{B}
\newcommand{\gapmin}{\Delta_{\textsc{min}}}

\newcommand{\bz}{\boldsymbol{z}}
\newcommand{\bx}{\boldsymbol{x}}
\newcommand{\br}{\boldsymbol{r}}
\newcommand{\bX}{\boldsymbol{X}}
\newcommand{\bu}{\boldsymbol{u}}
\newcommand{\by}{\boldsymbol{y}}
\newcommand{\bY}{\boldsymbol{Y}}
\newcommand{\bg}{\boldsymbol{g}}
\newcommand{\ba}{\boldsymbol{a}}
\newcommand{\be}{\boldsymbol{e}}
\newcommand{\bq}{\boldsymbol{q}}
\newcommand{\bp}{\boldsymbol{p}}
\newcommand{\bZ}{\boldsymbol{Z}}
\newcommand{\bS}{\boldsymbol{S}}
\newcommand{\bw}{\boldsymbol{w}}
\newcommand{\bW}{\boldsymbol{W}}
\newcommand{\bU}{\boldsymbol{U}}
\newcommand{\bv}{\boldsymbol{v}}
\newcommand{\bzero}{\boldsymbol{0}}

\newcommand{\blue}[1]{{\color{blue}#1}}

\newcommand{\field}[1]{\mathbb{#1}}
\newcommand{\fY}{\field{Y}}
\newcommand{\fX}{\field{X}}
\newcommand{\fH}{\field{H}}
\newcommand{\fR}{\field{R}}
\newcommand{\fB}{\field{B}}
\newcommand{\fS}{\field{S}}
\newcommand{\fN}{\field{N}}
\newcommand{\E}{\field{E}}
\renewcommand{\P}{\field{P}}
\newcommand{\RE}{{\text{\rm RE}}}
\newcommand{\LCB}{{\text{\rm LCB}}}
\newcommand{\Reg}{{\text{\rm Reg}}}
\newcommand{\Rel}{{\text{\rm Rel}}}

\renewcommand{\ss}{\subseteq}
\newcommand{\pred}{\yhat}

\newcommand{\normt}[1]{\norm{#1}_{t}}
\newcommand{\dualnormt}[1]{\norm{#1}_{t,*}}

\newcommand{\order}{\ensuremath{\mathcal{O}}}
\newcommand{\otil}{\ensuremath{\widetilde{\mathcal{O}}}}
\newcommand{\risk}{{\text{\rm Risk}}}
\newcommand{\iid}{{\text{\rm iid}}}
\newcommand{\seq}{{\text{\rm seq}}}
\newcommand{\iidV}{\calV^\iid}
\newcommand{\seqV}{\calV^\seq}
\newcommand{\poly}{{\text{\rm poly}}}
\newcommand{\sign}{{\text{\rm sign}}}
\newcommand{\ERM}{\pred_{{\text{\rm ERM}}}}
\newcommand{\iidRad}{\calR^{\iid}}
\newcommand{\iidCRad}{\wh{\calR}^{\iid}}
\newcommand{\VC}{{\text{\rm VCdim}}}
\newcommand{\vol}{{\text{\rm Vol}}}
\newcommand{\Holder}{{H{\"o}lder}\xspace}

\newcommand{\opt}{\mathring{q}}
\newcommand{\optpi}{\mathring{\pi}}

\newcommand{\minimax}[1]{\left\llangle #1 \right\rrangle}

\DeclareFontFamily{OMX}{MnSymbolE}{}
\DeclareFontShape{OMX}{MnSymbolE}{m}{n}{
	<-6>  MnSymbolE5
	<6-7>  MnSymbolE6
	<7-8>  MnSymbolE7
	<8-9>  MnSymbolE8
	<9-10> MnSymbolE9
	<10-12> MnSymbolE10
	<12->   MnSymbolE12}{}
\newcommand{\pref}[1]{\prettyref{#1}}
\newcommand{\pfref}[1]{Proof of \prettyref{#1}}
\newcommand{\savehyperref}[2]{\texorpdfstring{\hyperref[#1]{#2}}{#2}}

\newcommand{\err}{\textsc{Error}}
\newcommand{\martone}{\textsc{MDS}_1}
\newcommand{\marttwo}{\textsc{MDS}_2}

\newcommand{\regm}{\phi_{main}}
\newcommand{\regs}{\phi_{supp}}

\newpage


\begin{algorithm}[H]
	\caption{Delayed Hedge} \label{alg:hedge}
	\begin{algorithmic}
		\STATE \textbf{Input:} State space $\calS$, Action space $\calA$, Horizon $H$, Number of episodes $K$, Learning rate $\eta > 0$, Exploration parameter $\gamma > 0$, Confidence parameter $\delta > 0$.
		
		\STATE \textbf{Initialization:} Set $\probdist^1 (\pi) = \frac{1}{ \left\lvert \Omega \right\rvert }$ for every deterministic policy $\pi \in \Omega$; set $n^1_h(s,a) = 0 , n^1_h(s,a,s')$ for every $(s,a,s',h) \in \calS \times \calA \times \calS \times [H]$ and $\calP^1$ be the set of all transition functions. 
		
		\FOR{$k=1,2,...,K$}
		\STATE Play a randomly sampled policy from distribution $\probdist^k$ and observe trajectory $\{ (s^k_h,a^k_h) \}_{h=1}^H$.
		
		\STATE Compute upper occupancy bound $u_h^k(s,a) = \max_{p' \in \calP^k} \sum_{\pi\in \Omega} \probdist^k(\pi) q_h^{p',\pi}(s,a) $.
		
		

        \STATE Define confidence set $\calP^{k+1}$ by \cref{alg:update-cofidence-set}.
		
		\FOR{$j : j + d^j = k$}
		
		\STATE Observe feedback $\{ c^j_h(s^j_h,a^j_h) \}_{h=1}^H$.
		
		\STATE Compute loss estimator $\hatl^j_h(s,a) = \frac{c^j_h(s,a) \indevent{s^j_h = s , a^j_h = a}}{u^j_h(s,a) + \gamma} $ for every $(s,a,h) \in \calS \times \calA \times [H]$.
		
		\ENDFOR
		
		\STATE Update probability distribution over policy space: 
		\[
		\probdist^{k+1}(\pi) \propto \probdist^{k}(\pi) \cdot \exp \left( \eta \cdot b^k(\pi)  - \eta \sum_{j:j+d^j = k} \hatell^j(\pi)  \right), \forall \pi \in \calA^{\calS \times [H]}
		\]
		where $\hatell^j (\pi) = \sum_{h=1}^{H} \sum_{s,a} q_h^{\pi, \bar{p}^j}(s,a) \widehat{c}_h^j(s,a)$ denotes the loss suffered by policy $\pi$ with respect to the loss estimator $\widehat{c}^j$ and transition function $\bar{p}^j$, $b^{k}(\pi) = \max_{p' \in \calP^{k}} \norm{q^{\pi,\bar{p}^k } - q^{\pi,p'}}_1$ is the exploration bonus for policy $\pi$ at episode $k$.
		\ENDFOR
	\end{algorithmic}
\end{algorithm}
\begin{algorithm}[H]
	\caption{Update confidence set} \label{alg:update-cofidence-set}
	\begin{algorithmic}
		\STATE \textbf{Input:} trajectory $\{ (s^k_h,a^k_h) \}_{h=1}^H$.
		
		\STATE Update visit counters: $n^{k+1}_h(s^k_h,a^k_h) \gets n^k_h(s^k_h,a^k_h) + 1 , n^{k+1}_h(s^k_h,a^k_h,s^k_{h+1}) \gets n^k_h(s^k_h,a^k_h,s^k_{h+1}) + 1$ for every $h \in [H]$.
		
		\STATE Compute empirical transitions function $\bar{p}^{k+1}$: $ \bar p^{k+1}_h(s' \mid s,a) = \frac{n^{k+1}_h(s,a,s')}{n^{k+1}_h(s,a) \vee 1} \qquad \forall (s,a,s',h) $.
		
		\STATE Define confidence sets $\calP^{k+1}$ such that $p' \in \calP^{k+1}$ if and only if, for every $(s,a,s',h)$, $p'$ ensures $\sum_{s'} p'_h(s'|s,a)=1$ and:
		\[ 
		\left| p'_{h}(s'| s,a)-\bar{p}_{h}^{k+1}(s'| s,a) \right|
		\leq 
		\sqrt{ \frac{16 \bar{p}_{h}^{k+1}(s' | s,a)  \log\frac{10 HSAK}{\delta}}{n_{h}^{k+1}(s,a) \vee 1}} + \frac{10 \log\frac{10 HSAK}{\delta}}{n_{h}^{k+1}(s,a) \vee 1}.
		\]

	\end{algorithmic}
\end{algorithm}

\section{Delayed Hedge}
\label{appendix:delay-hedge}

In this section, we consider running Hedge over the policy space, that is, the set of all deterministic policies. We propose Algorithm~\ref{alg:hedge} with unknown transition and bandit feedback, which ensures $\otil\rbr{ \sqrt{K} + \sqrt{D}}$ regret as shown in Theorem~\ref{thm:regret-hedge} (ignoring dependence on other parameters).

\begin{theorem}\label{thm:regret-hedge} With $\eta = \gamma = \sqrt{\frac{  S \tcjconfpara }{ HD +  HSAK } }$, \cref{alg:hedge} ensures that 
	\[
	\regret
	=
	O\rbr{ H^2 S \sqrt{  A K  \tcjconfpara }+  H^{\nicefrac{3}{2}}\sqrt{SD \tcjconfpara} + H^3 S^3 A \tcjconfpara^3 +H^2 d_{max} \tcjconfpara } .
	\]
	with probability at least $1- 64\delta$ and the coefficient $\tcjconfpara = \log \frac{HSAK}{\delta}$. 
\end{theorem}

\subsection{Proof of the Main Theorem}

\begin{proof} [Proof of \cref{thm:regret-hedge}]  
We first decompose the regret decomposition as:
\begin{equation}
\begin{aligned}
\regret  = \sum_{k=1}^K \inner{\probdist^k - {\probdist}^\star, \ell^k} = \underbrace{\sum_{k=1}^K \left\langle \probdist^k , \ell^k - \hatell^k + b^k \right\rangle}_{\textsc{Est}}
+ \underbrace{\sum_{k=1}^K \left\langle \probdist^k - {\probdist}^\star, \hatell^k - b^k \right\rangle}_{\regthree}
+ \underbrace{ \sum_{k=1}^K \left\langle {\probdist}^\star, \hatell^k - b^k - \ell^k \right\rangle }_{\textsc{Bias}}.  
\label{paper-eq:hedge-regret-decomposition}
\end{aligned}
\end{equation}

	By combining \cref{lem:hedge-regret-est,lem:hedge-regret-bias,lem:hedge-regret-reg}, we arrive at the following bound of regret with learning rate $\eta$, exploration parameter $\gamma$ and confidence parameter $\delta$, with probability at least $1-64\delta$ that 
	\begin{equation}
	\begin{aligned}
	\regret & = \order\rbr{\frac{  H S \ln (A)  }{\eta} + \eta H^2 \rbr{ D +  H^2SAK} + \gamma HSAK + \rbr{ \frac{\eta}{\gamma} H^2 \rbr{ d_{max} + 1} + \frac{H}{\gamma} }\tcjconfpara }\\
	& \quad + \order\rbr{ H^2 S \sqrt{  A K \tcjconfpara }  +  H^3 S^3 A\ln K{\tcjconfpara}^2  }.  
	\end{aligned}
	\end{equation}
	Setting the learning rate and exploration parameter $\eta = \gamma = \sqrt{\frac{  S \ln (A) }{ HD +  HSAK } }$ , one can verify that the regret $\regret$ is bounded by  $\order\rbr{ H^2 S \sqrt{  A K \tcjconfpara }+  H^{\nicefrac{3}{2}}\sqrt{SD \tcjconfpara} + H^3 S^3 A \tcjconfpara^3 +H^2 d_{max} \tcjconfpara}$.
\end{proof}

Throughout the rest of this section, we will bound the three terms separately in \cref{lem:hedge-regret-bias,lem:hedge-regret-est,lem:hedge-regret-reg}.


\subsection{Bound on the Bias of the Cost Estimator ($\textsc{Bias}$ in \cref{paper-eq:hedge-regret-decomposition})}

\begin{lemma}[$\textsc{Bias}$] With probability at least $1-7\delta$, \cref{alg:hedge} ensures that $\textsc{Bias} = \order\rbr{ \frac{H\tcjconfpara}{\gamma} }$. 
	\label{lem:hedge-regret-bias}
\end{lemma}
\begin{proof}Similar to the analysis in \cite{jin2019learning}, we have $\textsc{Bias}$ bounded by  
	\begin{align}
	\sum_{k=1}^K \left\langle {\probdist}^\star, \hatell^k - b^k - \ell^k \right\rangle  
	& = \sum_{k=1}^K \inner{ q^{\pi^\star, \bar{p}^k}, \hatl^k } - \sum_{k=1}^{K} b^k(\pi^\star) - \sum_{k=1}^K \inner{ q^{\pi^\star, p}, c^k } \notag\\
	& \leq \order\rbr{ \frac{H}{\gamma} \log\rbr{ \frac{HSA}{\delta} } }  + \sum_{k=1}^K \inner{ q^{\pi^\star , \bar{p}^k} - q^{\pi^\star, p} , c^k } - b^k(\pi^\star) \notag \\
	& \leq \order\rbr{ \frac{H}{\gamma} \log\rbr{ \frac{HSA}{\delta} } }  + \sum_{k=1}^K \norm{q^{\pi^\star , \bar{p}^k} - q^{\pi^\star, p}}_1 - b^k(\pi^\star) \notag \\
	& = \order\rbr{ \frac{H}{\gamma} \log\rbr{ \frac{HSA}{\delta} } }, \label{eq:hedge_regret_add_1}
	\end{align}
	where the second step applies \cref{lem:jin2019_lemma_14} with probability at least $1-6\delta$; 
	the third step applies H\"older's inequality; 
	the last step follows from the event $p\in \cap_{k}\calP^{k}$ which holds with probability at least $1-\delta$, and the definition of exploration bonus $b^k(\pi)$. 
\end{proof}

\subsection{Bound on the Transition Estimation Error ($\textsc{Est}$ in \cref{paper-eq:hedge-regret-decomposition})}

\begin{lemma}[$\textsc{Est}$]\label{lem:hedge-regret-est} With probability at least $1-8\delta$, \cref{alg:hedge} ensures that 
	\begin{align*}
	\textsc{Est} = \order\rbr{\gamma HSAK + H^2S\sqrt{AK  \log\tcjconfpara} + S^3H^3 A \ln K \tcjconfpara}. 
	\end{align*}
\end{lemma}
\begin{proof}Observe that, $\sum_{k=1}^{K}\inner{\probdist^k , \ell^k - \hatell^k + b^k}$ can be upper bounded under the event that $p\in\cap_{k} \calP^k$ by 
	\begin{align}
	& \sum_{k=1}^{K}\sum_{\pi \in \Omega} \probdist^k(\pi) \rbr{\inner{ q^{\pi, p} , c^k } - \inner{ q^{\pi, \bar{p}^k} , \hatl^k } } + \sum_{k=1}^{K}\sum_{\pi \in \Omega} \probdist^k(\pi) b^k(\pi) \notag \\ 
	& = \sum_{k=1}^{K}\sum_{\pi \in \Omega} \probdist^k(\pi) \inner{ q^{\pi, p} - q^{\pi, \bar{p}^k}, c^k }  + \sum_{k=1}^{K}\sum_{\pi \in \Omega} \probdist^k(\pi) \inner{ q^{\pi, \bar{p}^k} , c^k - \hatl^k } + \sum_{k=1}^{K}\sum_{\pi \in \Omega} \probdist^k(\pi) b^k(\pi) \notag  \\
	& \leq \sum_{k=1}^{K} \inner{ q^k, c^k -\hatl^k } + 2 \sum_{k=1}^{K} \sum_{\pi \in \Omega} \probdist^k(\pi) b^k(\pi)  \label{eq:hedge-est-basic-decomp} 
	\end{align}	
	where $q^k = \sum_{\pi\in \Omega}\probdist^k(\pi) q^{\pi,\bar{p}^k}$ is the estimated occupancy measure at episode $k$, and 
	the second step follows from the definition of $b^k$ and Hölder's inequality.
	
	Note that, $\inner{ q^k, \hatl^k}$ is bounded by $H$ because $\bar{p}^{k} \in \calP^k$ and $u_h^k(s,a) \geq q_h^k(s,a)$ by its definition. 
	Thus, with the help of Azuma’s inequality, we have with probability at least $1-\delta$, 
	\begin{align*}
	\sum_{k=1}^{K} \inner{ q^k, \E^k\sbr{\hatl^k}  -\hatl^k } \leq\order\rbr{H\sqrt{ K \ln \rbr{\frac{1}{\delta} }} }.
	\end{align*}
	where $E^k[\cdot] = E[\cdot \mid \calH^k]$ and $\calH^k$ is the history of episodes $1,...,k-1$.
	We then focus on the term $\sum_{k=1}^{K} \inner{ q^k, c^k - \E^k\sbr{\hatl^k} }$ and rewrite it as 
	\begin{align}
	& \sum_{k=1}^{K} \sum_{h,s,a} q_h^k(s,a)  c_h^k(s,a) \rbr{ 1 -  \frac{ \E^k\sbr{\Ind{ s_h^k = s, a_h^k = a }} }{ u_h^k(s,a) + \gamma }}\notag \\
	& = \sum_{k=1}^{K} \sum_{h,s,a} q_h^k(s,a)  c_h^k(s,a) \rbr{ 1 -  \frac{ \whatq^k_h(s,a) }{ u_h^k(s,a) + \gamma }} \notag \\
	& = \sum_{k=1}^{K} \sum_{h,s,a} \frac{q_h^k(s,a)}{u_h^k(s,a) + \gamma}\rbr{  u_h^k(s,a) - \whatq^k_h(s,a) + \gamma  }  c_h^k(s,a) \notag \\
	& \leq \gamma HSAK + \sum_{k=1}^{K} \sum_{h,s,a} \abr{ u_h^k(s,a) - \whatq^k_h(s,a)   }  \label{eq:hedge-est-1} 
	\end{align}
	where $\whatq^k = \sum_{\pi\in \Omega}\probdist^k(\pi) q^{\pi,p}$ is the occupancy measure with the true transition $p$, and the last step comes from the fact that $u_h^k(s,a) \geq q_h^k(s,a)$ for all state-action pairs according to its definition. 
	
	Fixed the state-action pair $(s,a)$ and let $p' \in \calP^k$ be the transition function that yields $u_h^k(s,a)$ for simplicity. Then, we have the following inequality under the event $p \in \bigcap_k \calP^k$ that 
	\begin{align*}
	u_h^k(s,a) - \whatq^k_h(s,a) & = \sum_{\pi\in \Omega} \probdist^k(\pi) \rbr{ q^{ \pi, p' }_h(s,a) - q^{ \pi, p }_h(s,a) } \\
	& = \sum_{\pi\in \Omega} \probdist^k(\pi) \sum_{m=0}^{h-1} \sum_{x,y,z} q^{ \pi, p }_m(x,y) \cdot \rbr{ p_m(z|x,y) - p'_m(z|x,y) } \cdot q^{ \pi, p' }_{h|m+1}(s,a|z) \\
	\Rightarrow \abr{u_h^k(s,a) - \whatq^k_h(s,a)} & \leq \sum_{\pi\in \Omega} \probdist^k(\pi) \sum_{m=0}^{h-1} \sum_{s,a,s'} q^{ \pi, p }_m(x,y) \cdot \epsilon^k_m(z|x,y) \cdot q^{ \pi, p' }_{h|m+1}(s,a|z) 
	\end{align*}
	where the second step follows from \cite[Lemma D.3.1]{jin2021best} with the conditional occupancy measure $q^{ \pi, p' }_{h|m+1}(s,a|z)$ being the conditional probability of visiting state-action pair $(s,a)$ at step $h$ from state $z$ at state $m+1$ with policy $\pi$ and transition $p'$; the third step comes from taking the absolute value of both sides and the fact that 
	\begin{align}
	\epsilon_h^k(s'|s,a) \triangleq \order\rbr{\min\cbr{ 1, \sqrt{ \frac{ p_{h}^{k}(s' | s,a)  \tcjconfpara}{n_{h}^{k}(s,a) \vee 1}} + \frac{\tcjconfpara}{n_{h}^{k}(s,a) \vee 1}} } \geq \abr{ p'_h(s'|s,a) - p_h(s'|s,a)} \label{eq:hedge-est-epsilon}
	\end{align}
	for any transition tuple $(s,a,s')$ and step $h$ under the event $p \in \bigcap_{k}\calP^k$ according to \cite[Lemma D.3.3]{jin2021best}.
	In addition, we have $q^{ \pi, p' }_{h|m+1}(s,a|z) -  q^{\pi, p}_{h|m+1}(s,a|z)$ bounded by
	\begin{align*}
	&  \sum_{o=m+1}^{h-1} \sum_{u,v,w} q^{ \pi, p}_{o|m+1}(u,v|z) \cdot \rbr{ p^k_o(w|u,v) - p'_o(w|u,v) } \cdot q^{ \pi, p' }_{h|o+1}(s,a|w) \\
	& \leq \pi_h(a|s) \sum_{o=m+1}^{h-1} \sum_{u,v,w} q^{ \pi, p}_{o|m+1}(u,v|z) \cdot \abr{ p^k_o(w|u,v) - p'_o(w|u,v) } \\
	& \leq \pi_h(a|s) \sum_{o=m+1}^{h-1} \sum_{u,v} q^{ \pi, p}_{o|m+1}(u,v|z) \cdot \min\cbr{ 2, \sum_{w} \epsilon^k_o(w|u,v) }
	\end{align*}
	where the first step uses the fact that $q^{ \pi, p' }_{h|o+1}(s,a|w) \leq \pi_h(a|s)  \cdot q^{ \pi, p' }_{h|o+1}(s|w) =\pi_h(a|s)$; the second step follows from similar argument above; the last step uses the fact that $\sum_{w}\abr{ p^k_o(w|u,v) - p'_o(w|u,v) } \leq 2$. 
	
	Combining these inequalities, we have the second term of Equation~\eqref{eq:hedge-est-1}, $\sum_{k=1}^k\sum_{h,s,a}\abr{u_h^k(s,a) - \whatq^k_h(s,a)}$, bounded by
	\begin{equation}
	\begin{aligned}
	& \sum_{k=1}^K\sum_{\pi\in \Omega} \probdist^k(\pi) \sum_{h,s,a} \sum_{m=0}^{h-1} \sum_{x,y,z} q^{ \pi, p }_m(x,y)  \cdot \epsilon^k_m(z|x,y) \cdot q^{\pi, p}_{h|m+1}(s,a|z)  \\
	& +  \sum_{k=1}^K\sum_{\pi\in \Omega} \probdist^k(\pi) \sum_{h,s,a} \sum_{m=0}^{h-1} \sum_{x,y,z} \sum_{o=m+1}^{h-1} \sum_{u,v}  q^{ \pi, p }_m(x,y) \cdot \epsilon^k_m(z|x,y) \cdot q^{ \pi, p}_{o|m+1}(u,v|z) \cdot \min\cbr{ 2, \sum_{w} \epsilon^k_o(w|u,v) } \cdot \pi_h(a|s)  .
	\end{aligned}
	\label{eq:hedge-est-decomp-1}
	\end{equation}

	Note that, the first term of \cref{eq:hedge-est-decomp-1} can be bounded (under the event $p\in \bigcap_k \calP^k$) as 
	\begin{align}
	& \sum_{k=1}^K\sum_{\pi\in \Omega} \probdist^k(\pi) \sum_{h,s,a} \sum_{m=0}^{h-1} \sum_{x,y,z} q^{ \pi, p }_m(x,y)  \cdot \epsilon^k_m(z|x,y) \cdot q^{\pi, p}_{h|m+1}(s,a|z) \notag \\
	& = \sum_{k=1}^K\sum_{\pi\in \Omega} \probdist^k(\pi)  \sum_{m=0}^{H} \sum_{x,y,z} q^{ \pi, p }_m(x,y)  \cdot \epsilon^k_m(z|x,y) \cdot \rbr{ \sum_{h=m+1}^{H} \sum_{s,a} q^{\pi, p}_{h|m+1}(s,a|z)}\notag \\
	& \leq H \sum_{k=1}^Ks\sum_{\pi\in \Omega} \probdist^k(\pi)  \sum_{m=0}^{H} \sum_{x,y,z} q^{ \pi, p }_m(x,y)  \cdot \epsilon^k_m(z|x,y) \notag \\
	& = H  \sum_{k=1}^K\sum_{m=0}^{H} \sum_{x,y} \rbr{ \sum_{\pi\in \Omega} \probdist^k(\pi)   q^{ \pi, p }_m(x,y)}  \cdot \rbr{\sum_{z}\epsilon^k_m(z|x,y)} \notag \\
	& = H \sum_{k=1}^K\sum_{m=0}^{H} \sum_{x,y} \whatq^k_m(x,y) \cdot \rbr{\sum_{z}\epsilon^k_m(z|x,y)} \notag \\
	& = \order\rbr{  H \sum_{k=1}^K\sum_{m=0}^{H} \sum_{x,y} \whatq^{ \pi, p }_m(x,y) \rbr{\sqrt{ \frac{S\tcjconfpara }{n_{m}^{k}(x,y) \vee 1} } + \frac{S\tcjconfpara}{n_{m}^{k}(x,y) \vee 1} }   }  \notag  \\
	& = \order\rbr{ H^2S\sqrt{AK  \log\tcjconfpara  } } \label{eq:hedge-est-term-1}  
	\end{align}
	where the second steps follows from the fact that $\sum_{s,a} q^{\pi, p}_{h|m+1}(s,a|z) = 1$ for any policy $\pi$ and step $h\geq m+1$; the fourth step uses the definition of $\whatq^k$, the true occupancy measure at episode $k$; the fifth step uses the properties of $\epsilon^k$ under the event $p\in \bigcap_k \calP^k$; the final step applies \cref{lem:jin2019_lemma_10}, which yields a high probability bound with the help of a standard Bernstein-type concentration inequality for martingale. 
	
	Observing that $\sum_{h=o+1}^{H}\sum_{s,a}\pi_h(a|s)\leq SH$, we can reorder the summation and bound the second term of \cref{eq:hedge-est-decomp-1} by $SH$ multiplying
	\begin{align*}
	\sum_{k=1}^k\sum_{\pi\in \Omega} \probdist^k(\pi) \sum_{m=0}^{h-1} \sum_{x,y,z} \sum_{o=m+1}^{h-1} \sum_{u,v}  q^{ \pi, p }_m(x,y) \cdot \epsilon^k_m(z|x,y) \cdot q^{ \pi, p}_{o|m+1}(u,v|z) \cdot \min\cbr{ 2, \sum_{w} \epsilon^k_o(w|u,v) } .
	\end{align*}
	Similar to the proof in Appendix B.2 of \cite{jin2019learning}, we can further rewrite and bound the term above by 
	\begin{align*}
	& \order\rbr{\sum_{k=1}^K \sum_{\pi\in \Omega} \probdist^k(\pi) \sum_{m=0}^{H-1} \sum_{x,y,z} \sum_{o=m+1}^{H} \sum_{u,v,w}  q^{ \pi, p }_m(x,y) \cdot \sqrt{\frac{p_m(z|x,y)\tcjconfpara}{n_m^k(x,y)\vee 1}} \cdot q^{ \pi, p}_{o|m+1}(u,v|z) \cdot \sqrt{\frac{p_o(w|u,v)\tcjconfpara}{n_o^k(u,v)\vee 1}}   } \\
	& \ + \order\rbr{ \sum_{k=1}^K\sum_{\pi\in \Omega} \probdist^k(\pi) \sum_{m=0}^{H-1} \sum_{x,y,z}  \frac{q^{ \pi, p }_m(x,y) \tcjconfpara }{n_m^k(x,y)\vee 1}  \rbr{ \sum_{o=m+1}^{H} \sum_{u,v} q^{ \pi, p}_{o|m+1}(u,v|z) \min\cbr{\sum_{w} \epsilon_o^k(w|u,v),2 } }   }   \\
	&  + \order\rbr{\sum_{k=1}^K \sum_{\pi\in \Omega} \probdist^k(\pi) \sum_{o=0}^{H} \sum_{u,v,w}  \rbr{ \sum_{m=0}^{o-1} \sum_{x,y,z}  q^{ \pi, p }_m(x,y) \cdot p_m(z|x,y) \cdot q^{ \pi, p}_{o|m+1}(u,v|z) } \cdot \frac{\tcjconfpara}{n_o^k(u,v)\vee 1} }  
	\end{align*}
by using the property of $\epsilon_h^k$ as in \cref{eq:hedge-est-epsilon}  and the fact that $\sqrt{xy} \leq x + y$ for any $x,y>0$, therefore, $\epsilon^k_h(s'|s,a) \leq \order\rbr{ p_h(s'|s,a)  + \frac{\tcjconfpara}{n_h^k(s,a) \vee 1}  }$ holds for any $(s,a,s')$.
	
	Clearly, the later two are able to be reformulated and then bounded as 
	\begin{align}
	& \order\rbr{ \sum_{k=1}^K\sum_{\pi\in \Omega} \probdist^k(\pi) \sum_{m=0}^{H-1} \sum_{x,y,z}  \frac{q^{ \pi, p }_m(x,y) \tcjconfpara }{n_m^k(x,y)\vee 1}  \rbr{ \sum_{o=m+1}^{H} \sum_{u,v} q^{ \pi, p}_{o|m+1}(u,v|z) \min\cbr{\sum_{w} \epsilon_o^k(w|u,v),2 } }   }  \notag   \\
	& \quad + \order\rbr{ \sum_{k=1}^K\sum_{\pi\in \Omega} \probdist^k(\pi) \sum_{o=0}^{H} \sum_{u,v,w}  \rbr{ \sum_{m=0}^{o-1} \sum_{x,y,z}  q^{ \pi, p }_m(x,y) \cdot p_m(z|x,y) \cdot q^{ \pi, p}_{o|m+1}(u,v|z) } \cdot \frac{\tcjconfpara}{n_o^k(u,v)\vee 1} }\notag  \\
	& = \order\rbr{H \sum_{k=1}^K\sum_{\pi\in \Omega} \probdist^k(\pi) \sum_{m=0}^{H-1} \sum_{x,y,z}  \frac{q^{ \pi, p }_m(x,y) \tcjconfpara }{n_m^k(x,y)\vee 1} + H \sum_{k=1}^K\sum_{\pi\in \Omega} \probdist^k(\pi) \sum_{o=0}^{H} \sum_{u,v,w} \frac{q^{ \pi, p}_o(u,v) \tcjconfpara}{n_o^k(u,v)\vee 1} } \notag \\
	& = \order\rbr{SH\tcjconfpara^2 \sum_{k=1}^K\sum_{m=0}^{H-1} \sum_{x,y}  \frac{\whatq^k_m(x,y) }{n_m^k(x,y)\vee 1}  + SH\tcjconfpara^2 \sum_{k=1}^K\sum_{o=0}^{H} \sum_{u,v} \frac{\whatq^k_o(u,v) }{n_o^k(u,v)\vee 1} } \notag \\
	& = \order\rbr{ S^2HA \ln K \tcjconfpara^2  }  \label{eq:hedge-est-term-2} 
	\end{align}
	where the first step comes from the facts that $\sum_{o=m+1}^{H} \sum_{u,v} q^{ \pi, p}_{o|m+1}(u,v|z)\leq H$ for any $z$, and $\sum_{x,y,z}  q^{ \pi, p }_m(x,y) \cdot p_m(z|x,y) \cdot q^{ \pi, p}_{o|m+1}(u,v|z) = q^{ \pi, p}(u,v)$ for any $(u,v)$ according to the definitions of conditional occupancy measures; the second step follows from the definition of $\whatq^k$; the last step applies \cref{lem:jin2019_lemma_10} with probability at least $1-2\delta$. 
	
	On the other hand, the first term can be written as $SH\tcjconfpara$ multiplied by the following (ignoring some constants): 
	\begin{align}
	& \sum_{k=1}^{K} \sum_{\pi\in \Omega} \probdist^k(\pi) \sum_{m=0}^{H-1} \sum_{x,y,z} \sum_{o=m+1}^{H} \sum_{u,v,w}  q^{ \pi, p }_m(x,y) \cdot \sqrt{\frac{p_m(z|x,y)}{n_m^k(x,y) \vee 1}} \cdot q^{ \pi, p}_{o|m+1}(u,v|z) \cdot \sqrt{\frac{p_o(w|u,v)}{n_o^k(u,v) \vee 1}} \notag \\
	& = \sum_{k=1}^{K} \sum_{\pi\in \Omega} \probdist^k(\pi) \sum_{m=0}^{H-1} \sum_{x,y,z} \sum_{o=m+1}^{H} \sum_{u,v,w}   \sqrt{ \frac{q^{ \pi, p }_m(x,y) p_m(z|x,y) q^{ \pi, p}_{o|m+1}(u,v|z)}{n_m^k(x,y) \vee 1}} \cdot   \sqrt{\frac{q^{ \pi, p }_m(x,y)  p_o(w|u,v) q^{ \pi, p}_{o|m+1}(u,v|z) }{n_o^k(u,v) \vee 1}} \notag \\
	& = \sum_{k=1}^{K} \sum_{\pi\in \Omega} \sum_{m=0}^{H-1} \sum_{x,y,z} \sum_{o=m+1}^{H} \sum_{u,v,w}   \sqrt{ \frac{\probdist^k(\pi) q^{ \pi, p }_m(x,y) p_m(z|x,y) q^{ \pi, p}_{o|m+1}(u,v|z)}{n_o^k(u,v) \vee 1} } \cdot   \sqrt{ \frac{\probdist^k(\pi) q^{\pi, p }_m(x,y)  p_o(w|u,v) q^{ \pi, p}_{o|m+1}(u,v|z) }{n_m^k(x,y) \vee 1} }\notag \\
	& \leq \sum_{m=0}^{H-1} \sum_{o=m+1}^{H} \sqrt{ \sum_{k=1}^{K} \sum_{\pi\in \Omega}  \sum_{x,y,z}\sum_{u,v,w} \frac{\probdist^k(\pi) q^{ \pi, p }_m(x,y) p_m(z|x,y) q^{ \pi, p}_{o|m+1}(u,v|z)}{n_o^k(u,v) \vee 1}   }\notag \\
	& \quad \cdot \sqrt{ \sum_{k=1}^{K} \sum_{\pi\in \Omega} \sum_{x,y,z}\sum_{u,v,w} \frac{\probdist^k(\pi) q^{\pi, p }_m(x,y)  p_o(w|u,v) q^{ \pi, p}_{o|m+1}(u,v|z) }{n_m^k(x,y) \vee 1}  } \notag\\
	& \leq \sum_{m=0}^{H-1} \sum_{o=m+1}^{H} \sqrt{ \sum_{k=1}^{K} \sum_{\pi\in \Omega}  \sum_{u,v,w} \frac{\probdist^k(\pi) q^{ \pi, p }_o(u,v)}{n_o^k(u,v) \vee 1}   }  \cdot \sqrt{ \sum_{k=1}^{K} \sum_{\pi\in \Omega} \sum_{x,y,z}\frac{\probdist^k(\pi) q^{\pi, p }_m(x,y) }{n_m^k(x,y) \vee 1}  } \notag\\
	& = S \sum_{m=0}^{H-1} \sum_{o=m+1}^{H} \sqrt{ \sum_{k=1}^{K} \sum_{u,v} \frac{\whatq_o^k(u,v)}{n_o^k(u,v) \vee 1}   }  \cdot \sqrt{ \sum_{k=1}^{K} \sum_{x,y,z}\frac{\whatq_m(x,y) }{n_m^k(x,y) \vee 1}  } \notag \\
	& = \order\rbr{ S^2 H^2 A \ln K } \label{eq:hedge-est-term-3} 
	\end{align}
	where the third step uses Cauchy-Schwarz inequality; the fourth step follows from the properties of conditional occupancy measure $\sum_{x,y,z} q^{ \pi, p }_m(x,y) p_m(z|x,y) q^{ \pi, p}_{o|m+1}(u,v|z) = q^{ \pi, p }_m(u,v)$; the last step applies \cref{lem:jin2019_lemma_10} with probability at least $1-2\delta$. 
	
	Combining \cref{eq:hedge-est-decomp-1,eq:hedge-est-term-1,eq:hedge-est-term-2,eq:hedge-est-term-3} into \cref{eq:hedge-est-1}, we have the following inequality holds with probability at least $1-4\delta$ under the event $p\in\bigcap_{k}\calP^k$ that 
	\begin{align}
	\sum_{k=1}^{K} \inner{ q^k, c^k -\hatl^k } = \order\rbr{ \gamma HSAK + H^2S\sqrt{AK  \log\tcjconfpara} + S^3H^3 A \ln K \tcjconfpara  }.\label{eq:hedge_regret_add_4}
	\end{align} 
	
	With slightly abuse of notations, we use $\bar{p}^k(\pi)$ to denote the transition function that yields $b^k(\pi)$ associated with $\pi$ and confidence set $\calP^k$, that is, $\bar{p}^k(\pi) = \argmax_{p' \in \calP^k} \norm{ q^{\pi,p'} - q^{\pi,\bar{p}^k} }_1$ . Thus, for $\sum_{k=1}^{K}\inner{\probdist^k, b^k}$, we have the following inequality holds with probability at least $1-2\delta$ that 
	\begin{align}
	\sum_{k=1}^{K}\inner{\probdist^k, b^k} & = \sum_{k=1}^{K} \sum_{\pi \in \Omega} \probdist^k(\pi) \norm{ q^{\pi, \bar{p}^k(\pi)} - q^{\pi, \bar{p}^k}  }_1  \notag \\
	& \leq \sum_{k=1}^{K} \sum_{\pi \in \Omega}  \probdist^k(\pi) \rbr{\norm{ q^{\pi, \bar{p}^k(\pi)} - q^{\pi, p}  }_1 + \norm{ q^{\pi, p} - q^{\pi, \bar{p}^k}  }_1  } \notag \\
	& \leq H \sum_{k=1}^{K} \sum_{\pi \in \Omega}  \probdist^k(\pi) \sum_{h=1}^{H} q^{\pi,p}_h(s,a) \cdot \rbr{ \norm{\bar{p}^k_h(\cdot|s,a) -  p_h(\cdot|s,a)  }_1 + \norm{\bar{p}_h^k(\pi)(\cdot|s,a) -  p_h(\cdot|s,a) }_1 }\notag \\
	& \leq H \sum_{k=1}^{K} \sum_{\pi \in \Omega}  \probdist^k(\pi) \sum_{h=1}^{H} q^{\pi,p}_h(s,a) \cdot \rbr{ \sum_{s'} \epsilon_h^k(s'|s,a) } \notag \\
	& \leq \order\rbr{ H \sum_{k=1}^{K} \sum_{h=1}^{H}\whatq^k_h(s,a) \sqrt{\frac{S \tcjconfpara }{n^k_h(s,a) \vee 1}} }  \notag \\
	& \leq \order\rbr{ H^2 S \sqrt{  A K \tcjconfpara } } \label{eq:hedge_regret_add_3}
	\end{align}
	where the second step follows from the triangle inequality for $\ell_1$ norms; the third step comes from  Lemma B.1 and B.2 of \cite{rosenberg2019online}; the forth step uses the property of $\epsilon^k$ defined in \cref{eq:hedge-est-epsilon}; the fifth step follows from the fact that $\sum_{\pi \in \Omega} \probdist^k(\pi)q^{\pi,p} =\whatq^k$; the final step follows from the same argument as in \cref{eq:hedge-est-term-1}. 
	
	Combining \cref{eq:hedge_regret_add_4,eq:hedge_regret_add_3} into \cref{eq:hedge-est-basic-decomp} concludes the proof. 
\end{proof}

\subsection{Bound on the Regret with respect to the Loss Estimators ($\regthree$ in \cref{paper-eq:hedge-regret-decomposition})}
\begin{lemma}[$\regthree$]\label{lem:hedge-regret-reg} With probability at least $1-32\delta$, \cref{alg:hedge} ensures that 
\begin{align*}
\regthree  = \order\rbr{\frac{  H S \ln (A)  }{\eta} + \eta H^2\rbr{ SAK+D}+ \frac{\eta}{\gamma} \cdot H^2 \rbr{ d_{max} +1 } \tcjconfpara  }. 
\end{align*}
\end{lemma}

\begin{proof}
Let $\{\widetilde{\probdist}^{k+1}\}_{k=1}^{K}$ be the sequence of probability distributions with both received and un-received loss estimators prior to episode $k+1$, that is, 
\begin{equation}
\widetilde{\probdist}^{k+1}(\pi) \propto \probdist^1(\pi) \cdot \exp\rbr{  - \eta \rbr{\sum_{j=1}^{k}\hatell^j(\pi) - \sum_{j=1}^{k} b^j(\pi) } }, \forall \pi  \in \calA^{\calS \times [H]}. \notag
\end{equation}
On the other hand, according to the fact that $b^j(\pi') \leq 2H$, we add a constant $2H$ uniformly to the loss vector $\hatell^k - b^k$ and construct $m^k(\pi) = \hatell^k(\pi)- b^k(\pi) + 2H$ to ensure the positiveness for any $\pi$ . Clearly, adding the constant uniformly will not change the outcomes of our algorithm. 

With the help of these notations, we are able to decompose $\regthree$ into two parts as: 
\begin{align} 
\regthree & = \underbrace{ \sum_{k=1}^K \inner{  \widetilde{\probdist}^{k+1} - {\probdist}^\star, \hatell^k - b^k  }  }_{\textsc{Cheating Regret}} +  \underbrace{ \sum_{k=1}^K \inner{ \probdist^k - \widetilde{\probdist}^{k+1}, \hatell^k - b^k }  }_{\textsc{Drift}}  \notag 
\end{align}
where $\textsc{Cheating-Regret}$ is bounded in \cite{gyorgy2020adapting} that 
\begin{equation}
\sum_{k=1}^K \inner{  \widetilde{\probdist}^{k+1} - {\probdist}^\star, m^k }   \leq \frac{\ln \left\lvert\Omega \right\rvert }{\eta} = \frac{\ln \left\lvert\calA^{\calS \times [H]} \right\rvert }{\eta} = \frac{  H S \ln (A)  }{\eta}. \label{eq:hedge_regret_1}
\end{equation}

For $\textsc{Drift}$, we first rewrite it as 
\begin{align}
  \sum_{k=1}^K \inner{ \probdist^k - \widetilde{\probdist}^{k+1}, \hatell^k - b^k } & = \sum_{k=1}^K \inner{ \probdist^k - \widetilde{\probdist}^{k+1}, 2H + \hatell^k - b^k } \notag \\
 & = \sum_{k=1}^K \sum_{\pi \in \Omega} \probdist^k(\pi) \rbr{ 2H + \hatell^k(\pi) - b^k(\pi)  } \cdot \rbr{ 1 - \frac{\widetilde{\probdist}^{k+1}(\pi)}{\probdist^k(\pi)}} \notag \\
 & = \sum_{k=1}^K \sum_{\pi \in \Omega} \probdist^k(\pi) m^k(\pi) \cdot \rbr{ 1 - \frac{\widetilde{\probdist}^{k+1}(\pi)}{\probdist^k(\pi)}} \label{eq:hedge_regret_2_rewrite}
\end{align}
where the second step follows from the fact that $\sum_{\pi\in \Omega} \widetilde{\probdist}^{k+1}(\pi) = \sum_{\pi\in \Omega} \probdist^{k}(\pi) = 1$. Then, we consider the ratio between $\probdist^k(\pi)$ and $\widetilde{\probdist}^{k+1}(\pi)$: 
\begin{align}
& \frac{ \widetilde{\probdist}^{k+1}(\pi) }{ \probdist^k(\pi) }  = \frac{ \exp\rbr{ - \eta \sum_{j=1}^{k} \rbr{\hatell^j(\pi) - b^j(\pi)}  } }{
	\sum_{\pi' \in \Omega} \exp\rbr{ - \eta \sum_{j=1}^{k}\rbr{\hatell^j(\pi')- b^j(\pi')}   }  } \cdot \frac{ \sum_{\pi'\in \Omega} \exp\rbr{ - \eta \sum_{j:j + d^j < k }\hatell^j(\pi') + \eta \sum_{j=1}^{k-1}  b^j(\pi')  } }{\exp\rbr{ - \eta \sum_{j:j + d^j < k }\hatell^j(\pi) + \eta \sum_{j=1}^{k-1}  b^j(\pi ) }  }  \notag \\
& = \frac{ \exp\rbr{ - \eta \sum_{j=1}^{k} \rbr{\hatell^j(\pi) - b^j(\pi)} - \eta 2H  } }{
	\sum_{\pi' \in \Omega} \exp\rbr{ - \eta \sum_{j=1}^{k}\rbr{\hatell^j(\pi')- b^j(\pi')} - \eta 2H  }   } \cdot \frac{ \sum_{\pi'\in \Omega} \exp\rbr{ - \eta \sum_{j:j + d^j < k }\hatell^j(\pi') + \eta \sum_{j=1}^{k-1}  b^j(\pi')  } }{\exp\rbr{ - \eta \sum_{j:j + d^j < k }\hatell^j(\pi) + \eta \sum_{j=1}^{k-1}  b^j(\pi ) }  }  \notag \\
& = \frac{ \sum_{\pi'\in \Omega} \exp\rbr{ - \eta \sum_{j:j + d^j < k }\hatell^j(\pi') + \eta \sum_{j=1}^{k-1}  b^j(\pi')  } }{
	\sum_{\pi' \in \Omega} \exp\rbr{ - \eta \sum_{j=1}^{k}\rbr{\hatell^j(\pi')- b^j(\pi')} - \eta 2H  }   } \cdot \frac{ \exp\rbr{ - \eta \sum_{j=1}^{k} \rbr{\hatell^j(\pi) - b^j(\pi)} - \eta 2H  }  }{\exp\rbr{ - \eta \sum_{j:j + d^j < k }\hatell^j(\pi) + \eta \sum_{j=1}^{k-1}  b^j(\pi ) }  } 
	\label{eq:hedge-bounding ratio}
\end{align}
where the second step follows from multiplying denominator and nominator together by $\exp(-\eta 2H)$. Note that $b^k(\pi)\leq 2H$ and $\hatell^k(\pi) \geq 0$ for any $\pi$ and $k$, we thus have the following inequality holds that 
$$ \sum_{j=1}^{k}\rbr{\hatell^j(\pi')- b^j(\pi')} + 2H =\sum_{j=1}^{k}\hatell^j(\pi') - \sum_{j=1}^{k-1}b^j(\pi') + 2H -  b^k(\pi')  \geq  \sum_{j=1,j + d^j < k }^{k-1}\hatell^j(\pi') - \sum_{j=1}^{k-1}  b^j(\pi')$$
which indicates that the first fraction is lower bounded by $1$. 

Therefore, the ratio $\widetilde{\probdist}^{k+1}(\pi) / \probdist^k(\pi)$ for any policy $\pi \in \Omega$ can be further  bounded by 
\begin{align}
\widetilde{\probdist}^{k+1}(\pi) / \probdist^k(\pi) &  \geq \exp\rbr{ - \eta \rbr{\hatell^k(\pi) +  \sum_{j=1, j + d^j \geq k }^{k-1}\hatell^k(\pi) + \rbr{2H - b^k(\pi)}  } }  \notag  \\
& \geq 1 - \eta \rbr{m^k(\pi) +  \sum_{j=1, j + d^j \geq k }^{k-1}\hatell^k(\pi) }, \notag 
\end{align}
where the last step uses $1+x\le e^x$ for any $x \in \bbR$. 

Plugging this inequality back to \cref{eq:hedge_regret_2_rewrite}, we have $\textsc{Drift}$ bounded and then decomposed into two parts as 
\begin{align}
\sum_{k=1}^K \sum_{\pi \in \Omega }  \probdist^k(\pi) m^k(\pi)  \rbr{ 1  - \frac{\widetilde{\probdist}^{k+1}(\pi) }{\probdist^k(\pi) } } 
\notag & \leq \eta \sum_{k=1}^K \sum_{\pi \in \Omega }  \probdist^k(\pi) m^k(\pi)  \rbr{ m^k(\pi) +  \sum_{j=1, j + d^j \geq k }^{k-1}\hatell^k(\pi)   }       \\
 = \eta \sum_{k=1}^K \sum_{\pi \in \Omega }  & \probdist^k(\pi) m^k(\pi)^2  +   \eta \sum_{k=1}^K \sum_{\pi \in \Omega }  \probdist^k(\pi) m^k(\pi) \sum_{j=1, j + d^j \geq k }^{k-1}\hatell^j(\pi)      \label{eq:hedge-regret-mid-decomp} 
\end{align}
where the first part associates with the regret incurred without the delayed feedback and can be controlled by standard arguments as: 
\begin{align}
\eta \sum_{k=1}^K \sum_{\pi \in \Omega }   \probdist^k(\pi) m^k(\pi)^2 
& = \eta \sum_{k=1}^K \sum_{\pi \in \Omega }   \probdist^k(\pi) \rbr{  \sum_{h=1}^{H} \sum_{s,a} q_h^{\pi, \bar{p}^k}(s,a) \widehat{c}_h^k(s,a) + 2H - b^k(\pi)  }^2     \notag  \\
& \leq 2 \eta \sum_{k=1}^K \sum_{\pi \in \Omega }   \probdist^k(\pi) \sbr{ \rbr{  \sum_{h=1}^{H} \sum_{s,a} q_h^{\pi, \bar{p}^k}(s,a) \widehat{c}_h^k(s,a)}^2 + 4H^2 } \notag \\
& \leq 8\eta H^2K + 2\eta H \sum_{k=1}^K \sum_{\pi \in \Omega }  \probdist^k(\pi)  \sum_{h=1}^{H} \rbr{\sum_{s,a}  q_h^{\pi, \bar{p}^k}(s,a) \cdot \widehat{c}_h^k(s,a)}^2   \notag  \\ 
& =  8\eta H^2K + 2\eta H \sum_{k=1}^K  \sum_{h=1}^{H} \sum_{s,a} \sum_{\pi \in \Omega }   \probdist^k(\pi)  q_h^{\pi, \bar{p}^k}(s,a)^2 \cdot  \widehat{c}_h^k(s,a)^2   \notag \\
& \leq  8\eta H^2K + 2\eta H \sum_{k=1}^K  \sum_{h=1}^{H} \sum_{s,a}  \widehat{c}_h^k(s,a)^2 \rbr{ \sum_{\pi \in \Omega }   \probdist^k(\pi)   q_h^{\pi, \bar{p}^k}(s,a)^2}     \notag \\
& \leq 8\eta H^2K + 2\eta H \sum_{k=1}^K  \sum_{h=1}^{H} \sum_{s,a}  \widehat{c}_h^k(s,a)  \notag 
\end{align}
where the second step follows from the fact that $(x+y)^2\leq x^2 + y^2$; the third step uses Cauchy-Schwartz inequality; the forth step follows from the fact $\mathbb{I}\cbr{s_h^k = s, a_h^k = a } \mathbb{I}\cbr{s_h^k = s', a_h^k = a' } = 0 $ for all $(s,a),(s',a') \in \calS \times \calA$ such that $(s,a) \neq (s',a')$; the final step uses the fact that $u_h^k(s,a) \geq \sum_{\pi \in \Omega} \probdist^k(\pi)   q_h^{\pi, \bar{p}^k}(s,a)$ and the definition of loss estimator $\hatl^k$. 

Moreover, with \cref{lem:jin2019_lemma_11}, we can show that the following inequality hold with probability at least $1-9\delta$ that 
\begin{align}
8\eta H^2K + 2\eta H \sum_{k=1}^K  \sum_{h=1}^{H} \sum_{s,a}  \widehat{c}_h^k(s,a)  = \order\rbr{ \eta H^2SAK + \frac{\eta H^2}{\gamma}\tcjconfpara }. 	\label{eq:hedge_regret_2}
\end{align}

Similarly, for some part of the second term of Equation~\eqref{eq:hedge-regret-mid-decomp}, we have 
\begin{align}
\eta \sum_{k=1}^K \sum_{\pi \in \Omega }  \probdist^k(\pi) & \rbr{2H - b^k(\pi) } \sum_{j=1, j + d^j \geq k }^{k-1}\hatl^k(\pi)
\leq 2 \eta H \sum_{k=1}^K \sum_{j=1, j + d^j \geq k }^{k-1} \sum_{\pi \in \Omega }  \probdist^k(\pi) \rbr{ \sum_{h=1}^{H}\sum_{s,a}  q_h^{\pi, \bar{p}^j}(s,a) \hatl_h^j(s,a)  }      \notag \\
& =  2 \eta H \sum_{k=1}^K \sum_{j=1, j + d^j \geq k }^{k-1} \sum_{\pi \in \Omega } \sum_{h=1}^{H}\sum_{s,a}  \probdist^k(\pi)  q_h^{\pi, \bar{p}^j}(s,a) \hatl_h^j(s,a)    \notag \\
& \leq  \order\rbr{\frac{\eta}{\gamma} H^2 d_{max}\tcjconfpara   }+ 2 \eta H \sum_{k=1}^K \sum_{j=1, j + d^j \geq k }^{k-1} \sum_{\pi \in \Omega } \sum_{h=1}^{H}\sum_{s,a}  \probdist^k(\pi)  q_h^{\pi, \bar{p}^j}(s,a) \notag \\
& =  \order\rbr{\frac{\eta}{\gamma} H^2 d_{max}\tcjconfpara  + \eta H^2 D }  \label{eq:hedge_regret_3} 
\end{align}
where the third step uses \cref{lem:jin2019_lemma_11} under the event that $p\in \cap_{k} \calP^k$, which holds with probability at least $1-9\delta$. 

On the other hand, the rest of the second part can be be bounded with respect to the conditional independence between loss estimators $\widehat{c}^k$ and $\widehat{c}_j$ for any $j < k$ satisfying $j + d^j\ge k$: 
\begin{align}
& \eta \sum_{k=1}^K \sum_{\pi \in \Omega }  \probdist^k(\pi) \hatl^k(\pi) \sum_{j=1, j + d^j \geq k }^{k-1}\hatl^k(\pi)    \notag \\
& \leq \eta \sum_{k=1}^K \sum_{j=1, j + d^j \geq k }^{k-1} \sum_{\pi \in \Omega }  \probdist^k(\pi) \rbr{ \sum_{h=1}^{H}\sum_{s,a}q_h^{\pi, \bar{p}^k}(s,a) \hatl_h^k(s,a)  } \rbr{ \sum_{h=1}^{H}\sum_{s,a}  q_h^{\pi, \bar{p}^j}(s,a) \hatl_h^j(s,a)  }      \notag \\
& =  \eta \sum_{k=1}^K \sum_{j=1, j + d^j \geq k }^{k-1}\sum_{h=1}^{H}\sum_{s,a}  \sum_{h'=1}^{H}\sum_{s',a'}  \hatl_h^k(s,a) \hatl_{h'}^j(s',a') \rbr{ \sum_{\pi\in \Omega} \probdist^k(\pi)\cdot q_h^{\pi, \bar{p}^k}(s,a) q^{\pi, \bar{p}^j}_{h'}(s',a') }    \notag 
\end{align}
where the first step uses the definition of loss estimators. Similarly, we have the following inequality holds with probability at least $1-12\delta$ that 
\begin{align}
& \eta \sum_{k=1}^K \sum_{j=1, j + d^j \geq k }^{k-1}\sum_{h=1}^{H}\sum_{s,a}  \sum_{h'=1}^{H}\sum_{s',a'}  \hatl_h^k(s,a) \hatl_{h'}^j(s',a') \rbr{ \sum_{\pi\in \Omega} \probdist^k(\pi)\cdot q_h^{\pi, \bar{p}^k}(s,a) q^{\pi, \bar{p}^j}_{h'}(s',a') } \notag \\
& \leq \eta \sum_{k=1}^K \sum_{j=1, j + d^j \geq k }^{k-1}\sum_{h=1}^{H}\sum_{s,a}  \sum_{h'=1}^{H}\sum_{s',a'}  \hatl_h^k(s,a) \rbr{ \sum_{\pi\in \Omega} \probdist^k(\pi)\cdot q_h^{\pi, \bar{p}^k}(s,a) q^{\pi, \bar{p}^j}_{h'}(s',a') } + \order\rbr{\frac{\eta}{\gamma} H^2 d_{max}\tcjconfpara   }    \notag \\
& \leq \eta \sum_{k=1}^K \sum_{j=1, j + d^j \geq k }^{k-1}\sum_{h=1}^{H}\sum_{s,a}  \sum_{h'=1}^{H}\sum_{s',a'} \sum_{\pi\in \Omega} \probdist^k(\pi)\cdot q_h^{\pi, \bar{p}^k}(s,a) q^{\pi, \bar{p}^j}_{h'}(s',a')  + \order\rbr{\frac{\eta}{\gamma} H^2 d_{max} \tcjconfpara  }    \notag \\
& =   \eta \sum_{k=1}^K \sum_{j=1, j + d^j \geq k }^{k-1} \sum_{\pi\in \Omega} \probdist^k(\pi) \rbr{ \sum_{h=1}^{H}\sum_{s,a}  q_h^{\pi, \bar{p}^k}(s,a) } \rbr{ \sum_{h=1}^{H}\sum_{s,a}  q_h^{\pi, \bar{p}^j}(s,a) } + \order\rbr{\frac{\eta}{\gamma} H^2 d_{max} \tcjconfpara  }   \notag \\
& =  \order\rbr{\frac{\eta}{\gamma} H^2 d_{max}  \tcjconfpara }+ \eta H^2 \sum_{k=1}^K \sum_{j=1, j + d^j \geq k }^{k-1} 1   =   \order\rbr{\frac{\eta}{\gamma} H^2 d_{max} \tcjconfpara  } +\eta H^2 \sum_{j=1}^K \sum_{k= 1, k>j, k\leq j + d^j }^{K} 1   \notag \\
& =  \order\rbr{\eta H^2 D  + \frac{\eta}{\gamma} H^2 d_{max} \tcjconfpara  }  \label{eq:hedge_regret_4}
\end{align}
where the first and second step apply \cref{lem:jin2019_lemma_11} twice under the event that $p\in \cap_{k} \calP^k$, based on the fact that $q^{\pi, \bar{p}^j}_{h'}(s',a')  \leq 1$ and $\sum_{\pi\in \Omega} \probdist^k(\pi)\cdot q_h^{\pi, \bar{p}^k}(s,a) q^{\pi, \bar{p}^j}_{h'}(s',a')  \leq \sum_{\pi\in \Omega} \probdist^k(\pi)\cdot q_h^{\pi, \bar{p}^k}(s,a) \leq u^k_h(s,a)$.

Combining \cref{eq:hedge_regret_2,eq:hedge_regret_3,eq:hedge_regret_4} yields the following bound of $\textsc{Drift}$ with probability at least $1-30\delta$ under the event $p \in \cap_{k=1}^K \calP^k$: 
\begin{equation}
\label{eq:hedge_regret_5}
\textsc{Drift} = \order\rbr{ \eta H^2 \rbr{ D +  H^2SAK} +  \frac{\eta}{\gamma} H^2 \rbr{ d_{max} + 1} \tcjconfpara  }.
\end{equation}

Finally, combining the bounds for $\textsc{Cheating-Regret}$ and $\textsc{Drift}$ in \cref{eq:hedge_regret_1,eq:hedge_regret_5} concludes the proof. 
\end{proof}

\subsection{Supplementary Lemmas}

In this section, we list the supplementary lemmas which directly attained from the previous work \cite{jin2019learning}.

\begin{lemma}[Lemma 4 of \cite{jin2019learning}]\label{lem:jin2019_lemma_4}
With probability at least $1-6\delta$, for any collection of transition functions $\cbr{p_k^{s_h}}_{s\in \calS,h\in[H]}$ such that $p_k^{s_h}$ belongs to the confidence set $\calP^k$ defined by \cref{alg:update-cofidence-set} for all every $(s,h) \in \calS \times [H]$, we have 
\begin{align*}
    \sum_{k=1}^{K} \sum_{s,h} \left\lvert q^{ \pi^k, p_k^{s_h}}(s_h) - q^{ \pi^k,p }(s_h) \right\rvert = \order\rbr{ H^2 S\sqrt{AK\log\rbr{ \frac{HSAK}{\delta}  } } + H^3 S^3 A \log^3\rbr{ \frac{HSAK}{\delta}} }.
\end{align*}
\end{lemma}

\begin{lemma}[Lemma 10 of \cite{jin2019learning}]\label{lem:jin2019_lemma_10}
With probability at least $1-2\delta$, we have for all $h\in[H]$,
\begin{align*}
    \sum_{k=1}^{K} \sum_{s\in \calS,a \in \calA} \frac{q^{\pi^k,p}_h(s,a)}{\sqrt{n_h^k(s,a) \vee 1 } } = \order\rbr{ \sqrt{SAK} + SA\log K +   \log\rbr{ \frac{H}{\delta} } },
\end{align*}
and 
\begin{align*}
    \sum_{k=1}^{K} \sum_{s\in \calS,a \in \calA} \frac{q^{\pi^k,p}_h(s,a)}{n_h^k(s,a) \vee 1 } = \order\rbr{ SA\log K +   \log\rbr{ \frac{H}{\delta} } }
\end{align*}
where $p$ here is the true transition function, and $q^{\pi^k,p}_h(s,a)$ denotes the probability of visiting state-action pair $(s,a)$ at step $h$ via the policy $\pi^k$ for episode $k$. 
\end{lemma}

\begin{lemma}[Lemma 11 of \cite{jin2019learning}]\label{lem:jin2019_lemma_11}
For any sequence of functions $\alpha_1, \alpha_2,\ldots \alpha_K$ such that $\alpha_k \in [0,2\gamma]^{\calS \times \calA}$ is $\calF_k$-measurable for all $k$, with probability at least $1-\delta$ we have for every $h\in [H]$ that 
\begin{align*}
\sum_{k=1}^{K} \sum_{s,a} \alpha_k(s,a) \rbr{ \widehat{c}^k_h(s,a) - \frac{q^{\pi^k,p}_h(s,a)}{u^k_h(s,a)} \cdot c^k_h(s,a)} \leq \order\rbr{ \log\frac{H}{\delta}}  
\end{align*}
where $q^{\pi^k,p}_h(s,a)$ is the true probability of visiting state-action pair $(s,a)$ at step $h$ in episode $k$, and $u_h^k(s,a)$ defined in \cref{alg:hedge} is the upper occupancy bound of this probability. 
\end{lemma}

\begin{lemma}[Lemma 14 of \cite{jin2019learning}]\label{lem:jin2019_lemma_14}
For any policy $\pi^\star$, with probability at least $1-6\delta$,  \cref{alg:hedge} ensures that 
\begin{align*}
    \sum_{k=1}^{K} \inner{q^{\pi^\star,p} , \widehat{c}^k - c^k } = \order\rbr{ \frac{H}{\gamma} \log \rbr{ \frac{HSA}{\delta}  }  }.
\end{align*}
\end{lemma}

\newpage


\section{FTRL with normal loss estimator}
\label{appendix:delay-ftrl}

\begin{algorithm}[t]
	\caption{Delayed UOB-FTRL with Normal Loss Estimator} \label{alg:ftrl-normal-loss-estimator}
	\begin{algorithmic}
		\STATE \textbf{Input:} State space $\calS$, Action space $\calA$, Horizon $H$, Number of episodes $K$, Learning rate $\eta > 0$, Exploration parameter $\gamma > 0$, Confidence parameter $\delta > 0$.
		
		\STATE \textbf{Initialization:} Set $\pi^{1}_{h}(a\mid s)=\frac{1}{A}$, $q_{h}^{1}(s,a,s')=\frac{1}{S^2A} , n^1_h(s,a) = 0 , n^1_h(s,a,s')$ for every $(s,a,s',h) \in \calS \times \calA \times \calS \times [H]$ and $\calP^1$ be the set of all transition functions. 
		
		\FOR{$k=1,2,...,K$}
		
		\STATE Play episode $k$ with policy $\pi^k$ and observe trajectory $\{ (s^k_h,a^k_h) \}_{h=1}^H$.
		
		
		
        \STATE Define confidence set $\calP^{k+1}$ by \cref{alg:update-cofidence-set}.
		
		\FOR{$j : j + d^j = k$}
		
		\STATE Observe feedback $\{ c^j_h(s^j_h,a^j_h) \}_{h=1}^H$.
		
		\STATE Compute upper occupancy bound $u_h^j(s,a) = \max_{p' \in \calP^j} q_h^{p',\pi^j}(s,a) $.
		
		\STATE Compute loss estimator $\hatl^j_h(s,a) = \frac{c^j_h(s,a) \indevent{s^j_h = s , a^j_h = a}}{u^j_h(s,a) + \gamma}$ for every $(s,a,h) \in \calS \times \calA \times [H]$.
		
		\ENDFOR
		
		\STATE Update occupancy measure: 
		\[
		q^{k+1}=\arg\min_{q\in \cap_{j=1}^{k+1} \ocsetk{j} } \left\langle q,\sum_{j : j+d^j \le k} \hatl^{j} \right\rangle + \phi(q),
		\]
		where $\phi(q) = \frac{1}{\eta} \sum_{h,s,a,s'} q_{h}(s,a,s')\log q_{h}(s,a,s')$ 
        is the Shannon entropy regularizer,
		and $\ocsetk{k} = \{q^{\pi,p'} \mid \pi\in (\Delta_\calA)^{\calS \times [H]}, p'\in \calP^{k} \}$.
		
		\STATE Update policy: $\pi_{h}^{k+1}(a\mid s)
		=\frac{\sum_{s'} q_{h}^{k+1}(s,a,s')}{\sum_{a'} \sum_{s'} q_{h}^{k+1}(s,a',s')}$ for every $(s,a,h) \in \calS \times \calA \times [H]$.
		\ENDFOR
	\end{algorithmic}
\end{algorithm}

In this section, we show that applying the FTRL framework with normal loss estimators and fixed amount Shannon entropy can achieve $\otil\rbr{ \sqrt{K} + \sqrt{D}}$ expected regret (ignoring dependence on other parameters). We propose Algorithm~\ref{alg:ftrl-normal-loss-estimator} which based on this simple idea and \cref{thm:regret-ftrl-with-normal-loss-estimator} below shows that our algorithm essentially achieves this goal. 

As one may noticed that, compared with \cref{alg:o-reps-new-estimator-unknown} which uses the Online Mirror Descent framework, \cref{alg:ftrl-normal-loss-estimator} uses $\cap_{j=1}^{k} \ocsetk{j}$, the set of occupancy measures associated with transition functions that belong to all confidence sets prior to episode $k$, as the decision space to compute $q^k$. 
This setup is necessary to adopt the FTRL framework for ensuring that a shrinking sequence of decision sets, which is critical to analyze the penalty term as in \cref{lemma:penalty-ftrl-normal-loss-estimator}. 
Please see the proof of \cref{lemma:penalty-ftrl-normal-loss-estimator} for more details. 
On the other hand, the unknown underlying transition $p$ belongs to all the confidence sets with high probability, which ensures that the intersection of confidence sets is nonempty with high probability. 

\begin{theorem}
	\label{thm:regret-ftrl-with-normal-loss-estimator}
	With confidence parameter $\delta = \frac{1}{H^2S^2A^2K^5}$, learning rate $\eta = \sqrt{\frac{H \log \frac{HSAK}{\delta}}{HSAK + (HSA)^2 D}}$ and exploration parameter $\gamma = \sqrt{\frac{\log \frac{HSAK}{\delta}}{SAK}}$, \cref{alg:ftrl-normal-loss-estimator} ensures that 
	\[
	\E\sbr{\regret} 
	=
	O \left( H^2 S\sqrt{AK \log (HSAK)} + HSA \sqrt{HD \log (HSAK)} + H^4 S^2 A^2 \log^2 (HSAK) \right).
	\] 
\end{theorem}


\subsection{Proof of the Main Theorem}

We first decompose the regret into four terms according to the work of \cite{jin2019learning}: 
\begin{equation}
\begin{aligned}
\regret  = \underbrace{\sum_{k=1}^K \left\langle q^{\pi^k} - \estq^k, c^k \right\rangle}_{\textsc{Est}}
+ \underbrace{\sum_{k=1}^{K} \left\langle \estq^k , c^k  - \hatl^k \right\rangle}_{\regtwo} + \underbrace{\sum_{k=1}^K \left\langle \estq^k - q^*, \hatl^k \right\rangle}_{\regthree}
+ \underbrace{\sum_{k=1}^K \left\langle q^\star, \hatl^k - c^k \right\rangle}_{\regfour},  
\label{eq:general-regret-decomposition}
\end{aligned}
\end{equation}
where $\estq^k$ is the computed occupancy measure of episode $k$; $\estq^{\pi^k}$ is the underlying occupancy measure associated with the unknown transition $p$ and policy $\pi^k$; $q^\star$ is the occupancy measure of the optimal policy $\pi^\star$ in hindsight . 

Then, with the help of Lemma 4, 6 and 14 of \cite{jin2019learning}, we have the following lemma for $\textsc{Est}$, $\regtwo$ and $\regfour$.
\begin{lemma} with probability at least $1-9\delta$, Algorithm~\ref{alg:ftrl-normal-loss-estimator} ensures that 
	\begin{equation}
	\begin{aligned}
	\textsc{Est} & = \order\rbr{ H^2 S \sqrt{  A K \log \rbr{\frac{HSAK}{\delta}} } + H^4 S^2 A^2 \log^2 \rbr{\frac{HSAK}{\delta}} }, \\
	\regtwo &  = \order\rbr{H^2 S \sqrt{  A K \log \rbr{\frac{HSAK}{\delta}}}  +\gamma HSAK  }, \\
	\regfour & = \order\rbr{ \frac{H}{\gamma} \log\rbr{ \frac{HSA}{\delta} } }.  
	\end{aligned}
	\notag
	\end{equation}
	\label{lem:other-terms-ftrl-with-normal-loss-estimator}
\end{lemma}
\begin{proof}
	Without loss of generality, we convert our MDP setting to that of \cite{jin2019learning} by setting $\mathcal{X} = \calS \times [H]$ and $L = H$. Then, by direct application of Lemma 4, 6 and 14 of \cite{jin2019learning} (which are combined together in the proof of Theorem 3), we arrive at the high-probability bounds of these terms. 
	Note that, the double epoch scheduling and larger confidence sets of transition functions only changes the constant of regret bound, which is hidden in $\order\rbr{\cdot}$ operator. 
\end{proof}

Based on the high-probability bound, we have the following corollary for the expected bound of these terms. 
\begin{corollary} Algorithm~\ref{alg:ftrl-normal-loss-estimator}  ensures that $\E\sbr{ \textsc{Est} + \regtwo + \regfour  }$ is bounded at most $\order\rbr{H^4 S^2 A^2 \log^2 \rbr{\frac{HSAK}{\delta}}}$ plus:
	\begin{equation}
	\order\rbr{ H^2 S \sqrt{  A K \log \rbr{\frac{HSAK}{\delta}} }+ \gamma HSAK + \frac{H}{\gamma} \log\rbr{ \frac{HSA}{\delta} } + HK \delta   }. \notag  
	\end{equation}
	\label{col:other-terms-ftrl-with-normal-loss-estimator}
\end{corollary}

Then, we prove the following lemma for the expected bound of $\regthree$ with the help a unique novel analysis, and defer 
the complete proof to to \cref{sec:bound_of_regthree_FTRL}. 

\begin{lemma} Algorithm~\ref{alg:ftrl-normal-loss-estimator}  ensures that $\E\sbr{ \regthree  }$ is bounded by:
	\begin{equation}
		\order\rbr{ \frac{ H \ln \rbr{ S^2 A  }  }{\eta} + \eta \rbr{ HSAK + (HSA)^2 D } +  \frac{H^2S^2A^2K^3}{\gamma^2} \delta}.  \notag  
	\end{equation}
	\label{lemma:est-regret-ftrl-with-normal-loss-estimator}
\end{lemma}

With the help of above lemmas,  we are ready to prove the \cref{thm:regret-ftrl-with-normal-loss-estimator}. 
\begin{proof}[Proof of \cref{thm:regret-ftrl-with-normal-loss-estimator}] 
	
	Combining the expected bound of $\textsc{Est} + \regtwo + \regfour$ in Corollary~\ref{col:other-terms-ftrl-with-normal-loss-estimator} and that of $\regthree$ in \cref{lemma:est-regret-ftrl-with-normal-loss-estimator}, we are able to show that  the expected regret $\E\sbr{\regret}$ is bounded by
	\begin{align*}
	& \order\rbr{ H^2 S \sqrt{  A K \log \rbr{\frac{HSAK}{\delta}} } +  \gamma HSAK + \frac{H}{\gamma} \log\rbr{ \frac{HSA}{\delta} } + \frac{ H \ln \rbr{ S^2 A  }  }{\eta} + \eta \rbr{ HSAK + (HSA)^2 D } } \\
	& \quad + \order\rbr{ \frac{H^2S^2A^2K^3}{\gamma^2} \delta +  H^4 S^2 A^2 \log^2 \rbr{\frac{HSAK}{\delta}}  }.
	\end{align*} 
	
	Finally, selecting a small enough confidence parameter $\delta = \frac{1}{H^2S^2A^2K^5}$ and picking up the learning rate $\eta = \sqrt{\frac{H \log \frac{HSAK}{\delta}}{HSAK + (HSA)^2 D}}$ and the exploration parameter $\gamma = \sqrt{\frac{\log \frac{HSAK}{\delta}}{SAK}}$ ensure that 
	\begin{align*}
	\E\sbr{\regret}& = O\rbr{  H^2 S \sqrt{ A K \log (HSAK) } + HSA\sqrt{HD \log (HSAK)} +  H^4 S^2 A^2 \log^2 (HSAK) }. \qedhere
	\end{align*}
\end{proof}

\subsection{Bound on the Regret with respect to the Loss Estimators ($\regthree$ in \cref{eq:general-regret-decomposition})}
\label{sec:bound_of_regthree_FTRL}
In this part, we focus on $\regthree$ defined in Eq~\eqref{eq:general-regret-decomposition} with delayed feedback of losses, and prove \cref{lemma:est-regret-ftrl-with-normal-loss-estimator} through the introduced key steps in \cref{sec-paper:ftrl}. To this end, we will use the following decomposition of $\regthree$ in this section:
\begin{equation}
\begin{aligned}
\regthree  = \sum_{k=1}^{K} \inner{ \estq^k - q^\star, \hatl^k } &  = \sum_{k=1}^{K} \Phi_k( \estq^k ) + \inner{ \estq^k, \hatl^k } - \Phi_k'( \westq^k) &&(\textsc{Stability}) \notag \\ 
& + \sum_{k=1}^{K} \Phi_k'( \westq^k) - \Phi_k( \estq^k ) - \rbr{ \Phi_k^C( \wtilq_k' ) - \Phi_k^B( \wtilq_k) } &&(\textsc{Delay-caused Drift}) \notag \\
& + \sum_{k=1}^{K}  \Phi_k^C( \wtilq_k' ) - \Phi_k^B( \wtilq_k) - \inner{q^\star, \hatl^k} &&(\textsc{Penalty})
\end{aligned}
\end{equation}
where the functions $\Phi_k, \Phi_k', \Phi_k^B, \Phi_k^C$ and the occupancy measures $\estq^k, \westq^k, \wtilq_k, \wtilq_k'$ are defined as 
\begin{align}
\Phi_k(q) & = \inner{q, \hatL_k^{\text{obs}} } + \phi(q),  & \estq^k = \argmin_{q\in \cap_{j=1}^k \ocsetk{j} } \Phi_k(q), \notag \\
\Phi_k'(q) & = \inner{q, \hatL_k^{\text{obs}} + \hatl^k } + \phi(q),  & \westq^k= \argmin_{q\in \cap_{j=1}^k \ocsetk{j} } \Phi_k'(q),  \notag \\
\Phi_k^B(q) & = \inner{q, \hatL_k } + \phi(q),  & \wtilq_k = \argmin_{q\in \cap_{j=1}^k \ocsetk{j} } \Phi_k^B(q),   \notag \\
\Phi_k^C(q) & = \inner{q, \hatL_k + \hatl^k } + \phi(q),  & \wtilq_k' = \argmin_{q\in \cap_{j=1}^k \ocsetk{j} } \Phi_k^C(q).  \notag 
\end{align}
with $\hatL_k = \sum_{j=1}^{k-1} \hatl^j$ being the un-delayed cumulative loss estimator prior to episode $k$, and $\hatL_k^{\text{obs}} =  \sum_{j=1, j+ d^j <k }^{k-1}\hatl^j$ being the received cumulative loss estimator. 

On the other hand, with the help of $F_k^\star(x) = -\min_{q\in \cap_{j=1}^k \ocsetk{j} } \cbr{ \phi(x)  - \inner{x,q} }$, the convex conjugate with respect to  $\phi (\cdot)$, these functions and occupancy measures ensures that
\begin{align}
\Phi_k(\estq^k)  = - F_k^\star \rbr{ - \hatL_k^{\text{obs}}  }, \Phi_k'(\westq^k) = - F_k^\star \rbr{ - \hatL_k^{\text{obs}} - \hatl^k }, \Phi_k^B(\wtilq_k) = - F_k^\star \rbr{ - \hatL_k },  \Phi_k^C(\wtilq_k') = - F_k^\star \rbr{ - \hatL_k - \hatl^k }. \notag 
\end{align}
In addition, according to the property of convex conjugates, these occupancy measures are able to be presented as the gradient of the convex conjugate with different inputs as 
\begin{align}
\estq^k = \nabla F_k^\star \rbr{ - \hatL_k^{\text{obs}} },  \westq^k = \nabla F_k^\star \rbr{ - \hatL_k^{\text{obs}} - \hatl^k }, \wtilq_k = \nabla F_k^\star \rbr{ - \hatL_k } ,  \wtilq_k' =\nabla F_k^\star \rbr{ - \hatL_k - \hatl^k}. \notag 
\end{align} 

For notational convenience, we denote $\hatd_k =  \hatL_k  - \hatL_k^{\text{obs}}$ as the summation of un-received loss estimators prior to episode $k$, that is, $\hatd_k = \sum_{j = 1, j+d^j \geq k}^{k-1} \hatl_j$. Thus, $\Phi_k^B(\wtilq_k)$ and $\Phi_k^C(\wtilq_k') = - F_k^\star \rbr{ - \hatL_k - \hatl^k }$ can be represented as 
\begin{align}
\Phi_k^B(\wtilq_k') =  - F_k^\star \rbr{ - \hatL_k^{\text{obs}} - \hatd_k }, \Phi_k^C(\wtilq_k') =  - F_k^\star \rbr{ - \hatL_k^{\text{obs}}  -  \hatd_k - \hatl^k }. \notag
\end{align}

With the help of these definitions, we are now ready to bound the terms $\textsc{Stabilty}$, $\textsc{Delay-caused Drfit}$ and $\textsc{Penalty}$ in following lemmas. 

\begin{lemma} (Stability) With fixed learning rate $\eta>0$ and exploration $\gamma>0$, \cref{alg:ftrl-normal-loss-estimator} ensures that 
	\begin{equation}
	\sum_{k=1}^{K}\Phi_k(\estq^k) + \inner{\estq^k, \hatl^k} - \Phi'_k(\westq^k) \leq \eta \sum_{k=1}^{K} \sum_{h,s,a} \estq^k_h(s,a) \widehat{c}_h^k(s,a)^2.
	\notag  
	\end{equation}
	\label{lemma:stability-ftrl-normal-loss-estimator}
\end{lemma} 
\begin{proof}  Let  $D_k\rbr{u ,v} =  \phi(u) - \phi(v) - \inner{ u - v , \nabla \phi(v) }$ be the Bregman divergence with the convex regularizer $\phi$.  Then, 
	\begin{align}
	\Phi_k(\estq^k) & = \inner{\estq^k, \hatL_k^{\text{obs}}} + \phi(\estq^k) \notag =  \inner{\westq^k, \hatL_k^{\text{obs}} } + \phi(\westq^k) - \rbr{ \inner{ \westq^k- \estq^k, \hatL_k^{\text{obs}} } + \phi(\westq^k)  - \phi(\estq^k)   }  \notag \\
	& \leq \inner{\westq^k, \hatL_k^{\text{obs}} } + \phi(\westq^k) - \rbr{   -\inner{ \westq^k- \estq^k, \nabla \phi (\estq^k) } + \phi(\westq^k)  - \phi(\estq^k)   }   \notag \\
	& = \inner{\westq^k, \hatL_k^{\text{obs}} } + \phi(\westq^k) - D_k( \westq^k,\estq^k ) \notag = \Phi_k'(\westq^k) -  \inner{\westq^k, \hatl^k} - D_k( \westq^k,\estq^k ), \notag 
	\end{align}
	where the third step follows from the first order optimality of $q^k$ with respect to $\Phi_k$, in other words, $\inner{ \whatq^k - \estq^k, \hatL_k^{\text{obs}} + \nabla \phi(q^k) } \geq 0$. 
	Rearranging terms and adding $\inner{\estq^k,\hatl^k}$ on both sides give us the following inequality: 
	\begin{align}
	\Phi_k(\estq^k) + \inner{\estq^k, \hatl^k} - \Phi_k'(\westq^k) \leq \inner{\estq^k - \westq^k, \hatl^k} -  D_k( \westq^k,\estq^k ).\notag 
	\end{align}
	To bound the right hand side term, we relax the constraints and taking the maximum as:
	\begin{align}
	\inner{\estq^k - \westq^k, \hatl^k} -  D_k( \westq^k,\estq^k ) \leq \max_{q\in \fR_{\geq 0}^{\calS \times \calA \times [H] \times S } } \inner{\estq^k - q, \hatl^k} -  D_k(q,\estq^k)= \inner{\estq^k - \xi^k, \hatl^k} -  D_k(\xi^k,\estq^k), \notag 
	\end{align}
	where $\xi_k$ denotes the maximizer point. Setting the gradient to zero gives the equality that $\nabla \phi(\estq^k) - \nabla \phi( \xi^k ) = \hatl^k$. By direct calculation, one can verify that $\xi_h^k(s,a,s') = \estq_h^k(s,a,s') \cdot \exp\rbr{ - \eta \hatl_h^k(s,a) }$ for all  transition tuples. Therefore, we have the following inequality that 
	\begin{align}
	\inner{\estq^k - \xi^k, \hatl^k} -  D_k(\xi^k,\estq^k) & = \inner{\estq^k - \xi^k,\hatl^k} - \phi(\xi^k) + \phi(\estq^k) - \inner{\estq^k - \xi^k, \nabla \phi(\estq^k) }  \notag = D_k(\estq^k, \xi^k) \notag \\
	& = \frac{1}{\eta} \sum_{h=1}^{H} \sum_{s,a,s'} \rbr{  \estq_h^k(s,a,s') \ln\rbr{ \frac{\estq_h^k(s,a,s')}{\xi_h^k(s,a,s')} } - \estq_h^k(s,a,s')  + \xi_h^k(s,a,s')     } \notag \\
	& = \frac{1}{\eta} \sum_{h=1}^{H} \sum_{s,a,s'} \estq_h^k(s,a,s')  \rbr{ \eta \hatl_h^k(s,a) - 1 + \exp\rbr{-\eta \hatl_h^k(s,a) }    }   \notag \\
	& \leq \eta \sum_{h=1}^{H} \sum_{s,a,s'} \estq_h^k(s,a,s') \hatl_h^k(s,a)^2 \notag = \eta \sum_{h=1}^{H} \sum_{s,a} \estq_h^k(s,a) \hatl_h^k(s,a)^2, \notag 
	\end{align}
	where the second step uses $\nabla \phi(\estq^k) - \nabla \phi( \xi_k ) = \hatl^k$; the forth step follows from the fact that $e^{-x} \leq 1 - x + x^2$ for any $x\geq 0$ . Finally, taking the summation over all episodes finishes the proof. 
\end{proof}

\begin{lemma} (Delay-caused Drift) Algorithm~\ref{alg:ftrl-normal-loss-estimator} guarantees that
	\begin{align}
	\sum_{k=1}^{K} \Phi_k'( \westq^k) - \Phi_k( \estq^k ) - \rbr{ \Phi_k^C( \wtilq_k' ) - \Phi_k^B( \wtilq_k) } \leq 2  \eta \sum_{k=1}^{K}  \rbr{\sum_{h=1}^{H} \sum_{s,a} \hatl_h^k(s,a)}\cdot\rbr{\sum_{h=1}^{H} \sum_{s,a} \hatd_{h}^k(s,a)}. \notag 
	\end{align}
	\label{lemma:delay-drift-ftrl-normal-loss-estimator}
\end{lemma}

\begin{proof} 
	With the help of the convex conjugate $F^\star_k(\cdot)$, we have the following inequality holds for some $\theta \in [0,1]$ that:
	\begin{align}
	\Phi_k'( \westq^k) - \Phi_k( \estq^k ) - \rbr{ \Phi_k^C( \wtilq_k' ) - \Phi_k^B( \wtilq_k) } \notag 
	& = - F_k^\star( - \hatL_k^{\text{obs}} - \hatl^k  ) + F_k^\star( - \hatL_k^{\text{obs}}  )  -  \rbr{  - F_k^\star( - \hatL_k - \hatl^k ) +  F_k^\star( - \hatL_k )  } \notag  \\ 
	& = \int_{0}^{1} \inner{  \hatl^k, \nabla F_k^\star( - \hatL_k^{\text{obs}}  - x \hatl^k ) } dx -  \int_{0}^{1} \inner{  \hatl^k, \nabla F_k^\star( - \hatL_k  - x \hatl^k ) } dx \notag \\
	& = \int_{0}^{1} \inner{  \hatl^k, \nabla F_k^\star( - \hatL_k^{\text{obs}}  - x \hatl^k ) - \nabla F_k^\star( - \hatL_k  - x \hatl^k ) } dx \notag  \\
	& = \inner{ \hatl^k, \nabla F_k^\star( - \hatL_k^{\text{obs}}  - \theta \hatl^k ) - \nabla F_k^\star( - \hatL_k  - \theta \hatl^k ) }, \notag  
	\end{align}	
	where the second step uses Newton-Leibniz theorem; the forth step uses the mean value theorem.
	To analyze the right hand side, we define the functions $W$ and $W'$ as 
	\begin{equation}
	W(q) = \inner{q,\hatL_k^{\text{obs}}  + \theta \hatl^k} + \phi(q)
	\quad ; \quad
	W'(q) = \inner{q,\hatL_k + \theta \hatl^k} + \phi(q), \notag
	\end{equation}
	and denote their minimizer occupancy measures within the decision set $\cap_{j=1}^k \ocsetk{j}$ by $u$ and $v$. According to the properties of convex conjugate, we have $u = \nabla F_k^\star( - \hatL_k^{\text{obs}}  - \theta \hatl^k )$ and   $v = \nabla F_k^\star( - \hatL_k  - \theta \hatl^k )$.  
	
	To analyze $\inner{ u - v, \hatl^k}$,  we first lower bound $W(u) + \inner{u,  \hatd_k } - W'(v)$ as 
	\begin{align*}
	W(u) + \inner{u,  \hatd_k } - W'(v) = W'(u) - W'(v) = \inner{ \nabla W'(v), u - v } + \frac{1}{2} \norm{ u - v }^2_{\nabla^2 \phi(\xi) } \ge \frac{1}{2} \norm{ u - v }^2_{\nabla^2 \phi(\xi) },
	\end{align*}
	where the second step applies Taylor’s expansion with $\xi$ being an intermediate point between $u$ and $v$; the last step uses the first order optimality condition of $v$. 
	On the other hand, we can upper $W(u) + \inner{u,  \hatd_k } - W'(v)$ as 
	\begin{align*}
	W(u) + \inner{u,  \hatL_k  - \hatL_k^{\text{obs}} } - W'(v)  & = W(u) - W(v) +  \inner{ u - v, \hatL_k  - \hatL_k^{\text{obs}}   } \leq \inner{ u - v, \hatL_k  - \hatL_k^{\text{obs}}   } \\
	& \leq \norm{u - v}_{\nabla^2 \phi(\xi) } \norm{ \hatL_k  - \hatL_k^{\text{obs}}   }_{\nabla^{-2} \phi(\xi) },
	\end{align*}
	where the second step uses the optimality of $u$, and the last step comes from H\"older's inequality.
	Combining the lower bound and upper bound, we arrives at the following inequality 
	\begin{align*} 
	\norm{u - v}_{\nabla^2 \phi(\xi) } \leq 2 \norm{ \hatL_k  - \hatL_k^{\text{obs}}   }_{\nabla^{-2} \phi(\xi) }. 
	\end{align*}
	Therefore, we can upper bound the term $\inner{ \hatl^k, u - v }$ with the help of H\"older's inequality again as
	\begin{align*}
	\inner{ \hatl^k, u - v} \leq \norm{ \hatl^k }_{\nabla^{-2} \phi(\xi) } \norm{u - v}_{\nabla^2 \phi(\xi) } \leq 2 \norm{ \hatl^k }_{\nabla^{-2} \phi(\xi) } \norm{ \hatL_k  - \hatL_k^{\text{obs}}   }_{\nabla^{-2} \phi(\xi) }.
	\end{align*}
	
	By direct calculation, one can verify the following:
	\begin{align*}
	2\norm{ \hatl^k }_{\nabla^{-2} \phi(\xi) } \cdot \norm{ \hatd_k   }_{\nabla^{-2} \phi(\xi) }& = 2 \sqrt{ \eta \sum_{h=1}^{H} \sum_{s,a,s'}  \hatl_h^k(s,a)^2 \xi(s,a,s')  } \cdot \sqrt{  \eta \sum_{h=1}^{H} \sum_{s,a,s'}  \hatd_h^k(s,a)^2 \xi(s,a,s')  } \\
	& \leq 2 \eta \sqrt{ \sum_{h=1}^{H} \sum_{s,a}  \hatl_h^k(s,a)^2  } \cdot \sqrt{ \sum_{h=1}^{H} \sum_{s,a}   \hatd_h^k(s,a)^2   } \\
	& \leq 2 \eta \rbr{ \sum_{h=1}^{H} \sum_{s,a}  \hatl_h^k(s,a) } \cdot \rbr{ \sum_{h=1}^{H} \sum_{s,a}   \hatd_h^k(s,a) },
	\end{align*}
	where the second step follows from the fact that $\xi$ is a valid occupancy measure and $\sum_{s'}\xi(s,a,s') =  \xi(s,a) \leq 1$ holds for all state-action pairs. Taking the summation over all episodes concludes the proof. 
\end{proof}

\begin{lemma} (Penalty) With the shrinking decision set sequence that $\cap_{j=1}^{k+1} \ocsetk{j}\subset \cap_{j=1}^k \ocsetk{j}$ for $k=1,\ldots K-1$, Algorithm~\ref{alg:ftrl-normal-loss-estimator} ensures that
	\begin{equation}
	\sum_{k=1}^{K}  \Phi_k^C( \wtilq_k' ) - \Phi_k^B( \wtilq_k) - \inner{q^\star, \hatl^k}  \leq \frac{ H \ln \rbr{ S^2 A  }  }{\eta}. \notag 
	\end{equation}
	\label{lemma:penalty-ftrl-normal-loss-estimator}
\end{lemma}
\begin{proof} First, we observe that 
	\begin{align*}
	\Phi_k^C( \wtilq_k' ) & = \min_{q\in \cap_{j=1}^k \ocsetk{j}} \inner{ q, \hatL_k + \hatl^k } + \phi(q) \leq \min_{q\in \cap_{j=1}^{k+1} \ocsetk{j} } \inner{ q, \hatL_k + \hatl^k } + \phi(q) \\
	& = \min_{q\in \cap_{j=1}^{k+1} \ocsetk{j} } \inner{ q, \hatL_{k+1} } + \phi(q)  = \Phi_{k+1}^B( \wtilq_{k+1} ),
	\end{align*}
	where the second step follows from the fact that $\calP^{k+1} \subset \calP^k$ by the definition. 
	Therefore, we have the following inequality:
	\begin{align*}
	\sum_{k=1}^{K}  \Phi_k^C( \wtilq_k' ) - \Phi_k^B( \wtilq_k) - \inner{q^\star, \hatl^k} & = \Phi_K^C( \wtilq_K' ) - \Phi_1^B( \wtilq_1) - \inner{q^\star, \hatL_{K+1}}  + \sum_{k=1}^{K-1} \Phi_k^C( \wtilq_k' ) - \Phi_{k+1}^B( \wtilq_{k+1}) \\ 
	& \leq \Phi_K^C( \wtilq_K' ) - \Phi_1^B( \wtilq_1) - \inner{q^\star, \hatL_{K+1}}\leq \phi(q^\star) - \phi(\wtilq_1)  \leq \frac{ H \ln \rbr{ S^2 A  }  }{\eta},
	\end{align*}
	where the third step follows from the optimality of $\wtilq_K'$ and the last steps follows the standard argument of Shannon entropy (such as, Lemma 12 of \cite{jin2019learning}). 
\end{proof}

We are now ready to prove \cref{lemma:est-regret-ftrl-with-normal-loss-estimator} by combining the results of \cref{lemma:penalty-ftrl-normal-loss-estimator,lemma:delay-drift-ftrl-normal-loss-estimator,lemma:stability-ftrl-normal-loss-estimator} and taking the expectation. 

\begin{proof}[Proof of \cref{lemma:est-regret-ftrl-with-normal-loss-estimator}] 
By combining \cref{lemma:penalty-ftrl-normal-loss-estimator,lemma:delay-drift-ftrl-normal-loss-estimator,lemma:stability-ftrl-normal-loss-estimator}, we have $\regthree$ bounded as 
	\begin{align*}
	\regthree & \leq \frac{ H \ln \rbr{ S^2 A  }  }{\eta} + \eta \sum_{k=1}^{K} \sum_{h=1}^{H} \sum_{s,a} \estq_h^k(s,a) \hatl_h^k(s,a)^2 +  2 \eta \sum_{k=1}^{K} \sum_{h=1}^{H} \sum_{s,a} \sum_{h'=1}^{H} \sum_{s',a'}\hatl_h^k(s,a) \hatd_{h'}^k(s',a')  . 
	\end{align*}
	
	To analyze the expectation, we use the indicator $Z_k = \Ind{ p \notin \calP_k  }$ to denote the event that the true transition function $p$ is not included in the confidence set of episode $k$. Clearly, one can verify that $q_h^k(s,a) \leq Z_k + u_h^k(s,a)$ and $q_h^{\pi^k}(s,a) \leq Z_k + u_h^k(s,a)$ due to the definition of upper occupancy bound $u_k$ and the property of occupancy measures. Therefore, we are able to bound $\E\sbr{\regthree}$ by 
	\begin{align*}
	& \frac{ H \ln \rbr{ S^2 A  }  }{\eta} + \eta \E\sbr{\sum_{k=1}^{K} \sum_{h=1}^{H} \sum_{s,a} \estq_h^k(s,a) \hatl_h^k(s,a)^2 +  2 \sum_{h=1}^{H} \sum_{s,a} \sum_{h'=1}^{H} \sum_{s',a'}\hatl_h^k(s,a) \hatd_h^k(s,a)} \\
	& \leq \frac{ H \ln \rbr{ S^2 A  }  }{\eta} + \eta \E\sbr{ \sum_{k=1}^{K} \E_k\sbr{\sum_{h=1}^{H} \sum_{s,a}  \hatl_h^k(s,a) +  2 \eta \sum_{h=1}^{H} \sum_{s,a} \sum_{h'=1}^{H} \sum_{s',a'}\hatl_h^k(s,a) \hatd_h^k(s,a)} } \\
	& \leq \frac{ H \ln \rbr{ S^2 A  }  }{\eta} + \eta \E\sbr{ \sum_{k=1}^{K}  \sum_{h=1}^{H} \sum_{s,a} \frac{q_h^{\pi^k}(s,a)}{u_h^k(s,a) + \gamma } + 2 \sum_{j=1,j+d^j \geq k}^{k-1}\sum_{h'=1}^{H}\sum_{s',a'} \frac{q_h^{\pi^k}(s,a)}{u_h^k(s,a) + \gamma } \frac{q_{h'}^{\pi^j}(s',a') }{ u_{h'}^j(s',a') + \gamma} } \\
	& \leq \frac{ H \ln \rbr{ S^2 A  }  }{\eta} + \eta \rbr{ HSAK + 2(HSA)^2 D } + \frac{HSAK + 4(HSA)^2 D}{\gamma^2} \cdot \E\sbr{ \sum_{k=1}^{K} Z_k  } ,
	\end{align*}
	where the first step uses the fact that $\estq_h^k(s,a) \leq u_h^k(s,a)$ for any state-action pair; the second step uses the definition of loss estimators; the third step follows from the fact that $q_h^{\pi^k}(s,a) \leq Z_k + u_h^k(s,a)$. 
	
	According to Lemma 2 of \cite{jin2019learning}, we have the expectation of $\E\sbr{ \sum_{k=1}^{K} Z_k }$ bounded by $4K\delta$, and the following upper bound of $\E\sbr{\regthree}$: 
	\begin{align*}
	\order\rbr{ \frac{ H \ln \rbr{ S^2 A  }  }{\eta} + \eta \rbr{ HSAK + (HSA)^2 D } +  \frac{H^2S^2A^2K^3}{\gamma^2} \delta}. 
	\end{align*}
\end{proof}

\newpage

\section{Delayed O-REPS with delay-adapted estimator}
\label{appendix:delay-adapted O-REPS known}

\begin{algorithm}[t]
    \caption{Delayed O-REPS with delay-adapted estimator and known transition} \label{alg:o-reps-new-estimator}
    \begin{algorithmic}
        \STATE \textbf{Input:} State space $\calS$, Action space $\calA$, Horizon $H$, Number of episodes $K$, Transition function $p$, Learning rate $\eta > 0$, Exploration parameter $\gamma > 0$.
        
        \STATE \textbf{Initialization:} Set $\pi^{1}_{h}(a\mid s)=\frac{1}{A}$, $q_{h}^{1}(s,a)=\frac{1}{SA}$ for every $(s,a,h) \in \calS \times \calA \times [H]$.
        
        \FOR{$k=1,2,...,K$}
        
            \STATE Play episode $k$ with policy $\pi^k$ and observe trajectory $\{ (s^k_h,a^k_h) \}_{h=1}^H$.
            
            \FOR{$j : j + d^j = k$}
            
                \STATE Observe feedback $\{ c^j_h(s^j_h,a^j_h) \}_{h=1}^H$.
                
                \STATE Compute loss estimator $\hat c^j_h(s,a) = \frac{c^j_h(s,a) \indevent{s^j_h = s , a^j_h = a}}{\max\{q^j_h(s,a),q^k_h(s,a)\} + \gamma}$ for every $(s,a,h) \in \calS \times \calA \times [H]$.
            
            \ENDFOR
            
            \STATE Update occupancy measure: 
            \begin{align}
                q^{k+1} = \arg \min_{q \in \ocset} \eta \left\langle q , \sum_{j:j+d^j=k} \hat c^j \right\rangle + \KL{q}{q^k},
                \label{eq:OMD-update-known}
            \end{align}
            where $\KL{q}{q'} = \sum_{h,s,a} q_h(s,a) \ln \frac{q_h(s,a)}{q'_h(s,a)} + q'_h(s,a) - q_h(s,a)$.
            
            \STATE Update policy: $\pi_{h}^{k+1}(a\mid s)
            =\frac{q_{h}^{k+1}(s,a)}{\sum_{a'}q_{h}^{k+1}(s,a')}$ for every $(s,a,h) \in \calS \times \calA \times [H]$.
        \ENDFOR
    \end{algorithmic}
\end{algorithm}

Explicitly solving this optimization problem in \cref{eq:OMD-update-known}, we get \cite{zimin2013online}:
\[
    q^{k+1}_h(s,a) = \frac{q^k_h(s,a) e^{B^k_h(s,a \mid v^k)}}{Z^k_h(v^k)},
\]
for:
\begin{align*}
    B^k_h(s,a \mid v) 
    & = 
    v_h(s) -\eta \sum_{j:j+d^j=k} \hat c^j_h(s,a) - \sum_{s'} p_h(s' \mid s,a) v_{h+1}(s')
    \\
    Z^k_h(v) 
    & = 
    \sum_{s,a} q^k_h(s,a) e^{B^k_h(s,a \mid v)}
    \\
    v^k 
    & = 
    \arg \min_v \sum_h \log Z^k_h(v).
\end{align*}
These different formulations will be helpful in the regret analysis.

\begin{theorem}
    \label{thm:regret-o-reps-new-estimator}
     Running O-REPS with the delay-adapted estimator, $\eta = \gamma = \min \{ \sqrt{\frac{\log \frac{HSA}{\delta}}{SAK}} , \sqrt{\frac{\log \frac{HSA}{\delta}}{\sqrt{HSA} D}} \}$ guarantees, with probability $1 - \delta$,
    \[
        \regret
        =
        O \left( H\sqrt{SAK \log \frac{HSA}{\delta}} + (H S A)^{1/4} \cdot H \sqrt{D \log \frac{HSA}{\delta}} + H^{3/2} d_{max} \log \frac{H}{\delta} \right).
    \]
\end{theorem}

\subsection{The good event}

Let $\tilde{\mathcal{H}}^{k}$ be the history of episodes $\{j:j+d^{j}<k\}$.
Define the following events: 
\begin{align*}
    E^c 
    & = 
    \left\{ \sum_{k=1}^K \langle \bbE [ \hat c^k \mid \filt{k+d^k} ] - \hat c^k , q^k \rangle \le 4 H \sqrt{K \log \frac{10}{\delta}} \right\}
    \\
    E^{\hat c}
    & =
    \left\{ \sum_{k=1}^{K}\langle|q^{k}-q^{k+d^k}|,\hat{c}^{k}\rangle \le 4 \sum_{k=1}^{K}\langle|q^{k}-q^{k+d^k}|,c^{k}\rangle + \frac{40H \log \frac{10 H}{\delta}}{\gamma} \right\}
    \\
    E^d
    & =
    \left\{ \sum_{k,h,s,a} |\calF^{k+d^k}| \hat c^k_h(s,a) \le \sum_{k,h,s,a} |\calF^{k+d^k}| c^k_h(s,a) + \frac{10 H d_{max} \log \frac{10 H}{\delta}}{\gamma}  \right\}
    \\
    E^{sq}
    & =
    \left\{ \sum_{k=1}^{K} \sum_{i=1}^{K} \indevent{k\leq i+d^{i} < k+d^{k}} \sum_{h,s,a}\sqrt{q_{h}^{i+d^{i}}(s,a)} (\hat{c}_{h}^{i}(s,a) - 4 c_{h}^{i}(s,a) ) \le \frac{10 H d_{max} \log \frac{10 H}{\delta}}{\gamma}  \right\}
    \\
    E^\star
    & =
    \left\{ \sum_{k=1}^K \langle \hat c^k - c^k , q^\star \rangle \le \frac{H  \log \frac{10 H S A}{\delta}}{\gamma} \right\}
\end{align*}

The good event is the intersection of the above events. 
The following lemma establishes that the good event holds with high probability. 

\begin{lemma}[The Good Event]
    \label{lemma:good-event-o-reps-new-estimator}
    Let $\bbG =E^c \cap E^{\hat c} \cap E^d \cap E^{sq} \cap E^\star$ be the good event. 
    It holds that $\Pr [ \bbG ] \geq 1-\delta$.
\end{lemma}

\begin{proof}
    We show that each of the events $\neg E^c, \neg E^{\hat c}, \neg E^d, \neg E^{sq}, \neg E^\star$ occur with probability at most $\delta / 5$. Then, by a union bound we obtain the statement.
    \begin{itemize}
        \item $\Pr[\neg E^c] < \delta/5$ by Azuma inequality since it is a martingale with respect to the filtration $\{\tilde{\mathcal{H}}^{1+d^{1}},\tilde{\mathcal{H}}^{2+d^{2}},\dots\}$ where the differences are bounded by $H$.
        
        \item $\Pr[\neg E^{\hat c}] < \delta/5$ by \cite[Lemma E.2]{cohen2021minimax} since $\langle|q^{k}-q^{k+d^k}|,\hat{c}^{k}\rangle \le H/\gamma$, and $\bbE [ \langle|q^{k}-q^{k+d^k}|,\hat{c}^{k}\rangle \mid \filt{i+d^i} ] \le \langle|q^{k}-q^{k+d^k}|,c^{k}\rangle$.
        
        \item $\Pr[\neg E^d] < \delta/5$ by \cite[Lemma 11]{jin2019learning}.
        
        \item $\Pr[\neg E^{sq}] < \delta/5$ by \cite[Lemma E.2]{cohen2021minimax} in the following way. 
        Denote $Y_i = \sum_k \indevent{k\leq i+d^{i} < k+d^{k}} \sum_{h,s,a}\sqrt{q_{h}^{i+d^{i}}(s,a)} \hat{c}_{h}^{i}(s,a)$ and notice that $Y_i \le H d_{max} / \gamma$, and that:
        \[
            \bbE [ Y_i \mid \filt{i+d^i} ]
            \le
            \sum_k \indevent{k\leq i+d^{i} < k+d^{k}} \sum_{h,s,a}\sqrt{q_{h}^{i+d^{i}}(s,a)} c_{h}^{i}(s,a).
        \]
        
        \item $\Pr[\neg E^\star] < \delta/5$ by \cref{lem:jin2019_lemma_14}.
        \qedhere
    \end{itemize}
\end{proof}

\subsection{Proof of the Main Theorem}

\begin{proof}[Proof of \cref{thm:regret-o-reps-new-estimator}]
    By \cref{lemma:good-event-o-reps-new-estimator}, the good event holds with probability $1 - \delta$.
    We now analyze the regret under the assumption that the good event holds.
    We decompose the regret as follows:
    \begin{align}
        \nonumber
        \regret
        & =
        \sum_{k=1}^K \langle q^k - q^\star , c^k \rangle
        \\
        & = 
        \underbrace{\sum_{k=1}^{K}\langle q^{k},c^{k}-\hat{c}^{k}\rangle }_{\textsc{Bias}_1}
        + 
        \underbrace{\sum_{k=1}^{K}\langle q^\star,\hat{c}^{k}-c^{k}\rangle}_{\textsc{Bias}_2}
        +
        \underbrace{\sum_{k=1}^{K}\langle q^{k} - q^{k+d^k},\hat{c}^{k}\rangle}_{\textsc{Drift}}
        +
        \underbrace{\sum_{k=1}^{K}\langle q^{k+d^k} - q^\star,\hat{c}^{k}\rangle}_{\textsc{Reg}}.
        \label{eq:regret decomposition delay-adapted O-REPS known}
    \end{align}
    $\textsc{Bias}_2$ is bounded under event $E^\star$ by $O(\frac{H \log \frac{HSA}{\delta}}{\gamma})$, $\textsc{Reg}$ is bounded in \cref{lemma:Reg-o-reps-new-estimator} by $O (\frac{H \log (HSA)}{\eta} + \eta HSAK + \frac{\eta}{\gamma} d_{max} \log \frac{H}{\delta})$, $\textsc{Drift}$ is bounded in \cref{lemma:Drift-o-reps-new-estimator} by $O (\eta \sqrt{H^3 S A} (D + K) + \frac{\eta}{\gamma} H^{3/2} d_{max} \log \frac{H}{\delta} + \frac{H \log \frac{H}{\delta}}{\gamma} )$, and $\textsc{Bias}_1$ is bounded in \cref{lemma:Bias-1-o-reps-new-estimator} by $O (H \sqrt{K \log \frac{1}{\delta}} +\gamma HSAK + \eta\sqrt{H^3 S A}(D + K) + \frac{\eta}{\gamma}H^{3/2}d_{max} \log \frac{H}{\delta} )$.
    Putting everything together:
    \[
        \regret
        =
        O \left( H \sqrt{K \log \frac{1}{\delta}} + (\eta + \gamma) H S A K + (\frac{1}{\eta} + \frac{1}{\gamma}) H \log \frac{HSA}{\delta} + \eta \sqrt{H^3 S A} (D+K) + \frac{\eta}{\gamma} H^{3/2} d_{max} \log \frac{H}{\delta} \right),
    \]
    and plugging in the definitions of $\eta$ and $\gamma$ finishes the proof.
\end{proof}

\subsection{Bound on the Regret with respect to the Loss Estimators and Future Policies (\textsc{Reg} in \cref{eq:regret decomposition delay-adapted O-REPS known})}

\begin{lemma}[$\textsc{Reg}$ Term]
    \label{lemma:Reg-o-reps-new-estimator}
    Under the good event,
    \[
        \sum_{k=1}^{K}\langle q^{k+d^k} - q^\star,\hat{c}^{k}\rangle
        =
        O \left( \frac{H \log (HSA)}{\eta} + \eta HSAK + \frac{\eta}{\gamma} H d_{max} \log \frac{H}{\delta} \right).
    \]
\end{lemma}

\begin{proof}
    Let $\tilde{q}_{h}^{k+1}(s,a)
    =q_{h}^{k}(s,a)e^{-\eta\sum_{j: j+d^j=k} \hat c_{h}^{j}(s,a)}$.
    Taking the log, 
    \[
        \eta\sum_{j: j+d^j=k} \hat c_{h}^{j}(s,a)
        = \log q_{h}^{k}(s,a)-\log\tilde{q}_{h}^{k+1}(s,a).
    \]
    Hence for any $q$ 
    \begin{align*}
        \eta\left\langle \sum_{j: j+d^j=k} \hat c_{h}^{j},q^{k}-q^\star \right\rangle  
        & =
        \left\langle \log q^{k} - \log\tilde{q}^{k+1},q^{k}-q^\star \right\rangle 
        =
        \KL{q^\star}{q^k} - \KL{q^\star}{\tilde q^{k+1}} + \KL{q^k}{\tilde q^{k+1}}
        \\
        & \leq 
        \KL{q^\star}{q^{k}} - \KL{q^\star}{q^{k+1}} - \KL{q^{k+1}}{\tilde{q}^{k+1}}+\KL{q^{k}}{\tilde{q}^{k+1}}
        \\
        & \leq 
        \KL{q^\star}{q^{k}} - \KL{q^\star}{q^{k+1}} + \KL{q^{k}}{\tilde{q}^{k+1}},
    \end{align*}
    where the second equality follows directly the definition of KL, the first inequality is by \cite[Lemma 1.2]{ziminonline}, and the second inequality is since the KL is non-negative.
    Now, the last term is bounded as follows: 
    \begin{align*}
        \KL{q^{k}}{\tilde{q}^{k+1}} 
        & \leq 
        \KL{q^{k}}{\tilde{q}^{k+1}} + \KL{\tilde{q}^{k+1}}{q^{k}}
        \\
        & =
        \sum_{h}\sum_{s,a} \tilde{q}_{h}^{k+1}(s,a) \log\frac{\tilde{q}_{h}^{k+1}(s,a)}{q_{h}^{k}(s,a)} 
        + \sum_{h}\sum_{s,a} q_{h}^{k}(s,a) \log\frac{q_{h}^{k}(s,a)}{\tilde{q}_{h}^{k+1}(s,a)}
        \\
        & =
        \langle q^{k}-\tilde{q}^{k+1},\log q^{k}-\log\tilde{q}^{k+1}\rangle
        =
        \eta\biggl\langle q^{k}-\tilde{q}^{k+1},\sum_{j: j+d^j=k} \hat c^j
        \biggr\rangle.
    \end{align*}
    We get that 
    \[
        \eta\left\langle \sum_{j: j+d^j=k} \hat c^j,q^{k}-q^\star \right\rangle 
        \leq 
        \KL{q^\star}{q^{k}}-\KL{q^\star}{q^{k+1}} 
        + \eta\left\langle q^{k}-\tilde{q}^{k+1},\sum_{j: j+d^j=k} \hat c^j\right\rangle.
    \]
    Summing over $k$ and dividing by $\eta$, we get 
    \begin{align*}
        \underbrace{
        \sum_{k=1}^{K}\sum_{j: j+d^j=k}\left\langle  \hat c^{j},q^{k}-q^\star \right\rangle}
        _{(*)}
        & \leq
        \frac{\KL{q^\star}{q^{1}}-\KL{q^\star}{q^{K+1}}}{\eta}
        +\sum_{k=1}^{K}\left\langle  q^{k}-\tilde{q}^{k+1},\sum_{j: j+d^j=k} \hat c^{j} \right\rangle \\
        & \leq
        \frac{\KL{q^\star}{q^{1}}}{\eta}
        +\sum_{k=1}^{K}\left\langle  q^{k}-\tilde{q}^{k+1},\sum_{j: j+d^j=k} \hat c^{j} \right\rangle \\
        & \leq
        \frac{2H\log(SA)}{\eta}
        + \underbrace{
        \sum_{k=1}^{K}\left\langle  q^{k}-\tilde{q}^{k+1},\sum_{j: j+d^j=k} \hat c^{j} \right\rangle}
        _{(**)},
    \end{align*}
    where the last inequality is a standard argument (see \cite{ziminonline,hazan2019introduction}).
    We now
    rearrange $(*)$ and $(**)$: 
    \begin{align*}
        (*)
        & =
        \sum_{k=1}^{K}\sum_{j=1}^{K}\ind\{ j+d^{j}=k\} \langle \hat c^{j},q^{k}-q^\star\rangle
        =
        \sum_{j=1}^{K}\sum_{k=1}^{K}\ind\{ j+d^{j}=k\} \langle \hat c^{j},q^{k}-q^\star\rangle 
        \\
        & =
        \sum_{j=1}^{K}\langle \hat c^{j},q^{j+d^{j}}-q^\star\rangle
        =
        \sum_{k=1}^{K}\langle \hat c^{k},q^{k+d^{k}}-q^\star\rangle.
    \end{align*}
    In a similar way, 
    \begin{align*}
        (**)  
        & =
        \sum_{k=1}^{K}\sum_{j: j+d^j=k}\langle q^{k}-\tilde{q}^{k+1},\hat c^{j}\rangle =
        \sum_{k=1}^{K}\sum_{j=1}^{K}\ind\{ j+d^j = k\} \langle q^{k}-\tilde{q}^{k+1},\hat c^{j}\rangle 
        \\
        & =
        \sum_{j=1}^{K}\sum_{k=1}^{K}\ind\{ j+d^j = k\} \langle q^{k}-\tilde{q}^{k+1},\hat c^{j}\rangle
        =
        \sum_{k=1}^{K}\langle q^{k+d^{k}}-\tilde{q}^{k+d^{k}+1},\hat c^{k}\rangle.
    \end{align*}
    This gives us, 
    \[
        \sum_{k=1}^{K}\langle \hat c^{k},q^{k+d^{k}}-q^\star\rangle \leq\frac{2H\log(S A)}{\eta}+\sum_{k=1}^{K}\langle q^{k+d^{k}}-\tilde{q}^{k+d^{k}+1},\hat c^{k}\rangle.
    \]
    It remains to bound the second term on the right hand side:
    \begin{align*}
        \sum_{k}\langle q^{k+d^{k}}-\tilde{q}^{k+d^{k}+1},\hat c^{k}\rangle 
        & =
        \sum_{k,h,s,a} \hat c_{h}^{k}(s,a)(q_{h}^{k+d^{k}}(s,a)-\tilde{q}_{h}^{k+d^{k}+1}(s,a))
        \\
        & =
        \sum_{k,h,s,a} \hat c_{h}^{k}(s,a) \left(
        q_{h}^{k+d^{k}}(s,a)-q_{h}^{k+d^{k}}(s,a)e^{-\eta\sum_{j : j+d^j = k+d^{k}} \hat c_{h}^{j}(s,a)} 
        \right)
        \\
        & =
        \sum_{k,h,s,a}q_{h}^{k+d^{k}}(s,a) \hat c^k_h(s,a) \left(
        1-e^{-\eta\sum_{j : j+d^j = k+d^{k}} \hat c_{h}^{j}(s,a)}
        \right)
        \\
        \tag{\ensuremath{1-e^{-x}\leq x}} 
        & \leq
        \eta\sum_{k,h,s,a}q_{h}^{k+d^{k}}(s,a) \hat c^k_h(s,a) \left(
        \sum_{j : j+d^j = k+d^{k}}\hat c_{h}^{j}(s,a)
        \right)
        \\
        & =
        \eta\sum_{k,h,s,a}q_{h}^{k+d^{k}}(s,a) \frac{\indevent{s^k_h = s , a^k_h = a} c^k_h(s,a)}{\max\{q^k_h(s,a) , q^{k + d^k}_h(s,a)\} + \gamma} \left(
        \sum_{j : j+d^j = k+d^{k}}\hat c_{h}^{j}(s,a)
        \right)
        \\
        & \leq
        \eta\sum_{k,h,s,a} \sum_{j : j+d^j = k+d^{k}}\hat c_{h}^{j}(s,a)
        =
        \eta\sum_{k,h,s,a} \sum_{j} \indevent{j+d^j = k+d^{k}} \hat c_{h}^{j}(s,a)
        \\
        & =
        \eta\sum_{j,h,s,a} \hat c_{h}^{j}(s,a) \sum_{k} \indevent{j+d^j = k+d^{k}}
        \le
        \eta\sum_{k,h,s,a} |\calF^{k+d^k}| \hat c^k_h(s,a).
    \end{align*}
    Finally, by event $E^d$,
    \begin{align*}
        \sum_{k,h,s,a} |\calF^{k+d^k}| \hat c^k_h(s,a)
        & =
        O \left( \sum_{k,h,s,a} |\calF^{k+d^k}| c^k_h(s,a) + \frac{H d_{max} \log \frac{H}{\delta}}{\gamma} \right)
        =
        O \left( \eta HSAK + \frac{H d_{max} \log \frac{H}{\delta}}{\gamma} \right).
        \qedhere
    \end{align*}
\end{proof}

\subsection{Bound on the Delay-caused Drift ($\textsc{Drift}$ in \cref{eq:regret decomposition delay-adapted O-REPS known})}

\begin{lemma}[$\textsc{Drift}$ term]
    \label{lemma:Drift-o-reps-new-estimator}
    Under the good event,
    \[
        \sum_{k=1}^{K} \langle q^{k}-q^{k+d^k},\hat{c}^{k}\rangle
        =
        O \left( \eta \sqrt{H^3 S A} (D + K) + \frac{\eta}{\gamma} H^{3/2} d_{max} \log \frac{H}{\delta} + \frac{H \log \frac{H}{\delta}}{\gamma} \right).
    \]
\end{lemma}

\begin{proof}
    By event $E^{\hat c}$ we have:
    \[
        \sum_{k=1}^{K}\langle \hat c^{k},q^{k}-q^{k+d^{k}}\rangle
        \le
        \sum_{k=1}^{K}\langle \hat c^{k},|q^{k}-q^{k+d^{k}}|\rangle
        =
        O \left( \sum_{k=1}^{K}\langle c^{k},|q^{k}-q^{k+d^{k}}|\rangle + \frac{H \log \frac{H}{\delta}}{\gamma} \right).
    \]
    Now, by Pinsker inequality and Jensen inequality:
    \begin{align*}
        \sum_{k=1}^{K}\langle c^{k},|q^{k}-q^{k+d^{k}}|\rangle
        & \leq
        \sum_{k=1}^{K}\sum_{j=k}^{k+d^{k}-1}\sum_{h,s,a} |q_{h}^{j}(s,a)-q_{h}^{j+1}(s,a)|
        =
        \sum_{k=1}^{K}\sum_{j=k}^{k+d^{k}-1} \sum_h \lVert q^j_h - q^{j+1}_h \rVert_1
        \\
        & \leq
        \sum_{k=1}^{K}\sum_{j=k}^{k+d^{k}-1}\sum_{h} \sqrt{2\KL{q_{h}^{j}}{q_{h}^{j+1}}}
        \leq
        \sum_{k=1}^{K}\sum_{j=k}^{k+d^{k}-1} \sqrt{2H\sum_{h}\KL{q_{h}^{j}}{q_{h}^{j+1}}}
        \\
        & \le
        \sum_{k=1}^{K}\sum_{j=k}^{k+d^{k}-1}\sqrt{H\sum_{h}\sum_{s,a}q_{h}^{j}(s,a)\Bigl(\eta\sum_{i : i+d^i = j}\hat c_{h}^{i}(s,a)\Bigr)^{2}}
        \\
        & \le
        \eta \sqrt{H} \sum_{k=1}^{K}\sum_{j=k}^{k+d^{k}-1} \sum_{i : i+d^i = j} \sum_{h,s,a} \sqrt{q^j_h(s,a)} \hat c^i_h(s,a),
    \end{align*}
    where the last inequality is by $\lVert x \rVert_2 \le \lVert x \rVert_1$, and the one before is by \cref{lemma:adjacent-o-mesure-KL-bound}.
    Finally, we rearrange as follows:
    \begin{align*}
        \sum_{k=1}^{K}\sum_{j=k}^{k+d^{k}-1} \sum_{i : i+d^i = j} & \sum_{h,s,a} \sqrt{q^j_h(s,a)} \hat c^i_h(s,a)
        =
        \sum_{k,j,i} \indevent{k \le j < k+d^k , i+d^i=j} \sum_{h,s,a} \sqrt{q^j_h(s,a)} \hat c^i_h(s,a)
        \\
        & =
        \sum_{k,j,i} \indevent{k \le j < k+d^k , i+d^i=j} \sum_{h,s,a} \sqrt{q^{i+d^i}_h(s,a)} \hat c^i_h(s,a)
        \\
        & =
        \sum_{k,i} \indevent{k \le i+d^i < k+d^k} \sum_{h,s,a} \sqrt{q^{i+d^i}_h(s,a)} \hat c^i_h(s,a)
        \\
        & =
        O \left( \sum_{k,i} \indevent{k \le i+d^i < k+d^k} \sum_{h,s,a} \sqrt{q^{i+d^i}_h(s,a)} c^i_h(s,a) + \frac{H d_{max} \log \frac{H}{\delta}}{\gamma} \right),
    \end{align*}
    where the last relation is by event $E^{sq}$.
    To finish the proof we use \cref{lemma:sum k<i<k+d^k = D}:
    \begin{align*}
        \sum_{k,i} \indevent{k \le i+d^i < k+d^k} & \sum_{h,s,a} \sqrt{q^{i+d^i}_h(s,a)} c^i_h(s,a)
        \le
        \sqrt{H S A} \sum_{k,i} \indevent{k \le i+d^i < k+d^k} \sqrt{\sum_{h,s,a} q^{i+d^i}_h(s,a)}
        \\
        & =
        H \sqrt{S A} \sum_{k,i} \indevent{k \le i+d^i < k+d^k}
        \le
        H \sqrt{S A} (D + K). \qedhere
    \end{align*}
\end{proof}

\subsection{Bound on the Bias of the Delay-adapted Estimator ($\textsc{Bias}_1$ in \cref{eq:regret decomposition delay-adapted O-REPS known})}

\begin{lemma}[$\textsc{Bias}_1$]
    \label{lemma:Bias-1-o-reps-new-estimator}
    Under the good event,
    \[
        \sum_{k=1}^{K}\langle c^{k}-\hat{c}^{k},q^{k}\rangle
        =
        O \left(
         H \sqrt{K \log \frac{1}{\delta}} +\gamma HSAK + \eta\sqrt{H^3 S A}(D + K) + \frac{\eta}{\gamma}H^{3/2}d_{max} \log \frac{H}{\delta} \right).
    \]
\end{lemma}

\begin{proof}
    Decompose $\textsc{Bias}_1$ as follows:
    \begin{align*}
    	\sum_{k=1}^{K}\langle c^{k}-\hat{c}^{k},q^{k}\rangle=\sum_{k=1}^{K}\langle c^{k}-\bbE\Big[\hat{c}^{k} \mid \tilde{\mathcal{H}}^{k+d^{k}}\Big],q^{k}\rangle+\sum_{k=1}^{K}\langle\bbE\Big[\hat{c}^{k} \mid \tilde{\mathcal{H}}^{k+d^{k}}\Big]-\hat{c}^{k},q^{k}\rangle.
    \end{align*}
    The second term is bounded by $O(H\sqrt{K\log \frac{1}{\delta}})$ under event $E^c$.
    The first term is bounded as follows:
    \begin{align*}
    	\sum_{k=1}^{K}\langle c^{k} & -\bbE[\hat{c}^{k}  \mid  \tilde{\mathcal{H}}^{k+d^{k}}],q^{k}\rangle 
    	=
    	\sum_{k,h,s,a} q_{h}^{k}(s,a)c_{h}^{k}(s,a) \left( 1- \frac{\bbE\Big[\indevent{s_{h}^{k}=s,a_{h}^{k}=a} \mid \tilde{\mathcal{H}}^{k+d^{k}}\Big]}{\max\{q_{h}^{k+d^k}(s,a) , q_{h}^{k}(s,a) \} +\gamma} \right) 
    	\\
        & =\sum_{k,h,s,a}q_{h}^{k}(s,a)c_{h}^{k}(s,a) \left( 1- \frac{q_{h}^{k}(s,a)}{\max\{q_{h}^{k+d^k}(s,a) , q_{h}^{k}(s,a) \} +\gamma} \right) 
        \\
        & =
        \sum_{k,h,s,a} \frac{q_{h}^{k}(s,a)}{\max\{q_{h}^{k+d^k}(s,a) , q_{h}^{k}(s,a) \} + \gamma}(\max\{q_{h}^{k+d^k}(s,a) , q_{h}^{k}(s,a) \} - q_{h}^{k}(s,a) +\gamma)
        \\
        & \leq
        \sum_{k,h,s,a}(\max\{q_{h}^{k+d^k}(s,a) , q_{h}^{k}(s,a) \} - q_{h}^{k}(s,a)) +\gamma HSAK
        \\
        &\le
        \sum_{k,h,s,a} |q_{h}^{k+d^{k}}(s,a) - q_{h}^{k}(s,a) |+\gamma HSAK
        \\
        & \leq 
        \eta \sqrt{H^3 S A}(D + K) +  \frac{\eta}{\gamma}H^{3/2}d_{\max} + \gamma HSAK.
    \end{align*}
    where the first equality uses the fact that $q^{k}$ and $q^{k+d^{k}}$
    are determined by the history $\tilde{\mathcal{H}}^{k+d^{k}}$, the
    second equality is since the $k$-th episode is not part of the history
    $\tilde{\mathcal{H}}^{k+d^{k}}$ as $k\notin\{j:j+d^{j}<k+d^{k}\}$,
    and the last inequality is as in the proof of \cref{lemma:Drift-o-reps-new-estimator}.
\end{proof}

\subsection{Auxiliary lemmas}

\begin{lemma}
    \label{lemma:adjacent-o-mesure-KL-bound}
    $
        \sum_{h} \KL{q_{h}^{k}}{q_{h}^{k+1}} 
        \leq
        \frac{\eta^2}{2} \sum_{h,s,a} q_{h}^{k}(s,a)(\sum_{j:j+d^j=k}\hat c_{h}^{j}(s,a))^2.
    $
\end{lemma}

\begin{proof}
    We start with expanding $\KL{q_{h}^{k}}{q_{h}^{k+1}}$ as follows:
    \begin{align}
        \nonumber
        \sum_h \KL{q_{h}^{k}}{q_{h}^{k+1}}
        & =
        \sum_{h,s,a} q^k_h(s,a) \log \frac{q^k_h(s,a)}{q^{k+1}_h(s,a)}
        =
        \sum_{h,s,a} q^k_h(s,a) \log \frac{Z^k_h(v^k) q^k_h(s,a)}{q^k_h(s,a) e^{B^k_h(s,a \mid v^k)}}
        \\
        \nonumber
        & =
        \sum_{h,s,a} q^k_h(s,a) \log Z^k_h(v^k) - \sum_{h,s,a} q^k_h(s,a) B^k_h(s,a \mid v^k)
        \\
        \label{eq::adjacent-o-mesure-KL-bound-1}
        & =
        \sum_{h} \log Z^k_h(v^k) - \sum_{h,s,a} q^k_h(s,a) B^k_h(s,a \mid v^k).
    \end{align}
    For the first term in \cref{eq::adjacent-o-mesure-KL-bound-1}, by definition of $v^k$ and $Z^k_h$:
    \begin{align*}
        \sum_{h} \log Z^k_h(v^k)
        & \le
        \sum_{h} \log Z^k_h(0)
        =
        \sum_h \log \left( \sum_{s,a} q^k_h(s,a) e^{B^k_h(s,a \mid 0)} \right)
        =
        \sum_h \log \left( \sum_{s,a} q^k_h(s,a) e^{-\eta \sum_{j:j+d^j=k} \hat c^j_h(s,a)} \right)
        \\
        & \le
        \sum_h \log \left( \sum_{s,a} q^k_h(s,a) \left( 1 -\eta \sum_{j:j+d^j=k} \hat c^j_h(s,a) + \frac{1}{2} \left( \eta \sum_{j:j+d^j=k} \hat c^j_h(s,a) \right)^2 \right) \right)
        \\
        & =
        \sum_h \log \left( 1 - \eta \sum_{s,a} \sum_{j:j+d^j=k} q^k_h(s,a) \hat c^j_h(s,a) + \frac{\eta^2}{2} \sum_{s,a} q^k_h(s,a) \left( \sum_{j:j+d^j=k} \hat c^j_h(s,a) \right)^2 \right)
        \\
        & \le
        \sum_h \left( - \eta \sum_{s,a} \sum_{j:j+d^j=k} q^k_h(s,a) \hat c^j_h(s,a) + \frac{\eta^2}{2} \sum_{s,a} q^k_h(s,a) \left( \sum_{j:j+d^j=k} \hat c^j_h(s,a) \right)^2 \right)
        \\
        & =
        - \eta \sum_{h,s,a} \sum_{j:j+d^j=k} q^k_h(s,a) \hat c^j_h(s,a) + \frac{\eta^2}{2} \sum_{h,s,a} q^k_h(s,a) \left( \sum_{j:j+d^j=k} \hat c^j_h(s,a) \right)^2,
    \end{align*}
    where the second inequality is by $e^s \le 1 + s + s^2/2$ for $s \le 0$, and the third inequality is by $\log (1+s) \le s$ for all $s$.
    The second term in \cref{eq::adjacent-o-mesure-KL-bound-1} can be written as follows:
    \begin{align*}
        \sum_{h,s,a} q^k_h(s,a) B^k_h(s,a \mid v^k)
        & =
        \sum_{h,s,a} q^k_h(s,a) v^k_h(s) - \eta \sum_{h,s,a} \sum_{j:j+d^j=k} q^k_h(s,a) \hat c^j_h(s,a)
        \\
        & \qquad -
        \sum_{h,s,a,s'} q^k_h(s,a) p_h(s' \mid s,a) v^k_{h+1}(s').
    \end{align*}
    So now, by occupancy measure constraints:
    \begin{align*}
        \sum_{h,s,a,s'} q^k_h(s,a) p_h(s' \mid s,a) v^k_{h+1}(s')
        & =
        \sum_{h,s'} v^k_{h+1}(s') \sum_{s,a} q^k_h(s,a) p_h(s' \mid s,a)
        =
        \sum_{h,s',a'} q^k_{h+1}(s',a') v^k_{h+1}(s'),
    \end{align*}
    which forms a telescopic sum, so by $v^k_0(s) = v^k_{H+1}(s) = 0$, we have:
    \begin{align*}
        \sum_{h,s,a} q^k_h(s,a) B^k_h(s,a \mid v^k)
        &=
        - \eta \sum_{h,s,a} \sum_{j:j+d^j=k} q^k_h(s,a) \hat c^j_h(s,a). 
        \qedhere
    \end{align*}
\end{proof}

\begin{lemma}[\cite{thune2019nonstochastic}]
    \label{lemma:sum k<i<k+d^k = D}
    $
        \sum_{k=1}^{K} \sum_{i=1}^{K} \indevent{k\leq i+d^{i}<k+d^{k}}
        \leq D+K.
    $
\end{lemma}

\begin{proof}
    \begin{align*}
        \sum_{k=1}^{K} \sum_{i=1}^{K} & \indevent{k\leq i+d^{i}<k+d^{k}}
        =
        \sum_{k=1}^{K} \sum_{i=1}^{K} \indevent{k\leq i+d^{i}<k+d^{k}}
        \\
        & =\sum_{k=1}^{K} \sum_{i=1}^{k} \indevent{k\leq i+d^{i}<k+d^{k}} +\sum_{k=1}^{K} \sum_{i=k+1}^{K} \indevent{k\leq i+d^{i}<k+d^{k}}
        \\
        & =\sum_{k=1}^{K} \sum_{i=1}^{k} \indevent{k\leq i+d^{i}}-\sum_{k=1}^{K} \sum_{i=1}^{k} \indevent{k\leq i+d^{i},i+d^{i}\geq k+d^{k}} +\sum_{k=1}^{K} \sum_{i=k+1}^{K} \indevent{k\leq i+d^{i}<k+d^{k}}
        \\
        & =\sum_{k=1}^{K} \sum_{i=1}^{K} \indevent{i\leq k\leq i+d^{i}}-\sum_{k=1}^{K} \sum_{i=1}^{k} \indevent{k+d^{k}\leq i+d^{i}} +\sum_{k=1}^{K} \sum_{i=1}^{K} \indevent{i\geq k+1,k\leq i+d^{i}<k+d^{k}}
        \\
        & =\sum_{i=1}^{K} \sum_{k=1}^{K}\underbrace{\indevent{i\leq k\leq i+d^{i}}}_{=d^{i} +1}-\sum_{k=1}^{K} \sum_{i=1}^{K} \indevent{i\leq k,k+d^{k}\leq i+d^{i}} +\sum_{k=1}^{K} \sum_{i=1}^{K} \indevent{i\geq k+1,k\leq i+d^{i}<k+d^{k}} \\
        & \leq D+K-\sum_{k=1}^{K} \sum_{i=1}^{K} \indevent{i\leq k,k+d^{k}\leq i+d^{i}} +\sum_{k=1}^{K} \sum_{i=1}^{K} \indevent{k\leq i,i+d^{i}\leq k+d^{k}}
        \leq D+K.
        \qedhere
    \end{align*}
\end{proof}

\newpage

\section{Delayed UOB-REPS with delay-adapted estimator}
\label{appendix:delay-adapted O-REPS unknown}

\begin{algorithm}[t]
    \caption{Delayed UOB-REPS with delay-adapted estimator} \label{alg:o-reps-new-estimator-unknown}
    \begin{algorithmic}
        \STATE \textbf{Input:} State space $\calS$, Action space $\calA$, Horizon $H$, Number of episodes $K$, Learning rate $\eta > 0$, Exploration parameter $\gamma > 0$, Confidence parameter $\delta > 0$.
        
        \STATE \textbf{Initialization:} Set $\pi^{1}_{h}(a\mid s)=\frac{1}{A}$, $q_{h}^{1}(s,a,s')=\frac{1}{S^2A} , m^1_h(s,a) = 0 , m^1_h(s,a,s')$ for every $(s,a,s',h) \in \calS \times \calA \times \calS \times [H]$.
        
        \FOR{$k=1,2,...,K$}
        
            \STATE Play episode $k$ with policy $\pi^k$ and observe delayed trajectory feedback $\{ (s^j_h,a^j_h) \}_{h=1}^H$ for all $j$ such that $j + d^j = k$
            
            
            
            
            \STATE Update confidence set $\calP^{k+1}$ by \cref{alg:update-cofidence-set delay}.
            
            \FOR{$j : j + d^j = k$}
            
                \STATE Observe feedback $\{ c^j_h(s^j_h,a^j_h) \}_{h=1}^H$.
                
                \STATE Compute  $u_h^j(s,a) = \max_{p' \in \calP^j}q_h^{p',\pi^j}(s,a) $ and $u_h^k(s,a) = \max_{p' \in \calP^k}q_h^{p',\pi^k}(s,a)$.
                
                \STATE Compute loss estimator $\hat c^j_h(s,a) = \frac{c^j_h(s,a) \indevent{s^j_h = s , a^j_h = a}}{\max\{u^j_h(s,a),u^k_h(s,a)\} + \gamma}$ for every $(s,a,h) \in \calS \times \calA \times [H]$.
            
            \ENDFOR
            
            \STATE Update occupancy measure: 
            \begin{align}
                q^{k+1} = \arg \min_{q \in \ocsetk{k+1}} \eta \left\langle q , \sum_{j:j+d^j=k} \hat c^j \right\rangle + \KL{q}{q^k},
                \label{OMD-update-unknown}
            \end{align}
            where $\KL{q}{q'} = \sum_{h,s,a,s'} q_h(s,a,s') \ln \frac{q_h(s,a,s')}{q'_h(s,a,s')} + q'_h(s,a,s') - q_h(s,a,s')$ and  $\ocsetk{k+1} = \{q^{\pi,p'} \mid \pi \in (\Delta_\calA)^{\calS \times [H]}, p'\in \calP^{k+1} \}$.
            
            \STATE Update policy: $\pi_{h}^{k+1}(a\mid s)
            =\frac{\sum_{s'} q_{h}^{k+1}(s,a,s')}{\sum_{a'} \sum_{s'} q_{h}^{k+1}(s,a',s')}$ for every $(s,a,h) \in \calS \times \calA \times [H]$.
        \ENDFOR
    \end{algorithmic}
\end{algorithm}

\begin{algorithm}[t]
	\caption{Update confidence set with delayed trajectory feedback} \label{alg:update-cofidence-set delay}
	\begin{algorithmic}
		\STATE \textbf{Input:} trajectories $\{ (s^j_h,a^j_h) \}_{h\in [H], j: j + d^j = k}$.
		
		\STATE Update visit counters: ${m^{k+1}_h(s,a) \gets m^k_h(s,a) + \sum_{j:j+d^j = k} \bbI\{s_h^j = s, a_h^j = a\} }$,\\ $ {m^{k+1}_h(s,a,s') \gets m^k_h(s, a, s') +  \sum_{j:j+d^j = k} \bbI\{s_h^j = s, a_h^j = a, s_{h+1}^j = s'\} }$ for every $h,s,s'$ and $a$.
		
		\STATE Compute empirical transitions function $\bar{p}^{k+1}$: $ \bar p^{k+1}_h(s' \mid s,a) = \frac{m^{k+1}_h(s,a,s')}{m^{k+1}_h(s,a) \vee 1} \qquad \forall (s,a,s',h) $.
		
		\STATE Define confidence sets $\calP^{k+1}$ such that $p' \in \calP^{k+1}$ if and only if, for every $(s,a,s',h)$, $p'$ ensures $\sum_{s'} p'_h(s'|s,a)=1$ and:
		\[ 
		\left| p'_{h}(s'| s,a)-\bar{p}_{h}^{k+1}(s'| s,a) \right|
		\leq 
		\sqrt{ \frac{16 \bar{p}_{h}^{k+1}(s' | s,a)  \log\frac{10 HSAK}{\delta}}{m_{h}^{k+1}(s,a) \vee 1}} + \frac{10 \log\frac{10 HSAK}{\delta}}{m_{h}^{k+1}(s,a) \vee 1}.
		\]

	\end{algorithmic}
\end{algorithm}

\begin{remark}  
\label{remark: delayed trajectory}
Note that the confidence set at time $k$ in \cref{alg:o-reps-new-estimator-unknown} is constructed using only the trajectories from rounds $j$ such that $j + d^j < k$ (a.k.a delayed trajectory feedback \cite{lancewicki2020learning}). The main reason for that is that our analysis requires that $\pi^k$ would be completely determined by the history from rounds $j$ such that $j+d^j < k$. This is specifically crucial for the analysis of $\textsc{Bias}_1$ (see \cref{lemma:Bias-1-o-reps-new-estimator}) and in some of the concentration bounds. This means that our algorithm performs under the weaker assumption of delayed trajectory feedback, but this also comes at the price of an additional additive term in the regret of order $H^3 S^2 A d_{max}$. In order to eliminate the dependency in $d_{max}$ one can use the skipping technique of \cite{thune2019nonstochastic}. In this case the regret would scale as $\tilde O(H^2 S \sqrt{A D} )$, under the worst case.
\end{remark}

Explicitly solving this optimization problem in \cref{OMD-update-unknown}, we get \cite{rosenberg2019online}:
\[
    q^{k+1}_h(s,a,s') = \frac{q^k_h(s,a,s') e^{B^k_h(s,a,s' \mid v^{\mu^k},e^{\mu^k,\beta^k})}}{Z^k_h(v^{\mu^k},e^{\mu^k,\beta^k})},
\]
for:
\begin{align*}
    B_{h}^{k}(s,a,s'\mid v,e)
    & =
    e_h(s,a,s')
    + v_h(s,a,s')
    - \eta\sum_{j : j+d^j=k} \hat c_{h}^{j}(s,a)
    - \sum_{s''} \bar p_{h}^{k}(s''\mid s,a)v_{h+1}(s,a,s'')
    \\
    v^{\mu}_h(s,a,s')
    & =
    \mu_{h}^{-}(s,a,s')-\mu_{h}^{+}(s,a,s')
    \\
    e^{\mu,\beta}_h(s,a,s')
    & =
    \beta_{h+1}(s') - \beta_{h}(s) + \sum_{s''} ( \mu_{h}^{-}(s,a,s'') + \mu_{h}^{+}(s,a,s'') ) r_{h}^{k}(s'' \mid s,a)
    \\
    r_{h}^{k}(s' \mid s,a) & = \sqrt{ \frac{16 \bar{p}_{h}^{k}(s' | s,a)  \log\frac{10 HSAK}{\delta}}{m_{h}^{k}(s,a) \vee 1}} + \frac{10 \log\frac{10 HSAK}{\delta}}{m_{h}^{k}(s,a) \vee 1}
    \\
    Z_{h}^{k}(v,e)
    & =
    \sum_{s,a,s'}q_{h}^{k}(s,a,s')e^{B_{h}^{k}(s,a,s'\mid v,e)}
    \\
    \mu^{k},\beta^{k}
    & =
    \arg\min_{\beta,\mu\geq0} \sum_{h=1}^{H} \log Z_{h}^{k}(v^{\mu},e^{\mu,\beta}).
\end{align*}

\begin{theorem}
    \label{thm:regret-o-reps-new-estimator-unknown}
     Running UOB-REPS with the delay-adapted estimator, $\eta = \gamma = \min \{ \sqrt{\frac{\log \frac{KHSA}{\delta}}{SAK}} , \sqrt{\frac{\log \frac{KHSA}{\delta}}{\sqrt{HSA} D}} \}$ guarantees, with probability $1 - \delta$,
    \begin{align*}
        \regret
        & =
        O \biggl( H^2 S\sqrt{AK \log \frac{KHSA}{\delta}} + (H S A)^{1/4} \cdot H \sqrt{D \log \frac{KHSA}{\delta}} 
        \\
        & \qquad \qquad \qquad + H^3 S^2 A d_{max} \log \frac{KHSA}{\delta} + H^3 S^3 A \log^3 \frac{KHSA}{\delta} \biggr).
    \end{align*}
\end{theorem}

\subsection{The good event}
Let $\tilde{\mathcal{H}}^{k}$ be the history of episodes $\{j:j+d^{j}<k\}$, $\epsilon_{h}^{k}(s'\mid s,a)=16\sqrt{\frac{p_{h}(s'\mid s,a)\logterm}{n_{h}^{k}(s,a)\vee1}}+\frac{200\logterm}{n_{h}^{k}(s,a)\vee1}$ and $\logterm = \log\frac{HSAK}{\delta}$.
Define the following events: 
\begin{align*}
    E^{p}
    & = 
    \left\{ \forall k,s',s,a,h:
    \left| p_{h}(s'\mid s,a)-\bar{p}_{h}^{k}(s'\mid s,a) \right|
    \leq 
    4\sqrt{ \frac{\bar{p}_{h}^{k}(s' \mid s,a)  \log\frac{10 HSAK}{\delta}}{m_{h}^{k}(s,a) \vee 1}} + 10 \frac{\log\frac{10 HSAK}{\delta}}{m_{h}^{k}(s,a) \vee 1}
    \right\}
    \\
        E^{on1}
    & =
    \left\{   \sum_{k,h,s,a} \Big( q^{\pi^k}_h(s,a) - \indevent{s^{k,v}_h=s,a^{k,v}_h=a} \Big) \min \{ 2 , \epsilon^k_h(s,a) \} \le 10 \sqrt{ K \log \frac{30 K H S A }{\delta}} \right\}
    \\
    E^{on2}
    & =
    \left\{  \sum_{k,h,s,a}  q_{h}^{\pi^{k}}(s,a) \epsilon_{h}^{k}(s,a) \le 2\sum_{k,h,s,a} \indevent{s_{h}^{k,v}=s,a_{h}^{k,v}=a} \epsilon_{h}^{k}(s,a)+100 H S \log^2 \frac{30 K H S A }{\delta} \right\} 
    \\
    E^{on3}
    & =
    \left\{  \sum_{k,s,a,h}  \frac{q_{h}^{\pi^{k}}(s,a)}{n_h^k(s,a)}  \le 2\sum_{k,s,a,h} \frac{ \indevent{s_{h}^{k,v}=s,a_{h}^{k,v}=a}}{n_h^k(s,a)} + H \log \frac{m}{\delta} \right\} 
    \\
    E^c 
    & = 
    \left\{ \sum_{k=1}^K \langle \bbE [ \hat c^k \mid \filt{k+d^k} ] - \hat c^k , q^k \rangle \le 4 H \sqrt{K \log \frac{10}{\delta}} \right\}
    \\
    E^{\hat c}
    & =
    \left\{ \sum_{k=1}^{K}\langle|q^{k}-q^{k+d^k}|,\hat{c}^{k}\rangle \le 4 \sum_{k=1}^{K}\langle|q^{k}-q^{k+d^k}|,c^{k}\rangle + \frac{40H \log \frac{10 H}{\delta}}{\gamma} \right\}
    \\
    E^d
    & =
    \left\{ \sum_{k,h,s,a} |\calF^{k+d^k}| \hat c^k_h(s,a) \le \sum_{k,h,s,a} |\calF^{k+d^k}| c^k_h(s,a) + \frac{10 H d_{max} \log \frac{10 H}{\delta}}{\gamma}  \right\}
    \\
    E^{sq}
    & =
    \left\{ \sum_{k=1}^{K} \sum_{i=1}^{K} \indevent{k\leq i+d^{i} < k+d^{k}} \sum_{h,s,a}\sqrt{q_{h}^{i+d^{i}}(s,a)} (\hat{c}_{h}^{i}(s,a) - 4 c_{h}^{i}(s,a) ) \le \frac{10 H d_{max} \log \frac{10 H}{\delta}}{\gamma}  \right\}
    \\
    E^\star
    & =
    \left\{ \sum_{k=1}^K \langle \hat c^k - c^k , q^\star \rangle \le \frac{H  \log \frac{10 H S A}{\delta}}{\gamma} \right\}
\end{align*}

The good event is the intersection of the above events. 
The following lemma establishes that the good event holds with high probability. 

\begin{lemma}[The Good Event]
    \label{lemma:good-event-o-reps-new-estimator-unknown}
    Let $\bbG =E^p\cap E^{on1}\cap E^{on2}\cap E^{on3} \cap E^c \cap E^{\hat c} \cap E^d \cap E^{sq} \cap E^\star$ be the good event. 
    It holds that $\Pr [ \bbG ] \geq 1-\delta$.
\end{lemma}

\begin{proof}
    Similar to the proof of \cref{lemma:good-event-o-reps-new-estimator}.
    Events $E^p, E^{on1}, E^{on2}$ and $E^{on3}$ are standard (see, e.g., \cite{jin2019learning,pmlr-v162-lancewicki22a}).
\end{proof}

\subsection{Proof of the Main Theorem}

\begin{proof}[Proof of \cref{thm:regret-o-reps-new-estimator-unknown}]
    By \cref{lemma:good-event-o-reps-new-estimator-unknown}, the good event holds with probability $1 - \delta$.
    We now analyze the regret under the assumption that the good event holds.
    We decompose the regret as follows:
    \begin{align}
        \nonumber
        & \regret
        =
        \sum_{k=1}^K \langle q^{\pi^k} - q , c^k \rangle
        \\
        & \,\,= 
        \underbrace{\sum_{k=1}^{K}\langle q^{\pi^k} - q^k,c^{k}\rangle }_{\textsc{Est}}
        +
        \underbrace{\sum_{k=1}^{K}\langle q^{k},c^{k}-\hat{c}^{k}\rangle }_{\textsc{Bias}_1}
        + 
        \underbrace{\sum_{k=1}^{K}\langle q^\star,\hat{c}^{k}-c^{k}\rangle}_{\textsc{Bias}_2}
        +
        \underbrace{\sum_{k=1}^{K}\langle q^{k} - q^{k+d^k},\hat{c}^{k}\rangle}_{\textsc{Drift}}
        +
        \underbrace{\sum_{k=1}^{K}\langle q^{k+d^k} - q^\star,\hat{c}^{k}\rangle}_{\textsc{Reg}}.
        \label{eq:regret decomposition delay-adapted O-REPS unknown}
    \end{align}
    $\textsc{Bias}_2$ is bounded under event $E^\star$ by $O(\frac{H \logterm}{\gamma})$, 
    $\textsc{Est}$ is bounded in \cref{lemma:Est-uob-reps} by $ O (H^2 S \sqrt{A K \logterm} + H^2 S^2 A \logterm^2 + H^2 S A d_{max})$, 
    $\textsc{Reg}$ is bounded in \cref{lemma:Reg-uob-reps} by $O (\frac{H \logterm}{\eta} + \eta HSAK + \frac{\eta}{\gamma} H d_{max} \logterm)$, 
    $\textsc{Drift}$ is bounded in \cref{lemma:Drift-uob-reps} by $O (\eta\sqrt{H^3 S A} (D + K) + \frac{\eta}{\gamma} H^{3/2} d_{max} \logterm + \frac{H \logterm}{\gamma})$, and $\textsc{Bias}_1$ is bounded in \cref{lemma:Bias-1-uob-reps} by $O (H^2 S \sqrt{A K \logterm} + H^3 S^3 A \logterm^3 +\gamma HSAK + \eta\sqrt{H^3 S A}(D + K) + \frac{\eta}{\gamma}H^{3/2}d_{max} \logterm) + H^3 S^2 A d_{max}$.
    Putting everything together:
    \begin{align*}
        \regret
        & =
        O \Big( H^2 S \sqrt{A K \logterm} + H^3 S^3 A \logterm^3 + (\eta + \gamma) H S A K 
        \\
        & \qquad\qquad\qquad\qquad\qquad 
        + (\frac{1}{\eta} + \frac{1}{\gamma}) H \logterm + \eta \sqrt{H^3 S A} (D+K) + \frac{\eta}{\gamma} H^{3/2} d_{max} \logterm 
        + H^3 S^2 A d_{max}
        \Big) ,
    \end{align*}
    and plugging in the definitions of $\eta$ and $\gamma$ finishes the proof.
\end{proof}

\subsection{Bound on the Transition Estimation Error ($\textsc{Est}$ in \cref{eq:regret decomposition delay-adapted O-REPS unknown})}

\begin{lemma}[$\textsc{Est}$ Term]
    \label{lemma:Est-uob-reps}
    Under the good event,
    \[
        \sum_{k=1}^{K}\langle q^{\pi^k} - q^k,c^{k}\rangle
        =
        O \left( H^2 S \sqrt{A K \logterm} + H^2 S^2 A \logterm^2 + H^2 S A d_{max}\right).
    \]
\end{lemma}

\begin{proof}
    Let $q^k = q^{\pi^k,p^k}$.
    By the value difference lemma \cite{shani2020optimistic}:
    \begin{align*}
        \sum_{k=1}^{K}\langle q^{\pi^k} - q^k,c^{k}\rangle
        & =
        \sum_{k,h,s,a} q^{\pi^k}_h(s,a) \sum_{s'} \left( p^k_h(s' \mid s,a) - p_h(s' \mid s,a) \right) V^{\pi^k,p^k}_{h+1}(s')
        \\
        & \le
        H \sum_{k,h,s,a} q^{\pi^k}_h(s,a) \lVert p^k_h(\cdot \mid s,a) - p_h(\cdot \mid s,a) \rVert_1
        \\
        & =
        O ( H^2 S \sqrt{A K \logterm} + H^2 S^2 A \logterm^2 + H^2 S A d_{max} ),
    \end{align*}
    where the second inequality is by event $E^p$ and the last  is by \cite[lemma 5]{lancewicki2020learning}.
\end{proof}

\subsection{Bound on the Regret with respect to the Loss Estimators and Future Policies ($\textsc{Reg}$ in \cref{eq:regret decomposition delay-adapted O-REPS unknown})}

\begin{lemma}[$\textsc{Reg}$ Term]
    \label{lemma:Reg-uob-reps}
    Under the good event,
    \[
        \sum_{k=1}^{K}\langle q^{k+d^k} - q^\star,\hat{c}^{k}\rangle
        =
        O \left( \frac{H \logterm}{\eta} + \eta HSAK + \frac{\eta}{\gamma} H d_{max} \logterm \right).
    \]
\end{lemma}

\begin{proof}
    Let $\tilde{q}_{h}^{k+1}(s,a,s')
    =q_{h}^{k}(s,a,s')e^{-\eta\sum_{j: j+d^j=k} \hat c_{h}^{j}(s,a)}$.
    Taking the log, 
    \[
        \eta\sum_{j: j+d^j=k} \hat c_{h}^{j}(s,a)
        = \log q_{h}^{k}(s,a,s')-\log\tilde{q}_{h}^{k+1}(s,a,s').
    \]
    Hence,
    \begin{align*}
        \eta\left\langle \sum_{j: j+d^j=k} \hat c_{h}^{j},q^{k}-q \right\rangle  
        & =
        \left\langle \log q^{k} - \log\tilde{q}^{k+1},q^{k}-q^\star \right\rangle 
        =
        \KL{q^\star}{q^k} - \KL{q^\star}{\tilde q^{k+1}} + \KL{q^k}{\tilde q^{k+1}}
        \\
        & \leq 
        \KL{q^\star}{q^{k}} - \KL{q^\star}{q^{k+1}} - \KL{q^{k+1}}{\tilde{q}^{k+1}}+\KL{q^{k}}{\tilde{q}^{k+1}}
        \\
        & \leq 
        \KL{q^\star}{q^{k}} - \KL{q^\star}{q^{k+1}} + \KL{q^{k}}{\tilde{q}^{k+1}},
    \end{align*}
    where the second equality follows directly the definition of KL, the first inequality is by \cite[Lemma 1.2]{ziminonline}, and the second inequality is since the KL is non-negative.
    Now, the last term is bounded as follows: 
    \begin{align*}
        \KL{q^{k}}{\tilde{q}^{k+1}} 
        & \leq 
        \KL{q^{k}}{\tilde{q}^{k+1}} + \KL{\tilde{q}^{k+1}}{q^{k}}
        \\
        & =
        \sum_{h}\sum_{s,a,s'} \tilde{q}_{h}^{k+1}(s,a,s') \log\frac{\tilde{q}_{h}^{k+1}(s,a,s')}{q_{h}^{k}(s,a,s')} 
        + \sum_{h}\sum_{s,a,s'} q_{h}^{k}(s,a,s') \log\frac{q_{h}^{k}(s,a,s')}{\tilde{q}_{h}^{k+1}(s,a,s')}
        \\
        & =
        \langle q^{k}-\tilde{q}^{k+1},\log q^{k}-\log\tilde{q}^{k+1}\rangle
        =
        \eta\biggl\langle q^{k}-\tilde{q}^{k+1},\sum_{j: j+d^j=k} \hat c^j
        \biggr\rangle.
    \end{align*}
    We get that 
    \[
        \eta\left\langle \sum_{j: j+d^j=k} \hat c^j,q^{k}-q^\star \right\rangle 
        \leq 
        \KL{q^\star}{q^{k}}-\KL{q^\star}{q^{k+1}} 
        + \eta\left\langle q^{k}-\tilde{q}^{k+1},\sum_{j: j+d^j=k} \hat c^j\right\rangle.
    \]
    Summing over $k$ and dividing by $\eta$, we get 
    \begin{align*}
        \underbrace{
        \sum_{k=1}^{K}\sum_{j: j+d^j=k}\left\langle  \hat c^{j},q^{k}-q^\star \right\rangle}
        _{(*)}
        & \leq
        \frac{\KL{q^\star}{q^{1}}-\KL{q^\star}{q^{K+1}}}{\eta}
        +\sum_{k=1}^{K}\left\langle  q^{k}-\tilde{q}^{k+1},\sum_{j: j+d^j=k} \hat c^{j} \right\rangle \\
        & \leq
        \frac{\KL{q^\star}{q^{1}}}{\eta}
        +\sum_{k=1}^{K}\left\langle  q^{k}-\tilde{q}^{k+1},\sum_{j: j+d^j=k} \hat c^{j} \right\rangle 
        \\
        & \leq
        \frac{4H\log(SA)}{\eta}
        + \underbrace{
        \sum_{k=1}^{K}\left\langle  q^{k}-\tilde{q}^{k+1},\sum_{j: j+d^j=k} \hat c^{j} \right\rangle}
        _{(**)},
    \end{align*}
    where the last inequality is a standard argument (see \cite{ziminonline,hazan2019introduction}).
    We now
    rearrange $(*)$ and $(**)$: 
    \begin{align*}
        (*)
        & =
        \sum_{k=1}^{K}\sum_{j=1}^{K}\ind\{ j+d^{j}=k\} \langle \hat c^{j},q^{k}-q^\star\rangle
        =
        \sum_{j=1}^{K}\sum_{k=1}^{K}\ind\{ j+d^{j}=k\} \langle \hat c^{j},q^{k}-q^\star \rangle 
        \\
        & =
        \sum_{j=1}^{K}\langle \hat c^{j},q^{j+d^{j}}-q^\star \rangle
        =
        \sum_{k=1}^{K}\langle \hat c^{k},q^{k+d^{k}}-q^\star \rangle.
    \end{align*}
    In a similar way, 
    \begin{align*}
        (**)  
        & =
        \sum_{k=1}^{K}\sum_{j: j+d^j=k}\langle q^{k}-\tilde{q}^{k+1},\hat c^{j}\rangle =
        \sum_{k=1}^{K}\sum_{j=1}^{K}\ind\{ j+d^j = k\} \langle q^{k}-\tilde{q}^{k+1},\hat c^{j}\rangle 
        \\
        & =
        \sum_{j=1}^{K}\sum_{k=1}^{K}\ind\{ j+d^j = k\} \langle q^{k}-\tilde{q}^{k+1},\hat c^{j}\rangle
        =
        \sum_{k=1}^{K}\langle q^{k+d^{k}}-\tilde{q}^{k+d^{k}+1},\hat c^{k}\rangle.
    \end{align*}
    This gives us, 
    \[
        \sum_{k=1}^{K}\langle \hat c^{k},q^{k+d^{k}}-q^\star \rangle \leq\frac{4H\log(S A)}{\eta}+\sum_{k=1}^{K}\langle q^{k+d^{k}}-\tilde{q}^{k+d^{k}+1},\hat c^{k}\rangle.
    \]
    It remains to bound the second term on the right hand side:
    \begin{align*}
        \sum_{k}\langle q^{k+d^{k}}-\tilde{q}^{k+d^{k}+1},\hat c^{k}\rangle 
        & =
        \sum_{k,h,s,a,s'} \hat c_{h}^{k}(s,a)(q_{h}^{k+d^{k}}(s,a,s')-\tilde{q}_{h}^{k+d^{k}+1}(s,a,s'))
        \\
        & =
        \sum_{k,h,s,a,s'} \hat c_{h}^{k}(s,a) \left(
        q_{h}^{k+d^{k}}(s,a,s')-q_{h}^{k+d^{k}}(s,a,s')e^{-\eta\sum_{j : j+d^j = k+d^{k}} \hat c_{h}^{j}(s,a)} 
        \right)
        \\
        & =
        \sum_{k,h,s,a,s'}q_{h}^{k+d^{k}}(s,a,s') \hat c^k_h(s,a) \left(
        1-e^{-\eta\sum_{j : j+d^j = k+d^{k}} \hat c_{h}^{j}(s,a)}
        \right)
        \\
        \tag{\ensuremath{1-e^{-x}\leq x}} 
        & \leq
        \eta\sum_{k,h,s,a}q_{h}^{k+d^{k}}(s,a) \hat c^k_h(s,a) \left(
        \sum_{j : j+d^j = k+d^{k}}\hat c_{h}^{j}(s,a)
        \right)
        \\
        & =
        \eta\sum_{k,h,s,a}q_{h}^{k+d^{k}}(s,a) \frac{\indevent{s^k_h = s , a^k_h = a} c^k_h(s,a)}{\max\{u^k_h(s,a),u^{k + d^k}_h(s,a)\} + \gamma} \left(
        \sum_{j : j+d^j = k+d^{k}}\hat c_{h}^{j}(s,a)
        \right)
        \\
        & \leq
        \eta\sum_{k,h,s,a} \sum_{j : j+d^j = k+d^{k}}\hat c_{h}^{j}(s,a)
        =
        \eta\sum_{k,h,s,a} \sum_{j} \indevent{j+d^j = k+d^{k}} \hat c_{h}^{j}(s,a)
        \\
        & =
        \eta\sum_{j,h,s,a} \hat c_{h}^{j}(s,a) \sum_{k} \indevent{j+d^j = k+d^{k}}
        \le
        \eta\sum_{k,h,s,a} |\calF^{k+d^k}| \hat c^k_h(s,a),
    \end{align*}
    where the second inequality is since $u^{k + d^k}_h(s,a) \geq q^{k + d^k}_h(s,a)$ under the good event. Finally, by event $E^d$,
    \begin{align*}
        \sum_{k,h,s,a} |\calF^{k+d^k}| \hat c^k_h(s,a)
        & =
        O \left( \sum_{k,h,s,a} |\calF^{k+d^k}| c^k_h(s,a) + \frac{H d_{max} \logterm}{\gamma} \right)
        =
        O \left( \eta HSAK + \frac{H d_{max} \logterm}{\gamma} \right).
        \qedhere
    \end{align*}
\end{proof}

\subsection{Bound on the Delay-caused Drift ($\textsc{Drift}$ in \cref{eq:regret decomposition delay-adapted O-REPS unknown})}

\begin{lemma}[$\textsc{Drift}$ term]
    \label{lemma:Drift-uob-reps}
    Under the good event,
    \[
        \sum_{k=1}^{K} \langle q^{k}-q^{k+d^k},\hat{c}^{k}\rangle
        =
        O \left( \eta \sqrt{H^3 S A} (D + K) + \frac{\eta}{\gamma} H^{3/2} d_{max} \logterm + \frac{H \logterm}{\gamma} \right).
    \]
\end{lemma}

\begin{proof}
    By event $E^{\hat c}$ we have:
    \[
        \sum_{k=1}^{K}\langle \hat c^{k},q^{k}-q^{k+d^{k}}\rangle
        \le
        \sum_{k=1}^{K}\langle \hat c^{k},|q^{k}-q^{k+d^{k}}|\rangle
        =
        O \left( \sum_{k=1}^{K}\langle c^{k},|q^{k}-q^{k+d^{k}}|\rangle + \frac{H \logterm}{\gamma} \right).
    \]
    Now, by Pinsker inequality and Jensen inequality:
    \begin{align*}
        \sum_{k=1}^{K}\langle c^{k},|q^{k}-q^{k+d^{k}}|\rangle
        & \leq
        \sum_{k=1}^{K}\sum_{j=k}^{k+d^{k}-1}\sum_{h,s,a,s'} |q_{h}^{j}(s,a,s')-q_{h}^{j+1}(s,a,s')|
        =
        \sum_{k=1}^{K}\sum_{j=k}^{k+d^{k}-1} \sum_h \lVert q^j_h - q^{j+1}_h \rVert_1
        \\
        & \leq
        \sum_{k=1}^{K}\sum_{j=k}^{k+d^{k}-1}\sum_{h} \sqrt{2\KL{q_{h}^{j}}{q_{h}^{j+1}}}
        \leq
        \sum_{k=1}^{K}\sum_{j=k}^{k+d^{k}-1} \sqrt{2H\sum_{h}\KL{q_{h}^{j}}{q_{h}^{j+1}}}
        \\
        & \le
        \sum_{k=1}^{K}\sum_{j=k}^{k+d^{k}-1}\sqrt{H\sum_{h}\sum_{s,a,s'}q_{h}^{j}(s,a,s')\Bigl(\eta\sum_{i : i+d^i = j}\hat c_{h}^{i}(s,a)\Bigr)^{2}}
        \\
        & \le
        \eta \sqrt{H} \sum_{k=1}^{K}\sum_{j=k}^{k+d^{k}-1} \sum_{i : i+d^i = j} \sum_{h,s,a} \sqrt{q^j_h(s,a)} \hat c^i_h(s,a),
    \end{align*}
    where the last inequality is by $\lVert x \rVert_2 \le \lVert x \rVert_1$, and the one before is by \cref{lemma:adjacent-o-mesure-KL-bound-unknown-p}.
    Finally, we rearrange as follows:
    \begin{align*}
        \sum_{k=1}^{K}\sum_{j=k}^{k+d^{k}-1} \sum_{i : i+d^i = j} & \sum_{h,s,a} \sqrt{q^j_h(s,a)} \hat c^i_h(s,a)
        =
        \sum_{k,j,i} \indevent{k \le j < k+d^k , i+d^i=j} \sum_{h,s,a} \sqrt{q^j_h(s,a)} \hat c^i_h(s,a)
        \\
        & =
        \sum_{k,j,i} \indevent{k \le j < k+d^k , i+d^i=j} \sum_{h,s,a} \sqrt{q^{i+d^i}_h(s,a)} \hat c^i_h(s,a)
        \\
        & =
        \sum_{k,i} \indevent{k \le i+d^i < k+d^k} \sum_{h,s,a} \sqrt{q^{i+d^i}_h(s,a)} \hat c^i_h(s,a)
        \\
        & =
        O \left( \sum_{k,i} \indevent{k \le i+d^i < k+d^k} \sum_{h,s,a} \sqrt{q^{i+d^i}_h(s,a)} c^i_h(s,a) + \frac{H d_{max} \logterm}{\gamma} \right),
    \end{align*}
    where the last relation is by event $E^{sq}$.
    To finish the proof we use \cref{lemma:sum k<i<k+d^k = D}:
    \begin{align*}
        \sum_{k,i} \indevent{k \le i+d^i < k+d^k} & \sum_{h,s,a} \sqrt{q^{i+d^i}_h(s,a)} c^i_h(s,a)
        \le
        \sqrt{H S A} \sum_{k,i} \indevent{k \le i+d^i < k+d^k} \sqrt{\sum_{h,s,a} q^{i+d^i}_h(s,a)}
        \\
        & =
        H \sqrt{S A} \sum_{k,i} \indevent{k \le i+d^i < k+d^k}
        \le
        H \sqrt{S A} (D + K). \qedhere
    \end{align*}
\end{proof}

\subsection{Bound on the Bias of the Delay-adapted Estimator ($\textsc{Bias}_1$ in \cref{eq:regret decomposition delay-adapted O-REPS unknown})}

\begin{lemma}[$\textsc{Bias}_1$ Term]
    \label{lemma:Bias-1-uob-reps}
    Under the good event,
    \[
        \sum_{k=1}^{K}\langle c^{k}-\hat{c}^{k},q^{k}\rangle
        =
        O \left(
        H^2 S \sqrt{A K \logterm} + H^3 S^3 A \logterm^3 +\gamma HSAK + \eta\sqrt{H^3 S A}(D + K) + \frac{\eta}{\gamma}H^{3/2}d_{max} \logterm
        + H^3 S^2 A d_{max}
        \right).
    \]
\end{lemma}

\begin{proof}
    Decompose $\textsc{Bias}_1$ as follows:
    \begin{align*}
    	\sum_{k=1}^{K}\langle c^{k}-\hat{c}^{k},q^{k}\rangle=\sum_{k=1}^{K}\langle c^{k}-\bbE\Big[\hat{c}^{k} \mid \tilde{\mathcal{H}}^{k+d^{k}}\Big],q^{k}\rangle+\sum_{k=1}^{K}\langle\bbE\Big[\hat{c}^{k} \mid \tilde{\mathcal{H}}^{k+d^{k}}\Big]-\hat{c}^{k},q^{k}\rangle.
    \end{align*}
    The second term is bounded by $O(H\sqrt{K\logterm})$ under event $E^c$.
    The first term is bounded as follows:
    \begin{align*}
    	\sum_{k=1}^{K}\langle c^{k} & -\bbE[\hat{c}^{k}  \mid  \tilde{\mathcal{H}}^{k+d^{k}}],q^{k}\rangle 
    	=
    	\sum_{k,h,s,a,s'} q_{h}^{k}(s,a,s')c_{h}^{k}(s,a) \left( 1- \frac{\bbE\Big[\indevent{s_{h}^{k}=s,a_{h}^{k}=a} \mid \tilde{\mathcal{H}}^{k+d^{k}}\Big]}{\max\{u_{h}^{k}(s,a) , u_{h}^{k+d^k}(s,a)\} +\gamma} \right) 
    	\\
        & =\sum_{k,h,s,a}q_{h}^{k}(s,a)c_{h}^{k}(s,a) \left( 1- \frac{q_{h}^{\pi^k}(s,a)}{\max\{u_{h}^{k}(s,a) , u_{h}^{k+d^k}(s,a)\} +\gamma} \right) 
        \\
        & =
        \sum_{k,h,s,a} \frac{q_{h}^{k}(s,a)}{\max\{u_{h}^{k}(s,a) , u_{h}^{k+d^k}(s,a)\} + \gamma}(\max\{u_{h}^{k}(s,a) , u_{h}^{k+d^k}(s,a)\} - q_{h}^{\pi^k}(s,a) +\gamma)
        \\
        & \leq
        \sum_{k,h,s,a}(\max\{u_{h}^{k}(s,a) , u_{h}^{k+d^k}(s,a)\} - q_{h}^{\pi^k}(s,a)) +\gamma HSAK
        \\
        & \leq
        \sum_{k,h,s,a} | \max\{u_{h}^{k}(s,a) , u_{h}^{k+d^k}(s,a)\} - q_{h}^{\pi^k}(s,a) | +\gamma HSAK.
    \end{align*}
    where the first equality uses the fact that $u^{k}$ and $u^{k+d^k}$
    is determined by the history $\tilde{\mathcal{H}}^{k+d^{k}}$, the
    second equality is since the $k$-th episode is not part of the history
    $\tilde{\mathcal{H}}^{k+d^{k}}$ as $k\notin\{j:j+d^{j}<k+d^{k}\}$, and the first inequality is since $u_{h}^{k}(s,a)\geq q_{h}^{k}(s,a)$ under the good event.
    Finally, we bound:
    \begin{align*}
        \sum_{k,h,s,a} | \max\{u_{h}^{k}(s,a) , u_{h}^{k+d^k}(s,a)\} - q_{h}^{\pi^k}(s,a) |
        & \le
        \sum_{k,h,s,a} |  u_h^{k}(s,a) - q_{h}^{\pi^k}(s,a) | + \sum_{k,h,s,a} | u_h^{k+d^k}(s,a) - q_{h}^{\pi^k}(s,a) |.
    \end{align*}
    The first term is bounded in \cref{lemma:Jin-final delay} by $O(H^2 S \sqrt{A K \logterm} + H^3 S^3 A \logterm^3 + H^3 S^2 A d_{max})$, and for the second term:
    \begin{align*}
        \sum_{k,h,s,a} | u_h^{k+d^k}(s,a) - q_{h}^{\pi^k}(s,a) |
        & \le
        \sum_{k,h,s,a} | u_h^{k+d^k}(s,a) - q_{h}^{\pi^{k+d^k}}(s,a) | 
        \\
        & \qquad + 
        \sum_{k,h,s,a} | q_{h}^{\pi^{k+d^k}}(s,a) - q_{h}^{\pi^k}(s,a)|,
    \end{align*}
    where again the first term is bounded in \cref{lemma:Jin-final delay}.
    Finally,
    \begin{align*}
        \sum_{k,h,s,a} | q_{h}^{\pi^{k+d^k}}(s,a) - q_{h}^{\pi^k}(s,a)|
        & \le
        \sum_{k,h,s,a} | q_{h}^{\pi^{k+d^k}}(s,a) - q_{h}^{k+d^k}(s,a)| + \sum_{k,h,s,a} | q_{h}^k(s,a) - q_{h}^{\pi^k}(s,a)| 
        \\
        & \qquad + 
        \sum_{k,h,s,a} | q_{h}^{k+d^k}(s,a) - q_{h}^k(s,a)|,
    \end{align*}
    where the first two terms are bounded similarly to \cref{lemma:Est-uob-reps} and the last term is bounded similarly to \cref{lemma:Drift-uob-reps}.
\end{proof}

\subsection{Auxiliary lemmas}

\begin{lemma}
    \label{lemma:adjacent-o-mesure-KL-bound-unknown-p}
    $
        \sum_{h} \KL{q_{h}^{k}}{q_{h}^{k+1}} 
        \leq
        \frac{\eta^2}{2} \sum_{h,s,a,s'} q_{h}^{k}(s,a,s')(\sum_{j:j+d^j=k}\hat c_{h}^{j}(s,a))^2.
    $
\end{lemma}

\begin{proof}
    We start with expanding $\KL{q_{h}^{k}}{q_{h}^{k+1}}$ as follows:
    \begin{align*}
        \sum_{h}\KL{q_{h}^{k}}{q_{h}^{k+1}}
        & =
        \sum_{h}\sum_{s,a,s'}q_{h}^{k}(s,a,s') \log\frac{q_{h}^{k}(s,a,s')}{q_{h}^{k+1}(s,a,s')}
        =
        \sum_{h}\sum_{s,a,s'} q_{h}^{k}(s,a,s')\log\frac{Z_{h}^{k}(v^{\mu^{k}},e^{\mu^{k},\beta^{k}})} {e^{B_{h}^{k}(s,a,s' \mid v^{\mu^{k}},e^{\mu^{k},\beta^{k}})}}
        \\
        & =
        \underbrace{
        \sum_{h} \log Z_{h}^{k}(v^{\mu^{k}},e^{\mu^{k},\beta^{k}})}
        _{(A)} 
        -
        \underbrace{
        \sum_{h}\sum_{s,a,s'} q_{h}^{k}(s,a,s')B_{h}^{k}(s,a,s' \mid v^{\mu^{k}},e^{\mu^{k},\beta^{k}})}
        _{(B)}.
    \end{align*}
    By definition of $\mu^{k},\beta^{k}$, term $(A)$ can be bounded by
    \begin{align*}
        (A) 
        & \leq
        \sum_{h}\log Z_{h}^{k}(0,0)
        =
        \sum_{h}\log(\sum_{s,a,s'}q_{h}^{k}(s,a,s')e^{B_{h}^{k}(s,a,s'\mid0,0)})
        =
        \sum_{h}\log(\sum_{s,a,s'}q_{h}^{k}(s,a,s')e^{-\eta\sum_{j : j+d^j=k} \hat c_{h}^{j}(s,a)})\\
        & \leq
        \sum_{h} \log\left(\sum_{s,a,s'}q_{h}^{k}(s,a,s')\left(1-\eta\sum_{j : j+d^j=k} \hat c_{h}^{j}(s,a)+\frac{(\eta\sum_{j : j+d^j=k} \hat c_{h}^{j}(s,a))^{2}}{2}\right)\right)\\
        & =
        \sum_{h}\log\left(1-\eta\sum_{s,a,s'}\sum_{j : j+d^j=k}q_{h}^{k}(s,a,s') \hat c_{h}^{j}(s,a)+\sum_{s,a,s'}q_{h}^{k}(s,a,s')\frac{(\eta\sum_{j : j+d^j=k}\hat c_{h}^{j}(s,a))^{2}}{2}\right)
        \\
        & \leq
        - \eta\sum_{h}\sum_{s,a,s'}\sum_{j : j+d^j=k} q_{h}^{k}(s,a,s') \hat c_{h}^{j}(s,a)+\sum_{h}\sum_{s,a,s'}q_{h}^{k}(s,a,s')\frac{(\eta\sum_{j : j+d^j=k} \hat c_{h}^{j}(s,a))^{2}}{2},
    \end{align*}
    where the second inequality is by $e^s \le 1 + s + s^2/2$ for $s \le 0$, and the third inequality is by $\log (1+s) \le s$ for all $s$.
    Term $(B)$ can be rewritten as
    \begin{align*}
        (B)
        & =
        \sum_{h}\sum_{s,a,s'}q_{h}^{k}(s,a,s')(e_h^{\mu^{k},\beta^{k}}(s,a,s')
        +v_h^{\mu^{k}}(s,a,s')
        -
        \eta\sum_{j : j+d^j=k} \hat c_{h}^{j}(s,a)
        -\sum_{s''} \bar p_h^{k}(s''\mid s,a) v_{h+1}^{\mu^{k}}(s,a,s''))
        \\
        & =
        \sum_{h}\sum_{s,a,s'} q_{h}^{k}(s,a,s') e_h^{\mu^{k},\beta^{k}}(s,a,s')
        +\sum_{h}\sum_{s,a,s'}q_{h}^{k}(s,a,s')v_h^{\mu^{k}}(s,a,s')
        \\
        & \qquad -
        \eta\sum_{h}\sum_{s,a,s'}\sum_{j : j+d^j=k} q_{h}^{k}(s,a,s') \hat c_{h}^{j}(s,a)
        -
        \sum_{h}\sum_{s,a,s'}\sum_{s''}q_{h}^{k}(s,a,s')\bar p^{k}_{h}(s''\mid s,a)v_{h+1}^{\mu^{k}}(s,a,s'')
        \\
        & =
        \sum_{h}\sum_{s,a,s'} q_{h}^{k}(s,a,s') e_h^{\mu^{k},\beta^{k}}(s,a,s')
        +\sum_{h}\sum_{s,a,s'}q_{h}^{k}(s,a,s')v_h^{\mu^{k}}(s,a,s')
        \\
        & \qquad -
        \eta\sum_{h}\sum_{s,a,s'}\sum_{j : j+d^j=k} q_{h}^{k}(s,a,s') \hat c_{h}^{j}(s,a)
        -
        \sum_{h}\sum_{s,a}\sum_{s''}q_{h}^{k}(s,a)\bar p^{k}_{h}(s''\mid s,a)v_{h+1}^{\mu^{k}}(s,a,s'').
    \end{align*}
    Notice that:
    \begin{align*}
        \sum_{h,s,a,s''} & q_{h}^{k}(s,a) \bar p_h^{k}(s''\mid s,a) v_{h+1}^{\mu^{k}}(s,a,s'')
        \\
        & =
        \sum_{h,s,a,s''}q_{h}^{k}(s,a) p_h^{k}(s''\mid s,a)v_{h+1}^{\mu^{k}}(s,a,s'')
        +
        \sum_{h,s,a,s''}q_{h}^{k}(s,a) (\bar p_h^{k}(s''\mid s,a) - p_h^{k}(s''\mid s,a))v_{h+1}^{\mu^{k}}(s,a,s'')
        \\
        & =
        \sum_{h,s,a,s''}q_{h+1}^{k}(s,a,s'') v_{h+1}^{\mu^{k}}(s,a,s'')
        +
        \sum_{h,s,a,s''}q_{h}^{k}(s,a) (\bar p_h^{k}(s''\mid s,a) - p_h^{k}(s''\mid s,a))v_{h+1}^{\mu^{k}}(s,a,s''),
    \end{align*}
    and therefore:
    \begin{align*}
        (B)
        & =
        \sum_{h}\sum_{s,a,s'} q_{h}^{k}(s,a,s')e_h^{\mu^{k},\beta^{k}}(s,a,s')
        -
        \eta\sum_{h}\sum_{s,a,s'}\sum_{j : j+d^j=k}q_{h}^{k}(s,a,s')\hat c_{h}^{j}(s,a)
        \\
        & \qquad -
        \sum_{h,s,a,s''}q_{h}^{k}(s,a) (\bar p_h^{k}(s''\mid s,a) - p_h^{k}(s''\mid s,a))v_{h+1}^{\mu^{k}}(s,a,s'').
    \end{align*}
    Overall we get:
    \begin{align*}
        \sum_{h}\KL{q_{h}^{k}}{q_{h}^{k+1}}
        & \le
        \sum_{h,s,a,s'}q_{h}^{k}(s,a,s')\frac{(\eta\sum_{j : j+d^j=k} \hat c_{h}^{j}(s,a))^{2}}{2} - \sum_{h,s,a,s'} q_{h}^{k}(s,a,s')e_h^{\mu^{k},\beta^{k}}(s,a,s')
        \\
        & \qquad +
        \sum_{h,s,a,s'}q_{h}^{k}(s,a) (\bar p_h^{k}(s'\mid s,a) - p_h^{k}(s'\mid s,a))v_{h+1}^{\mu^{k}}(s,a,s').
    \end{align*}
    To finish the proof we show that:
    \[
        \sum_{h,s,a,s'}q_{h}^{k}(s,a) (\bar p_h^{k}(s'\mid s,a) - p_h^{k}(s'\mid s,a))v_{h+1}^{\mu^{k}}(s,a,s'')
        \le
        \sum_{h,s,a,s'} q_{h}^{k}(s,a,s')e_h^{\mu^{k},\beta^{k}}(s,a,s').
    \]
    By definition of $v^{\mu^k}$ and $\epsilon^k$, and since $\mu \ge 0$, we have:
    \begin{align*}
        \sum_{h,s,a,s'}q_{h}^{k}(s,a) & (\bar p_h^{k}(s'\mid s,a) - p_h^{k}(s'\mid s,a))v_{h+1}^{\mu^{k}}(s,a,s'')
        \\
        & =
        \sum_{h,s,a,s'}q_{h}^{k}(s,a) (\bar p_h^{k}(s'\mid s,a) - p_h^{k}(s'\mid s,a)) (\mu^{k,-}_{h+1}(s,a,s') - \mu^{k,+}_{h+1}(s,a,s'))
        \\
        & \le
        \sum_{h,s,a,s'}q_{h}^{k}(s,a) | \bar p_h^{k}(s'\mid s,a) - p_h^{k}(s'\mid s,a) | (\mu^{k,-}_{h+1}(s,a,s') + \mu^{k,+}_{h+1}(s,a,s'))
        \\
        & \le
        \sum_{h,s,a,s'}q_{h}^{k}(s,a) \epsilon^k_h(s' \mid s,a) (\mu^{k,-}_{h+1}(s,a,s') + \mu^{k,+}_{h+1}(s,a,s')),
    \end{align*}
    so to finish the proof it suffices to show that $\sum_{h,s,a,s'} \beta^k_{h+1}(s') - \beta^k_h(s) = 0$.
    Indeed, this follows as the sum is telescopic and $\beta^k_{H+1} = \beta^k_0 = 0$.
\end{proof}

\begin{lemma}[Lemma 8 of \cite{jin2019learning}; see also Lemma B.13 of in \cite{cohen2021minimax}]
    \label{lemma:confidence with m}
    Under the good event we have,
    \[
        \forall (k,s,a,s',h):
        \quad
        \ |p_h (s'|s,a) - \hat{p}^{k}_h (s'|s,a)|
        \le \tilde\epsilon^k_h( s' \mid  s,a).
    \]
    where $\tilde{\epsilon}_{h}^{k}(s'\mid s,a)=8\sqrt{\frac{p_{h}(s'\mid s,a)\logterm}{m_{h}^{k}(s,a)\vee1}}+\frac{100\logterm}{m_{h}^{k}(s,a)\vee1}$ 
\end{lemma}

\begin{lemma}
    \label{lemma:confidence with n}
    Under the good event we have, for any $(k,s,a,s',h)$ such that $m_h^k(s,a) \geq d_{max}$,
    \[
        |p_h (s'|s,a) - \hat{p}^{k}_h (s'|s,a)|
        \le \epsilon^k_h( s' \mid  s,a).
    \]
    where $\epsilon_{h}^{k}(s'\mid s,a)=16\sqrt{\frac{p_{h}(s'\mid s,a)\logterm}{n_{h}^{k}(s,a)\vee1}}+\frac{200\logterm}{n_{h}^{k}(s,a)\vee1}$.
\end{lemma}

\begin{proof}
    Note that if $m_h^k(s,a) \geq d_{max}$ then,
    \begin{align}
    \nonumber
        \frac{1}{m_{h}^{k}(s,a)\vee1} & =\frac{1}{n_{h}^{k}(s,a)\vee1}\frac{n_{h}^{k}(s,a)\vee1}{m_{h}^{k}(s,a)\vee1} = \frac{1}{n_{h}^{k}(s,a)\vee1} \left(1+\frac{n_{h}^{k}(s,a)\vee1-m_{h}^{k}(s,a)\vee1}{m_{h}^{k}(s,a)\vee1}\right) \\
        	                              & \leq\frac{1}{n_{h}^{k}(s,a)\vee1}\left(1+\frac{d_{max}}{m_{h}^{k}(s,a)\vee1}\right)\leq\frac{2}{n_{h}^{k}(s,a)\vee1}.                                                        \label{eq:epsilon < epsilon}                 
    \end{align}
    where the first inequality is since $n_h^k(s,a) - m_h^k(s,a) \leq d_{max}$.
    We complete the proof by combining \cref{lemma:confidence with m} with the fact that given \cref{eq:epsilon < epsilon} $\tilde{\epsilon}_{h}^{k}(s'\mid s,a)\leq\epsilon_{h}^{k}(s'\mid s,a)$.
\end{proof}

\begin{lemma}[Lemma E.4 of \cite{pmlr-v162-lancewicki22a} adapted to delays; see also Lemma 4 of \cite{jin2019learning}]
    \label{lemma:Jin1 delay}
    With delayed trajectory feedback, under the good event,
    \begin{align}
        \nonumber
    	\sum_{k=1}^K & \sum_{h,s,a} |u_{h}^{k}(s,a)  - q_{h}^{\pi^k}(s,a)| 
    	\lesssim
    	H\sum_{k=1}^{K}\sum_{h=1}^{H}\sum_{s \in \calS,a \in \calA}  \epsilon_{h}^{k}( s , a ) q_{h}^{\pi^k}(s,a)
    	\\
    	\nonumber
        & + HS \sum_{k=1}^{K} \sum_{1\leq h<\tilde{h}\leq H} \sum_{s \in \calS,a \in \calA,s' \in \calS} \sum_{\tilde{s} \in \calS,\tilde{a} \in \calA}  \epsilon_{h}^{k}(s' \mid  s,a) q_{h}^{\pi^k}(s,a)
    	  \min \left\{ 2 , \sum_{\tilde{s}' \in \calS} \epsilon_{\tilde{h}}^{k}(\tilde{s}' \mid \tilde{s},\tilde{a}) \right\} q_{\tilde{h}}^{\pi^k}(\tilde{s},\tilde{a} \mid s' ; h+1)
       \\
       & \qquad\qquad + H^3 S^2 A d_{max}
       \label{eq:Jin1 delay}
    \end{align}
    where $q_{\tilde{h}}^{\pi^k}(\tilde{s},\tilde{a} \mid \tilde{s}' ; h)$ be the probability to visit $(\tilde{s},\tilde{a})$ in time $\tilde h$ given that we visited $\tilde{s}'$ in time $h$, and $\epsilon_{h}^{k}(s'\mid s,a)=16\sqrt{\frac{p_{h}(s'\mid s,a)\logterm}{n_{h}^{k}(s,a)\vee1}}+\frac{200\logterm}{n_{h}^{k}(s,a)\vee1}$
\end{lemma}

\begin{proof}
    Let $\calK_{h,s,a}=\{k:s_{h}^{k}=s,a_{h}^{k}=a,m_{h}^{k}(s,a)\leq d_{max}\}$ and define $\bbI_{h,s,a,k}=\bbI\{k\in\calK_{h,s,a}\}$ and $\bar{\bbI}_{h,s,a,k}=1-\bbI_{h,s,a,k}$.
    Let $q^{k,s,h}$ be the occupancy measure such that
    $q_{h}^{k,s,h}(s)=u_{h}^{k}(s)$, and let $p^{k,s,h}$ be the transition
    that corresponds to $q^{k,s,h}$. Let $\sigma_{h}(s)$ be the set
    of all trajectories that end in $s$ in time $h$, i.e., $\sigma_{h}(s)=\{s_{1},a_{1},\dots,s_{h-1},a_{h-1},s_{h}\}$
    where $s_{h}=s$. We have: 
    \begin{align*}
    	u_{h}^{k}(s,a)=q_{h}^{k,s,h}(s,a) & =\pi_{h}^{k}(a\mid s)\sum_{\sigma_{h}(s)}\prod_{h'=1}^{h-1}\pi_{h'}^{k}(a_{h'}\mid s_{h'})p_{h'}^{k,s,h}(s_{h'+1}\mid s_{h'},a_{h'}) \\
    	q_{h}^{\pi^{k}}(s,a)              & =\pi_{h}^{k}(a\mid s)\sum_{\sigma_{h}(s)}\prod_{h'=1}^{h-1}\pi_{h'}^{k}(a_{h'}\mid s_{h'})p_{h'}(s_{h'+1}\mid s_{h'},a_{h'}).        
    \end{align*}
    Then, 
    \[
    	|u_{h}^{k}(s,a)-q_{h}^{\pi^{k}}(s,a)|=\pi_{h}^{k}(a\mid s)\sum_{\sigma_{h}(s)}\prod_{h'=1}^{h-1}\pi_{h'}^{k}(a_{h'}\mid s_{h'})\left|\prod_{h'=1}^{h-1}p_{h'}^{k,s,h}(s_{h'+1}\mid s_{h'},a_{h'})-\prod_{h'=1}^{h-1}p_{h'}(s_{h'+1}\mid s_{h'},a_{h'})\right|.
    \]
    We can rewrite the following term as, 
    \begin{align*}
    	\Bigg|\prod_{h'=1}^{h-1}p^{k,s,h} & (s_{h'+1}\mid s_{h'},a_{h'})-\prod_{h'=1}^{h-1}p_{h'}(s_{h'+1}\mid s_{h'},a_{h'})\Bigg|                                                                                                                               \\
    	                                  & =\Bigg|\sum_{l=2}^{h-1}\prod_{h'=1}^{l-1}p_{h'}(s_{h'+1}\mid s_{h'},a_{h'})\prod_{h'=l}^{h-1}p_{h'}^{k,s,h}(s_{h'+1}\mid s_{h'},a_{h'})+\prod_{h'=1}^{h-1}p_{h'}^{k,s,h}(s_{h'+1}\mid s_{h'},a_{h'})                  \\
    	                                  & -\prod_{h'=1}^{h-1}p_{h'}(s_{h'+1}\mid s_{h'},a_{h'})-\sum_{l=2}^{h-1}\prod_{h'=1}^{l-1}p_{h'}(s_{h'+1}\mid s_{h'},a_{h'})\prod_{h'=l}^{h-1}p_{h'}^{k,s,h}(s_{h'+1}\mid s_{h'},a_{h'})\Bigg|                          \\
    	                                  & =\Bigg|\sum_{l=1}^{h-1}\prod_{h'=1}^{l-1}p_{h'}(s_{h'+1}\mid s_{h'},a_{h'})\prod_{h'=l}^{h-1}p_{h'}^{k,s,h}(s_{h'+1}\mid s_{h'},a_{h'})                                                                               \\
    	                                  & -\sum_{l=2}^{h}\prod_{h'=1}^{l-1}p_{h'}(s_{h'+1}\mid s_{h'},a_{h'})\prod_{h'=l}^{h-1}p_{h'}^{k,s,h}(s_{h'+1}\mid s_{h'},a_{h'})\Bigg|                                                                                 \\
    	                                  & =\Bigg|\sum_{l=1}^{h-1}\prod_{h'=1}^{l-1}p_{h'}(s_{h'+1}\mid s_{h'},a_{h'})\prod_{h'=l}^{h-1}p_{h'}^{k,s,h}(s_{h'+1}\mid s_{h'},a_{h'})                                                                               \\
    	                                  & -\sum_{l=1}^{h-1}\prod_{h'=1}^{l}p_{h'}(s_{h'+1}\mid s_{h'},a_{h'})\prod_{h'=l+1}^{h-1}p_{h'}^{k,s,h}(s_{h'+1}\mid s_{h'},a_{h'})\Bigg|                                                                               \\
    	                                  & =\sum_{l=1}^{h-1}\left|p_{l}^{k,s,h}(s_{l+1}\mid s_{l},a_{l})-p_{l}(s_{l+1}\mid s_{l},a_{l})\right|\prod_{h'=1}^{l-1}p_{h}(s_{h'+1}\mid s_{h'},a_{h'})\prod_{h'=l+1}^{h-1}p_{h'}^{k,s,h}(s_{h'+1}\mid s_{h'},a_{h'}). 
    \end{align*}
    Hence, 
    \begin{align}
    	|u_{h}^{k}( & s,a)-q_{h}^{\pi^{k}}(s,a)|                                                                                                          \nonumber                                                                                                                   \\
     \nonumber
    	            & \leq\pi_{h}^{k}(a\mid s)\sum_{\sigma_{h}(s)}\prod_{h'=1}^{h-1}\pi_{h'}^{k}(a_{h'}\mid s_{h'})\sum_{l=1}^{h-1}\left|p_{l}^{k,s,h}(s_{l+1}\mid s_{l},a_{l})-p_{l}(s_{l+1}\mid s_{l},a_{l})\right|                                                        \\
                 \nonumber
    	            & \qquad\qquad\qquad\qquad\qquad\qquad\qquad\qquad\qquad\qquad\cdot\prod_{h'=1}^{l-1}p_{h'}(s_{h'+1}\mid s_{h'},a_{h'})\prod_{h'=l+1}^{h-1}p_{h'}^{k,s,h}(s_{h'+1}\mid s_{h'},a_{h'})                                                                    \\
                 \nonumber
    	            & \leq\sum_{l=1}^{h-1}\sum_{\sigma_{h}(s)}\left|p_{l}^{k,s,h}(s_{l+1}\mid s_{l},a_{l})-p_{l}(s_{l+1}\mid s_{l},a_{l})\right|\left(\pi_{l}^{k}(a_{l}\mid s_{l})\prod_{h'=1}^{l-1}\pi_{h'}^{k}(a_{h'}\mid s_{h'})p_{h'}(s_{h'+1}\mid s_{h'},a_{h'})\right) \\
                 \nonumber
    	            & \qquad\qquad\qquad\qquad\qquad\qquad\qquad\qquad\cdot\left(\pi_{h}^{k}(a\mid s)\prod_{h'=l+1}^{h-1}\pi_{h'}^{k}(a_{h'}\mid s_{h'})p_{h'}^{k,s,h}(s_{h'+1}\mid s_{h'},a_{h'})\right)                                                                    \\
                 \nonumber
    	            & =\sum_{l=1}^{h-1}\sum_{s_{l}\in\calS,a_{l}\in\calA,s_{l+1}\in\calS}\left|p_{l}^{k,s,h}(s_{l+1}\mid s_{l},a_{l})-p_{l}(s_{l+1}\mid s_{l},a_{l})\right|                                                                                                  \\
                 \nonumber
    	            & \qquad\qquad\qquad\quad\cdot\left(\sum_{\sigma_{l}(s_{l})}\pi_{l}^{k}(a_{l}\mid s_{l})\prod_{h'=1}^{l-1}\pi_{h'}^{k}(a_{h'}\mid s_{h'})p_{h'}(s_{h'+1}\mid s_{h'},a_{h'})\right)                                                                       \\
                 \nonumber
    	            & \qquad\qquad\qquad\quad\cdot\left(\sum_{a_{l+1}\in\calA}\sum_{\{s_{h''}\in\calS,a_{h''}\in\calA\}_{h''=l+2}^{h-1}}\pi_{h}^{k}(a\mid s)\prod_{h'=l+1}^{h-1}\pi_{h'}^{k}(a_{h'}\mid s_{h'})p_{h'}^{k,s,h}(s_{h'+1}\mid s_{h'},a_{h'})\right)             \\
    	            & =\sum_{l=1}^{h-1}\sum_{s_{l}\in\calS,a_{l}\in\calA,s_{l+1}\in\calS}\left|p_{l}^{k,s,h}(s_{l+1}\mid s_{l},a_{l})-p_{l}(s_{l+1}\mid s_{l},a_{l})\right|q_{l}^{\pi^{k}}(s_{l},a_{l})\cdot q_{h}^{k,s,h}(s,a\mid s_{l+1}),  
                 \label{eq: lemma E.4 u-q}
    \end{align}
    where we ease notation and denote $q_{h}^{k,s,h}(s,a\mid s_{l+1})=q_{h}^{k,s,h}(s,a\mid s_{l+1};l+1)$.
    Similarly, we can show that, 
    \begin{align}
    \nonumber
    	  & |q_{h}^{k,s,h}(s,a\mid s_{l+1})-q_{h}^{\pi^{k}}(s,a\mid s_{l+1})|                                                                                                                                                                                          \\
    	  & \quad\lesssim\sum_{h'=l+1}^{h-1}\sum_{s_{h'}\in\calS,a_{h'}\in\calA,s_{h'+1}\in\calS}\left|p_{h'}^{k,s,h}(s_{h'+1}\mid s_{h'},a_{h'})-p_{h'}(s_{h'+1}\mid s_{h'},a_{h'})\right|q_{h'}^{\pi^{k}}(s_{h'},a_{h'}\mid s_{l+1})q_{h'}^{k,s,h}(s,a\mid s_{h'+1}) \\
    	  & \quad\leq\pi_{h}^{k}(a\mid s)\sum_{h'=l+1}^{h-1}\sum_{s_{h'}\in\calS,a_{h'}\in\calA,s_{h'+1}\in\calS}\left|p_{h'}^{k,s,h}(s_{h'+1}\mid s_{h'},a_{h'})-p_{h'}(s_{h'+1}\mid s_{h'},a_{h'})\right|q_{h'}^{\pi^{k}}(s_{h'},a_{h'}\mid s_{l+1}),   
       \label{eq: lemma E.4 u-q conditional}
    \end{align}
    where the last is since $q_{h'}^{k,s,h}(s,a\mid s_{h'+1})\leq\pi_{h}^{k}(a\mid s)$.
    Decomposing \cref{eq: lemma E.4 u-q} for episodes $k\in\calK_{l,s_l,a_l}$ and $k \notin \calK_{l,s_l,a_l}$
    
    \begin{align*}
    	\sum_{h,s,a,k} & |u_{h}^{k}(s,a)-q_{h}^{\pi^{k}}(s,a)|                                                                                                                                                                                                                                                     \\
    	               & \lesssim\underbrace{\sum_{h,s,a,k}\sum_{l=1}^{h-1}\sum_{s_{l}\in\calS,a_{l}\in\calA,s_{l+1}\in\calS}\bbI_{l,s_{l},a_{l},k}\left|p_{l}^{k,s,h}(s_{l+1}\mid s_{l},a_{l})-p_{l}(s_{l+1}\mid s_{l},a_{l})\right|q_{l}^{\pi^{k}}(s_{l},a_{l})\cdot q_{h}^{k,s,h}(s,a\mid s_{l+1})}_{(i)}       \\
    	               & \qquad+\underbrace{\sum_{h,s,a,k}\sum_{l=1}^{h-1}\sum_{s_{l}\in\calS,a_{l}\in\calA,s_{l+1}\in\calS}\bar{\bbI}_{l,s_{l},a_{l},k}\left|p_{l}^{k,s,h}(s_{l+1}\mid s_{l},a_{l})-p_{l}(s_{l+1}\mid s_{l},a_{l})\right|q_{l}^{\pi^{k}}(s_{l},a_{l})\cdot q_{h}^{k,s,h}(s,a\mid s_{l+1})}_{(ii)} 
    \end{align*}
    Now,
    \begin{align*}
    	(i) & =\sum_{h,s,a,k}\sum_{l=1}^{h-1}\sum_{s_{l}\in\calS,a_{l}\in\calA,s_{l+1}\in\calS}\bbI_{l,s_{l},a_{l},k}\left|p_{l}^{k,s,h}(s_{l+1}\mid s_{l},a_{l})-p_{l}(s_{l+1}\mid s_{l},a_{l})\right|q_{l}^{\pi^{k}}(s_{l},a_{l})\cdot q_{h}^{k,s,h}(s,a\mid s_{l+1}) \\
    	    & \leq\sum_{h,s,a,k}\sum_{l=1}^{h-1}\sum_{s_{l}\in\calS,a_{l}\in\calA,s_{l+1}\in\calS}\bbI_{l,s_{l},a_{l},k}\left|p_{l}^{k,s,h}(s_{l+1}\mid s_{l},a_{l})-p_{l}(s_{l+1}\mid s_{l},a_{l})\right|q_{l}^{\pi^{k}}(s_{l},a_{l})\cdot\pi_{h}^{k}(a\mid s)         \\
    	    & =\sum_{h,s,k}\sum_{l=1}^{h-1}\sum_{s_{l}\in\calS,a_{l}\in\calA,s_{l+1}\in\calS}\bbI_{l,s_{l},a_{l},k}\left|p_{l}^{k,s,h}(s_{l+1}\mid s_{l},a_{l})-p_{l}(s_{l+1}\mid s_{l},a_{l})\right|q_{l}^{\pi^{k}}(s_{l},a_{l})                                       \\
    	    & =2S\sum_{h,k}\sum_{l=1}^{h-1}\sum_{s_{l}\in\calS,a_{l}\in\calA}\bbI_{l,s_{l},a_{l},k}q_{l}^{\pi^{k}}(s_{l},a_{l})                                                                                                                                         \\
    	    & \leq2S\sum_{h}\sum_{l=1}^{h-1}\sum_{s_{l}\in\calS,a_{l}\in\calA}\sum_{k}\bbI_{l,s_{l},a_{l},k}                                                                                                                                                            \\
    	    & \leq4S\sum_{h}\sum_{l=1}^{h-1}\sum_{s_{l}\in\calS,a_{l}\in\calA}d_{max}\leq4H^{2}S^{2}Ad_{max}      ,                                                                                                                                                      
    \end{align*}
    where the third inequality is since $|\calK_{h,s,a}| \leq 2 d_{max}$ for any $h,s$ and $a$.
    For $(ii)$ we first use \cref{eq: lemma E.4 u-q conditional} to bound,
    
    \begin{align*}
    	(ii) & \leq\underbrace{\sum_{h,s,a,k}\sum_{l=1}^{h-1}\sum_{s_{l}\in\calS,a_{l}\in\calA,s_{l+1}\in\calS}\bar{\bbI}_{l,s_{l},a_{l},k}\left|p_{l}^{k,s,h}(s_{l+1}\mid s_{l},a_{l})-p_{l}(s_{l+1}\mid s_{l},a_{l})\right|q_{l}^{\pi^{k}}(s_{l},a_{l})\cdot q_{h}^{\pi^{k}}(s,a\mid s_{l+1})}_{(iii)} \\
    	     & \qquad+\sum_{h,s,a,k}\sum_{l=1}^{h-1}\sum_{s_{l}\in\calS,a_{l}\in\calA,s_{l+1}\in\calS}\bar{\bbI}_{l,s_{l},a_{l},k}\left|p_{l}^{k,s,h}(s_{l+1}\mid s_{l},a_{l})-p_{l}(s_{l+1}\mid s_{l},a_{l})\right|q_{l}^{\pi^{k}}(s_{l},a_{l})\pi_{h}^{k}(a\mid s)                                     \\
    	     & \underbrace{\qquad\qquad\cdot\left(\sum_{h'=l+1}^{h-1}\sum_{s_{h'}\in\calS,a_{h'}\in\calA,s_{h'+1}\in\calS}\left|p_{h'}^{k,s,h}(s_{h'+1}\mid s_{h'},a_{h'})-p_{h'}(s_{h'+1}\mid s_{h'},a_{h'})\right|q_{h'}^{\pi^{k}}(s_{h'},a_{h'}\mid s_{l+1})\right)}_{(iv)}                           
    \end{align*}
    
    Now using \cref{lemma:confidence with n},
    
    \begin{align*}
    	(iii) & \le\sum_{k,h}\sum_{l=1}^{h-1}\sum_{s_{l}\in\calS,a_{l}\in\calA,s_{l+1}\in\calS}\bar{\bbI}_{l,s_{l},a_{l},k}\epsilon_{l}^{k}(s_{l+1}\mid s_{l},a_{l})q_{l}^{\pi^{k}}(s_{l},a_{l})\cdot\left(\sum_{s,a}q_{h}^{\pi^{k}}(s,a\mid s_{l+1})\right) \\
    	      & \leq H\sum_{k=1}^{K}\sum_{1\leq l\leq H}\sum_{s_{l}\in\calS,a_{l}\in\calA,s_{l+1}\in\calS}\epsilon_{l}^{k}(s_{l+1}\mid s_{l},a_{l})q_{l}^{\pi^{k}}(s_{l},a_{l})                                                                              \\
    	      & =H\sum_{k=1}^{K}\sum_{h=1}^{H}\sum_{s\in\calS,a\in\calA,s'\in\calS}\epsilon_{h}^{k}(s'\mid s,a)q_{h}^{\pi^{k}}(s,a)                                                                                                                          
    \end{align*}
    
    For $(iv)$ we again devide into $k\in\calK_{h',s_{h'},a_{h'}}$ and $k\notin\calK_{h',s_{h'},a_{h'}}$,
    
    \begin{align}
    \nonumber
    	(iv) =   & \sum_{h,s,a,k}\sum_{l=1}^{h-1}\sum_{s_{l}\in\calS,a_{l}\in\calA,s_{l+1}\in\calS}\bar{\bbI}_{l,s_{l},a_{l},k}\left|p_{l}^{k,s,h}(s_{l+1}\mid s_{l},a_{l})-p_{l}(s_{l+1}\mid s_{l},a_{l})\right|q_{l}^{\pi^{k}}(s_{l},a_{l})\pi_{h}^{k}(a\mid s)                                  \\
     \nonumber
    	  & \qquad\cdot\left(\sum_{h'=l+1}^{h-1}\sum_{s_{h'}\in\calS,a_{h'}\in\calA,s_{h'+1}\in\calS}\bbI_{h',s_{h'},a_{h'},k}\left|p_{h'}^{k,s,h}(s_{h'+1}\mid s_{h'},a_{h'})-p_{h'}(s_{h'+1}\mid s_{h'},a_{h'})\right|q_{h'}^{\pi^{k}}(s_{h'},a_{h'}\mid s_{l+1})\right)            \\
       \nonumber
    	  & \quad+\sum_{h,s,a,k}\sum_{l=1}^{h-1}\sum_{s_{l}\in\calS,a_{l}\in\calA,s_{l+1}\in\calS}\bar{\bbI}_{l,s_{l},a_{l},k}\left|p_{l}^{k,s,h}(s_{l+1}\mid s_{l},a_{l})-p_{l}(s_{l+1}\mid s_{l},a_{l})\right|q_{l}^{\pi^{k}}(s_{l},a_{l})\pi_{h}^{k}(a\mid s)                            \\
    	  & \qquad\cdot\left(\sum_{h'=l+1}^{h-1}\sum_{s_{h'}\in\calS,a_{h'}\in\calA,s_{h'+1}\in\calS}\bar{\bbI}_{h',s_{h'},a_{h'},k}\left|p_{h'}^{k,s,h}(s_{h'+1}\mid s_{h'},a_{h'})-p_{h'}(s_{h'+1}\mid s_{h'},a_{h'})\right|q_{h'}^{\pi^{k}}(s_{h'},a_{h'}\mid s_{l+1})\right) 
       \label{eq:lemma E.4 term iv}
    \end{align}
    
    The first term is bounded in a similar way to $(i)$ by,
    \begin{align*}
    	  & \sum_{h,s,a,k}\sum_{l=1}^{h-1}\sum_{s_{l}\in\calS,a_{l}\in\calA,s_{l+1}\in\calS}\bar{\bbI}_{l,s_{l},a_{l},k}\left|p_{l}^{k,s,h}(s_{l+1}\mid s_{l},a_{l})-p_{l}(s_{l+1}\mid s_{l},a_{l})\right|q_{l}^{\pi^{k}}(s_{l},a_{l})\pi_{h}^{k}(a\mid s)                                                                                \\
    	  & \qquad\qquad\cdot\left(\sum_{h'=l+1}^{h-1}\sum_{s_{h'}\in\calS,a_{h'}\in\calA,s_{h'+1}\in\calS}\bbI_{h',s_{h'},a_{h'},k}\left|p_{h'}^{k,s,h}(s_{h'+1}\mid s_{h'},a_{h'})-p_{h'}(s_{h'+1}\mid s_{h'},a_{h'})\right|q_{h'}^{\pi^{k}}(s_{h'},a_{h'}\mid s_{l+1})\right)                                                          \\
    	  & \leq2\sum_{h,s,a,k}\sum_{l=1}^{h-1}\sum_{s_{l}\in\calS,a_{l}\in\calA,s_{l+1}\in\calS}\bar{\bbI}_{l,s_{l},a_{l},k}\left|p_{l}^{k,s,h}(s_{l+1}\mid s_{l},a_{l})-p_{l}(s_{l+1}\mid s_{l},a_{l})\right|q_{l}^{\pi^{k}}(s_{l},a_{l})\pi_{h}^{k}(a\mid s)                                                                           \\
    	  & \qquad\qquad\cdot\left(\sum_{h'=l+1}^{h-1}\sum_{s_{h'}\in\calS,a_{h'}\in\calA}\bbI_{h',s_{h'},a_{h'},k}q_{h'}^{\pi^{k}}(s_{h'},a_{h'}\mid s_{l+1})\right)                                                                                                                                                                     \\
    	  & \leq2\sum_{h,s,k}\sum_{l=1}^{h-1}\sum_{s_{l}\in\calS,a_{l}\in\calA,s_{l+1}\in\calS}\bar{\bbI}_{l,s_{l},a_{l},k}\left|p_{l}^{k,s,h}(s_{l+1}\mid s_{l},a_{l})-p_{l}(s_{l+1}\mid s_{l},a_{l})\right|q_{l}^{\pi^{k}}(s_{l},a_{l})\cdot\left(\sum_{h'=l+1}^{h-1}\sum_{s_{h'}\in\calS,a_{h'}\in\calA}\bbI_{h',s_{h'},a_{h'}}\right) \\
    	  & \leq4\sum_{h,s,k}\sum_{l=1}^{h-1}\sum_{s_{l}\in\calS,a_{l}\in\calA}\bar{\bbI}_{l,s_{l},a_{l},k}q_{l}^{\pi^{k}}(s_{l},a_{l})\cdot\left(\sum_{h'=l+1}^{h-1}\sum_{s_{h'}\in\calS,a_{h'}\in\calA}\bbI_{h',s_{h'},a_{h'}}\right)                                                                                                   \\
    	  & \leq4\sum_{h,s,k}\sum_{l=1}^{h-1}\sum_{h'=l+1}^{h-1}\sum_{s_{h'}\in\calS,a_{h'}\in\calA}\bbI_{h',s_{h'},a_{h'},k}                                                                                                                                                                                                             \\
    	  & =4S\sum_{h}\sum_{l=1}^{h-1}\sum_{h'=l+1}^{h-1}\sum_{s_{h'}\in\calS,a_{h'}\in\calA}\sum_{k}\bbI_{h',s_{h'},a_{h'},k}                                                                                                                                                                                                           \\
    	  & \leq 8S\sum_{h}\sum_{l=1}^{h-1}\sum_{h'=l+1}^{h-1}\sum_{s_{h'}\in\calS,a_{h'}\in\calA}d_{max}=8H^{3}S^{2}Ad_{max}                                                                                                                                                                                                                 
    \end{align*}
    where the last inequality is  since $|\calK_{h,s,a}| \leq 2 d_{max}$ for any $h,s$ and $a$.
    Again, using \cref{lemma:confidence with n}, the second term in \cref{eq:lemma E.4 term iv} is bounded by, 
    
    \begin{align*}
    	  & \sum_{h,s,k}\sum_{l=1}^{h-1}\sum_{s_{l}\in\calS,a_{l}\in\calA,s_{l+1}\in\calS}\epsilon_{l}^{k}(s_{l+1}\mid s_{l},a_{l})q_{l}^{\pi^{k}}(s_{l},a_{l})\sum_{a}\pi_{h}^{k}(a\mid s)                                                                                                                                                                   \\
    	  & \qquad\qquad\qquad\cdot\left(\sum_{h'=l+1}^{h-1}\sum_{s_{h'}\in\calS,a_{h'}\in\calA}\min\left\{ 2,\sum_{s_{h'+1}\in\calS}\epsilon_{h'}^{k}(s_{h'+1}\mid s_{h'},a_{h'})\right\} q_{h'}^{\pi^{k}}(s_{h'},a_{h'}\mid s_{l+1})\right)                                                                                                                 \\
    	  & =HS\sum_{k=1}^{K}\sum_{1\leq l<h'\leq H}\sum_{s_{l}\in\calS,a_{l}\in\calA,s_{l+1}\in\calS}\sum_{s_{h'}\in\calS,a_{h'}\in\calA}\epsilon_{l}^{k}(s_{l+1}\mid s_{l},a_{l})q_{l}^{\pi^{k}}(s_{l},a_{l})                                                                                                                                               \\
    	  & \qquad\qquad\qquad\qquad\qquad\qquad\qquad\qquad\qquad\qquad\cdot\min\left\{ 2,\sum_{s_{h'+1}\in\calS}\epsilon_{h'}^{k}(s_{h'+1}\mid s_{h'},a_{h'})\right\} q_{h'}^{\pi^{k}}(s_{h'},a_{h'}\mid s_{l+1})                                                                                                                                           \\
    	  & =HS\sum_{k=1}^{K}\sum_{1\leq h<\tilde{h}\leq H}\sum_{s\in\calS,a\in\calA,s'\in\calS}\sum_{\tilde{s}\in\calS,\tilde{a}\in\calA}\epsilon_{h}^{k}(s'\mid s,a)q_{h}^{\pi^{k}}(s,a)\min\left\{ 2,\sum_{\tilde{s}'\in\calS}\epsilon_{\tilde{h}}^{k}(\tilde{s}'\mid\tilde{s},\tilde{a})\right\} q_{\tilde{h}}^{\pi^{k}}(\tilde{s},\tilde{a}\mid s';h+1). 
    \end{align*}
    Summing the different terms completes the proof.
\end{proof}

\newpage
\begin{lemma}[Lemma 4 of \cite{jin2019learning} adapted to delays]
    \label{lemma:Jin-final delay}
    With delayed trajectory feedback, under the good event,
    \[
        \sum_{h,s,a,k}|u_{h}^{k}(s,a) - q_{h}^{\pi^k}(s,a)|
        \lesssim \sqrt{ H^{4} S^{2} A K \logterm} + H^3 S^{3} A \logterm^{2} + H^3 S^2 A d_{max}.
    \]
\end{lemma}

\begin{proof}
    Given \cref{lemma:Jin1 delay}, the proof proceeds exactly like the proof  of \cite[Lemma E.5]{pmlr-v162-lancewicki22a}. 
\end{proof}

\end{document}